%% file: main.tex
\documentclass{article}
\usepackage{a4wide}
\input{preamble}

\input{new_macros}

\title{Adaptive Sampling for Estimating Multiple Probability Distributions}
\date{}
\author{
  Shubhanshu Shekhar\\
  \texttt{shshekha@eng.ucsd.edu} \\
  \and
  Tara Javidi \\
  \texttt{tjavidi@eng.ucsd.edu}
  \and 
  Mohammad Ghavamzadeh\\
  \texttt{mgh@fb.com} 
}

\begin{document}
\maketitle

\begin{abstract}
We consider the problem of allocating samples to a finite set of discrete distributions in order to learn them uniformly well in terms of four common distance measures: $\ell_2^2$, $\ell_1$, $f$-divergence, and separation distance. To present a unified treatment of these distances, we first propose a general optimistic tracking algorithm and analyze its sample allocation performance w.r.t.~an oracle. We then instantiate this algorithm for the four distance measures and derive bounds on the regret of their resulting allocation schemes. We verify our theoretical findings through some experiments. Finally, we show that the techniques developed in the paper can be easily extended to the related setting of minimizing the average error (in terms of the four distances) in learning a set of distributions. 
\end{abstract}

\section{Introduction}
\label{sec:introduction}
\input{sec1_Introduction.tex}

\section{Problem Setup}
\label{sec:problem_setup}
\input{sec2_problem_setup.tex}

\section{General Tracking Problem}
\label{sec:general_tracking}
\input{general_tracking.tex}
\section{Adaptive Allocation Algorithms}
\label{sec:algorithms}

\input{sec3_Algorithms.tex}

\section{Lower Bound}
\label{sec:lower}
\input{LowerBound.tex}

\section{Experiments}
\label{sec:experiments}
\input{Experiments.tex}

\section{Minimizing Average Discrepancy}
\label{sec:discussion}
\input{Discussion.tex}

\section{Conclusion}
\label{sec:conclusion}
\input{Conclusion.tex}

\newpage \newpage
 \bibliographystyle{apalike}
\bibliography{ref}

\newpage 
\onecolumn 
\begin{appendix}

\section{Table of Symbols}
\input{appendix_Table.tex}
\label{appendix:table_of_symbols}
\newpage 

\section{Regret with approximate objective function}
\input{appendix_approximation.tex}
\label{appendix:approximate} 

\section{Proof of General Tracking Lemma (Lemma~\ref{lemma:tracking1})}
\input{appendix_general_tracking.tex}

\section{Analysis for $\ell_2^2$ loss}
\label{appendix:analysis_ell2}

\input{appendix_regret_ell2}

\section{Analysis for $\ell_1$ loss}
\label{appendix:analysis_ell1}
\input{appendix_regret_ell1.tex}
\section{Regret bound for $f$-divergence}
\label{appendix:analysis_f_divergence}
\input{appendix_regret_f_divergence.tex}

\section{Analysis for  KL divergence}
\label{appendix:analysis_kl}
\input{appendix_regret_kl.tex}

\section{Regret bound for Separation Distance}
\label{appendix:analysis_separation}
\input{appendix_regret_separation.tex}
\section{Deferred Proofs from Section~\ref{sec:lower}}
\label{appendix:lower}
\input{appendix_LowerBound.tex}


\end{appendix}

\end{document}

%% file: preamble.tex
\usepackage{xcolor}

\usepackage{natbib}

\usepackage{subcaption}
\usepackage{diagbox}
\usepackage{amsmath}
\usepackage{amsthm}
\usepackage{amssymb}
\usepackage{bm} 
\usepackage{bbm} 
\usepackage{mathtools}
\usepackage{algorithm}
\usepackage{algorithmic}
\usepackage{enumitem}
\usepackage{longtable}

\newtheorem{theorem}{Theorem}
\newtheorem{lemma}{Lemma}
\newtheorem{proposition}{Proposition}

\newtheorem{corollary}{Corollary}

\newtheorem{definition}{Definition}
\newtheorem{remark}{Remark}

\newcommand{\lp}{\left (} 
\newcommand{\rp}{\right )} 
\newcommand{\lb}{\left [}
\newcommand{\rb}{\right ]}

\newcommand{\mc}[1]{\mathcal{#1}}
\newcommand{\mbb}[1]{\mathbb{#1}}
\newcommand{\mf}[1]{\mathfrak{#1}}
\newcommand{\indi}[1]{\mathbbm{1}_{ \{#1\} }}

\newcommand{\X}{\mathcal{X}}

\DeclareMathOperator*{\argmax}{arg\,max}
\DeclareMathOperator*{\argmin}{arg\,min}

\newcommand{\tbf}[1]{\textbf{#1}}

\makeatletter
\newcommand{\pushright}[1]{\ifmeasuring@#1\else\omit\hfill$\displaystyle#1$\fi\ignorespaces}
\newcommand{\pushleft}[1]{\ifmeasuring@#1\else\omit$\displaystyle#1$\hfill\fi\ignorespaces}
\makeatother


%% file: new_macros.tex
\newcommand{\f}{\varphi}

\newcommand{\Dtv}{D_{\text{TV}}} 
\newcommand{\Dkl}{D_{\text{KL}}} 
\newcommand{\Dltwo}{D_{\ell_2}}
\newcommand{\Dlone}{D_{\ell_1}} 
\newcommand{\Dsep}{D_{\text{s}}} 
\newcommand{\Dfdiv}{D_{f}}

\newcommand{\ciltwo}[1][i]{c_{#1}^{(\ell_2)}} 
\newcommand{\cilone}[1][i]{c_{#1}^{(\ell_1)}} 
\newcommand{\cikl}[1][i]{c_{#1}^{(\text{KL})}} 
\newcommand{\cisep}[1][i]{c_{#1}^{(s)}} 

\newcommand{\Cltwo}{C_{\ell_2}} 
\newcommand{\Clone}{C_{\ell_1}} 

\newcommand{\hatciltwo}[1][i]{\hat{c}_{#1, t}^{(\ell_2)}}
\newcommand{\hatcilone}[1][i]{\hat{c}_{#1, t}^{(\ell_1)}}

\newcommand{\eijtlone}{e_{ij, t}^{(\ell_1)}} 
\newcommand{\eitlone}{e_{i,t}^{(\ell_1)}}
\newcommand{\eijtltwo}{e_{ij,t}^{(\ell_2)}}
\newcommand{\eitltwo}{ e_{i,t}^{(\ell_2)}}

\newcommand{\uiltwo}[1][i]{u_{#1,t}^{(\ell_2)}}
\newcommand{\uilone}[1][i]{u_{#1,t}^{(\ell_1)}}

\newcommand{\Tins}[1][i]{\widetilde{T}_{#1}^*} 
\newcommand{\Tit}[1][i]{T_{#1,t}} 

\newcommand{\Algltwo}{\mathcal{A}_{\ell_2}}
\newcommand{\Alglone}{\mc{A}_{\ell_1}}
\newcommand{\Algkl}{\mc{A}_{\text{KL}}} 
\newcommand{\Algsep}{\mc{A}_s}
\newcommand{\Algfdiv}{\mc{A}_f}
\newcommand{\Algapprox}{\mc{\widetilde{A}}^*}

\newcommand{\Mltwo}{M_{\ell_2}}
\newcommand{\Mlone}{M_{\ell_1}}
\newcommand{\Mkl}{M_{\text{KL}}}

\newcommand{\lambdailtwo}[1][i]{\lambda_{#1}^{(\ell_2)}} 
\newcommand{\lambdailone}[1][i]{\lambda_{#1}^{(\ell_1)}}

\newcommand{\Ber}[1]{\text{Ber}\lp #1\rp}
\newcommand{\cfunc}[1]{\mf{c}\lp#1\rp}
\newcommand{\cfuncs}{\mf{c}}
\newcommand{\kl}{\text{kl}}

%% file: sec1_Introduction.tex
Consider the problem in which a learner  must allocate $n$ samples among $K$ discrete distributions to construct \emph{uniformly good} estimates of these distributions in terms of a distance measure $D$. Depending on $D$, certain distributions may require much fewer samples than the others to be estimated with the same precision.  The optimal sampling strategy for a given $n$ requires knowledge of the true distributions. The goal of this paper is to design adaptive allocation strategies that converge to the optimal strategy, without relying on the knowledge of the true distributions.

The problem described above models several practical applications. Here we present two examples. \tbf{(1)~Learning MDP model:} 
In many sequential decision-making problems, the agent's interaction with the environment is modelled as a Markov decision process (MDP). In these problems, it is often quite important to accurately estimate the dynamics (i.e.,~the transition structure of the MDP), given a finite exploration budget. Learning MDP model is equivalent to estimating $S\times A$ distributions, where $S$ and $A$ are the number of states and actions of the (finite) MDP. Therefore, if we assume the existence of a known policy that can efficiently transition the MDP between any two states, then the problem of learning MDP model given a fixed budget of exploration reduces to finding the optimal allocation of samples to these $S\times A$ distributions. Thus, the framework studied in this paper provides the first step towards solving the general problem of constructing accurate models for MDPs. The requirement of a known policy to transition between states can be relaxed by employing the techniques that have been recently developed for efficient exploration in MDPs (e.g.,~\cite{tarbouriech2019active,Hazan19PE}), which we leave for future work. 
\tbf{(2)~Compression of text files:} 
Given a sampling budget of $n$ bytes, consider the problem of  designing codes with minimum average length for text files in $K$ different languages. Since different languages may have different symbol frequencies, this can be formulated as learning $K$ distributions uniformly well in terms of certain $f$-divergences.

The problem considered in this paper lies in the general area of active learning with bandit feedback. Prior work in this area include~\cite{antos2008active,carpentier2011upper,Carpentier11FT,soare2014best,Riquelme17AL}.~\cite{antos2008active} studied the problem of learning the mean of $K$ distributions uniformly well, and proposed and analyzed an algorithm based on \emph{forced exploration} strategy.~\cite{carpentier2011upper} proposed and analyzed an alternative approach for the same problem, based on the UCB algorithm~\citep{auer2002finite}.~\cite{Carpentier11FT} analyzed an optimistic policy for the related problem of \emph{stratified-sampling}, where the goal is to learn $K$ distributions in terms of a weighted average distance (instead of max).~\cite{Soare2013Active} extended the optimistic strategy to the case of uniformly estimating $K$ linearly correlated distributions. \cite{Riquelme17AL} applied the optimistic strategy to the problem of allocating covariates (drawn in an i.i.d.~manner from some distribution) for uniform estimation of $K$ linear models.

The prior work mentioned above have focused on estimating the means of distributions in squared error sense, and those techniques do not extend to learning entire distributions in terms of other distance measures. 
We fill this gap in the literature, and provide a unified and comprehensive treatment of the problem of uniformly estimating $K$ distributions in terms of various distance measures. 

\tbf{Contributions.} 
In Sec.~\ref{sec:general_tracking}, we first identify an appropriate function class $\mc{F}$  within which  the \emph{objective functions}  for various distance measures should lie. We then propose a generic optimistic tracking strategy~(Alg.~\ref{algo:optimistic_tracking1}) for objective functions in $\mc{F}$,  and obtain a bound on its deviation from an (approx-) oracle allocation (defined in Sec.~\ref{sec:general_tracking}). Then in Sec.~\ref{sec:algorithms}, we apply the general approach outlined in Sec.~\ref{sec:general_tracking} to the case of four widely-used distance measures: $\ell_2^2$, $\ell_1$, $f$-divergence, and separation distance. We do so by deriving the objective functions (either exact or approximate) in $\mc{F}$ for these distances, and then instantiating Alg.~\ref{algo:optimistic_tracking1} for the derived objective functions. For each distance measure, we obtain bounds on the regret of the proposed sampling scheme w.r.t.~an oracle strategy. Experiments with synthetic examples in Sec.~\ref{sec:experiments} validate our theoretical results. Finally, while we focus on minimizing the maximum error in learning a set of distributions in this paper, our techniques can easily deal with the related setting of minimizing the average error, which we discuss in Sec.~\ref{sec:discussion}. 

\tbf{Technical tools.} The results of this paper require generalizing existing techniques, as well as introducing new methods. 
More specifically, the proof of Lemma~\ref{lemma:tracking1} abstracts out the arguments of~\cite[Thm.~1]{carpentier2011upper} to deal with a much larger class of objective functions.
Prior work with mean-squared error~\citep{antos2008active, carpentier2011upper} required bounding the first and second moments of random sums that could be achieved by a direct application of Wald's equations~\citep[Thm.~4.8.6]{durrett2019probability}. Our results on $f$-divergence~(Thm.~\ref{theorem:regret_f_div} and Lemma~\ref{lemma:f_div_moments}) require analyzing higher moments of random sums for which Wald's equations are not applicable. Deriving the approximate objective function for separation distance involves estimating the expectation of the maximum of some correlated random variables. We obtain upper and lower bounds on this expectation in Lemma~\ref{lemma:separation1} by first approximating the maximum with certain sums, and then bounding the sums using a normal approximation result~\citep[Thm.~3.2]{ross2011fundamentals}.

%% file: sec2_problem_setup.tex
Consider $K$ discrete distributions, $(P_i)_{i=1}^K$, that belong to the $(l-1)$-dimensional probability simplex $\Delta_l$, and take values in the set $\mc{X} = \{x_1, \ldots, x_l\}$. Each distribution $P_i$ is equivalently represented by a vector $P_i=(p_{i1},\ldots,p_{il})$ with $p_{ij} \geq 0,\;\forall j\in[l]$, and $\sum_{j=1}^l p_{ij} =1$. For any integer $b>0$, we denote by $[b]$, the set $\{1,\ldots,b\}$. Given a budget of $n \geq K$ samples, we consider the problem of allocating samples to each of the $K$ distributions in such a way that the maximum (over the $K$ distributions) discrepancy between the empirical distributions (estimated from the samples) and the true distributions is minimized.\footnote{The allocation strategies developed in this paper can be extended to the problem of minimizing the average (instead of minimizing the max) discrepancy between estimate and true distributions (see the discussion in Sec.~\ref{sec:discussion}).} To formally define this problem, suppose an allocation scheme assigns $(T_i)_{i=1}^K$ samples to the $K$ distributions, such that $T_i\geq 0,\;\forall i\in [K]$, and $\sum_{i=1}^K T_i = n$. Also suppose that $\hat{P}_i$ is the empirical distribution with $\hat{p}_{ij} = T_{ij}/T_i$, where $T_{ij}$ denotes the number of times the output $x_j$ was observed in the $T_i$ draws from $P_i$, and $D:\Delta_l \times \Delta_l \mapsto [0, \infty)$ is a distribution distance measure. Then, our problem of interest can be defined as finding an allocation scheme $(T_i)_{i=1}^K$ that solves the following constrained optimization problem:

\begin{equation}
\begin{aligned}
\underset{T_1,\ldots,T_K}{\text{ min}} \; \max_{i\in[K]} \;\mbb{E}\big[D(\hat{P}_i, P_i)\big], \;\;\; \text{s.t.} \;\; \sum_{i=1}^K T_i = n. 
\end{aligned}
\label{eq:objective1}
\end{equation}

We refer to the (non-integer) solution of~\eqref{eq:objective1} with full knowledge of $(P_i)_{i=1}^K$ as the \emph{oracle} allocation $(T^*_i)_{i=1}^K$. It is important to note that $(T^*_i)_{i=1}^K$ ensure that the objective functions $\gamma_i(T_i) \coloneqq \mbb{E}\big[D(\hat{P}_i,P_i)\big]$ are equal, for all $i\in[K]$. However, in practice, $(P_i)_{i=1}^K$ are not known. In this case, we refer to~\eqref{eq:objective1} as a \emph{tracking problem} in which the goal is to design adaptive sampling strategies that approximate the oracle allocation using running estimates of $(P_i)_{i=1}^K$.

\tbf{Choice of the Distance Measure.}
It is expected that the optimal allocation will be strongly dependent on the distance measure $D$.
We study four distances: $\ell_2^2$, $\ell_1$ or total variation (TV), $f$-divergence, and \emph{separation distance} in this paper. These distances include all those in~\cite{gibbs2002choosing} that do not require a metric structure on $\X$. 
The $f$-divergence family generalizes the well-known KL-divergence~($D_{\text{KL}})$ and includes a number of other common distances, such as total variation~($D_{\text{TV}})$, Hellinger ($D_H$), and $\chi^2$ ($D_{\chi^2}$). Applications of $f$-divergence include 
source and channel coding problems~\citep{csiszar1967source,csiszar1995channel}, testing goodness-of-fit~\citep{gyorfi2000goodness}, and distribution estimation~\citep{barron1992distribution}. The common $f$-divergences mentioned above satisfy the following chain of inequalities: $D_{\text{TV}} \leq D_H \leq \sqrt{D_{\text{KL}}} \leq \sqrt{D_{\chi^2}}$, that define a hierarchy of convergence among these measures~\citep[Eq.~2.27]{tsybakov2009lower}. 
The separation distance $D_{s}(P,Q)$~(defined formally in Sec.~\ref{subsec:separation_distance}) is the smallest $a\geq 0$, such that $P = (1-a)Q + aV$, for some distribution $V$~\citep[Sec.~2]{aldous1987strong}. This distance arises naturally in the study of the convergence of symmetric finite Markov chains to their stationary distribution. More specifically, if $Q$ is the stationary distribution of the Markov chain and $(P_t)_{t \geq 1}$ is its state distribution at time $t$, such that $Q=P_T$ at a random time $T$, then $D_s(P_t, Q) \leq \mathbb{P}(T>t)$~\citep[Sec.~3]{aldous1987strong}. 

\tbf{Allocation Scheme and Regret.} 
An \emph{adaptive} allocation scheme $\mc{A}$ consists of a sequence of mappings $(\pi_t)_{t \geq 1}$, where each mapping $\pi_t:\big(\mbb{N} \times (\mc{X}\times [K])^{t-1}\big) \mapsto [K]$ selects an arm to pull\footnote{Each distribution can be considered as an arm, and thus, we use the terms sampling from a distribution and pulling an arm interchangeably throughout the paper.} at time $t$, based on the budget $n$ and the history of pulls and observations up to time $t$. 
%
For an allocation scheme $\mc{A}$, a sampling budget $n$, and a distance measure $D$, we define the {\em risk} incurred by $\mc{A}$ as 
\begin{equation}
\label{eq:risk}
\mc{L}_n(\mc{A}, D) =\max_{i\in[K]} \mbb{E}\big[D(\hat{P}_i, P_i)\big].
\end{equation}

We denote by $\mc{A}^*$, the \emph{oracle} allocation rule.
 The performance of an allocation scheme $\mc{A}$ is measured by its suboptimality or \emph{regret} w.r.t.~$\mc{A}^*$, i.e., 
\begin{equation}
\mc{R}_n(\mc{A}, D) \coloneqq  \mc{L}_n(\mc{A}, D) - \mc{L}_n(\mc{A}^*, D). 
\label{eq:regret}
\end{equation}


%
\tbf{Notations.\footnote{We list all the notations used throughout the paper in Table~\ref{tab:table_of_notation} in Appendix~\ref{appendix:table_of_symbols}.}} 
For $0<\eta<1/2$, we define the $\eta$-interior of $(l-1)$-dimensional simplex $\Delta_l$, as $\Delta_l^{(\eta)} \coloneqq \{P \in \Delta_l \mid \eta \leq p_j \leq 1-\eta, \; \forall j\in[l]\}$. 
%
%
We use the Bernoulli random variable $Z_{ij}^{(s)}$ to represent the indicator that the $s^{th}$ draw from arm $i$ is equal to $x_j \in \mc{X}$. Note that for any draw $s$, we have $\mbb{E}[Z_{ij}^{(s)}]=p_{ij}$. For any $t \in [n]$, we define $W_{ij,t} = \sum_{s=1}^t \tilde{Z}_{ij}^{(s)}$, where $\tilde{Z}_{ij}^{(s)} \coloneqq Z_{ij}^{(s)} - p_{ij}$ is a centered Bernoulli variable. 
%
%
We also note that several terms such as $\f$, $A$, $B$, and $\tilde{e}_n$ (to be introduced in Sec.~\ref{sec:general_tracking}) are overloaded for different distance measures. For instance, we use $\f$ for both $\ell_1$ and KL-divergence, instead of writing $\f^{(\ell_1)}$ and $\f^{(\text{KL})}$. The meaning should be clear from the local context.

%% file: general_tracking.tex

Before proceeding to the analysis of specific distance measures, we first study an abstract class of problems similar to~\eqref{eq:objective1}, in which the objective functions satisfy certain regularity conditions that we define next. 

\begin{definition}
\label{def:regular}
We say a  function $\f: \mbb{R}\times \mbb{R} \mapsto \mbb{R}$ is {\bf\em regular}, if it has the following properties: {\bf 1)} $\f(\cdot, T)$ is concave and non-decreasing for all $T \in \mbb{R}$, {\bf 2)} $\f(c,\cdot)$ is convex and non-increasing for all $c \in \mbb{R}$, and {\bf 3)} $\f(c, \cdot)$ and $\f(\cdot, T)$ are differentiable for all $c, T \in \mbb{R}$. We denote by $\mathcal{F}$, the class of such functions.
\end{definition}
%

An analog of the optimization problem~\eqref{eq:objective1} when its objective function belongs to $\mc{F}$ can be stated as 

\begin{equation}
\label{eq:tracking1}
\underset{T_1,\ldots,T_K}{\text{ min}} \;\max_{i\in[K]} \; \f(c_i, T_i), \quad\; \text{s.t. } \;\; \sum_{i=1}^K T_i = n, 
\end{equation}

where the parameters $(c_i)_{i=1}^K$ depend on the distance $D$ and distributions $(P_i)_{i=1}^K$. We refer to the solution of~\eqref{eq:tracking1} with full knowledge of $(c_i)_{i=1}^K$ as $(\Tins)_{i=1}^K$, and to the corresponding allocation scheme as $\Algapprox$. Similar to~\eqref{eq:objective1}, when parameters $(c_i)_{i=1}^K$ are unknown, we refer to~\eqref{eq:tracking1} as a \emph{general tracking problem}.  

\tbf{Optimistic Tracking Algorithm.}
We now propose and analyze an adaptive sampling scheme, motivated by the upper-confidence bound (UCB) algorithm~\citep{auer2002finite} in multi-armed bandits, for solving the general tracking problem~\eqref{eq:tracking1}. 
The proposed scheme, whose pseudo-code is shown in Algorithm~\ref{algo:optimistic_tracking1}, samples \emph{optimistically} by plugging in high probability upper-bounds of $c_i$ in the objective function $\f$. Formally, for each arm $i\in[K]$ and time $t\in[n]$, we denote by $T_{i,t}$, the number of times that arm $i$ has been pulled prior to time $t$. We define the $(1-\delta)$-probability (high probability) event $\mc{E}\coloneqq\bigcap_{t\in[n]}\bigcap_{i\in[K]}\big\{|\hat{c}_{i,t} - c_i|\leq e_{i,t})\big\}$, where $\hat{c}_{i,t}$ is the empirical estimate of $c_i$ and $e_{i,t}$ is the radius (half of the length) of its confidence interval at time $t$ computed using $T_{i,t}$ samples. 
We define the upper-bound of $c_i$ at time $t$ as $u_{i,t} \coloneqq \hat{c}_{i,t} + e_{i,t}$ with the convention that $u_{i,1} = + \infty$. In the rest of the paper, we use $\hat{P}_{i,t}$ and $\hat{p}_{ij,t}$ to represent the estimates of $P_i$ and $p_{ij}$ at time $t$, computed by $T_{i,t}$ i.i.d.~samples.


\begin{algorithm}[H]
\caption{Optimistic Tracking Algorithm}
\label{algo:optimistic_tracking1}
\begin{algorithmic}[1]
\STATE {\bf Input:} {$K,n,\delta$}
\STATE Initialize $t \leftarrow 1$;
\WHILE{$t \leq n$}
\IF{$t\leq K$}
\STATE $I_t = t$;
\ELSE 
\STATE $I_t = \argmax_{i\in[K]} \; \f(u_{i,t}, T_{i,t})$; \label{line:arm-select-rule}
\ENDIF
\STATE Observe $X\sim P_{I_t}$; $\qquad$ $t \leftarrow t+1$;
\STATE Update $u_{i,t},\;\forall i\in[K]$;
\ENDWHILE
\end{algorithmic}
\end{algorithm}

We now state a lemma that bounds the deviation of the allocation obtained by Algorithm~\ref{algo:optimistic_tracking1} (our optimistic tracking algorithm), $(T_i)_{i=1}^K$, from the allocation $(\Tins)_{i=1}^K$, i.e.,~the solution to~\eqref{eq:tracking1} when the parameters $(c_i)_{i=1}^K$ are known. Before stating the lemma, we define $g_i^* \coloneqq \frac{\partial \f(c, \tilde{T}_i^*)}{\partial c} \big \rvert_{c=c_i}$ and $h_i^* \coloneqq \frac{\partial \f(c_i, T)}{\partial T} \big \rvert_{T=\tilde{T}_i^*}$. 

\begin{lemma}
\label{lemma:tracking1}
Define $A \coloneqq \max_{i\in[K]} g_i^*$, $B \coloneqq \left \lvert \max_{i\in[K]} h_i^* \right \rvert$, and $\tilde{e}_n \coloneqq \max_{i\in[K]} e^*_i$, where $e^*_i$ is the radius of the confidence interval of arm $i$ after $\Tins$ pulls. Then, under the event $\mc{E}$, and assuming that $B>0$ and $\Tins >1,\;\forall i\in[K]$, we have 
\begin{align*}
 - \frac{2A \tilde{e}_n }{B} \leq  \; T_{i} - \Tins \; \leq   \frac{2 A (K-1) \tilde{e}_n}{B}, \quad \forall i \in [K]. 
\end{align*}
\end{lemma}
\noindent The proof of Lemma~\ref{lemma:tracking1}, given  in Appendix~\ref{proof_ell2_lemma1},  generalizes the arguments used in~\cite[Thm.~1]{carpentier2011upper} to  handle any regular objective function $\f$. 

The above algorithm and lemma provide us with a road-map to design adaptive sampling algorithms for the tracking problem~\eqref{eq:objective1} for different choices of distribution distance $D$. The first step is to convert~\eqref{eq:objective1} to the corresponding general tracking problem~\eqref{eq:tracking1} by selecting the appropriate regular objective function $\f$. When the objective functions $(\gamma_i)_{i=1}^K$ in~\eqref{eq:objective1} can be computed in closed-form and belong to $\mc{F}$, which is the case for the $\ell_2^2$-distance, we set $\f(c_i,T_i)=\gamma_i(T_i)$. When the objective functions $(\gamma_i)_{i=1}^K$ can be computed in closed-form, but do not belong to $\mc{F}$, which is the case for $\ell_1$-distance, or when they cannot be computed in closed-form, which is the case for the $f$-divergence and the separation distance, we set $\f(c_i, T_i)$ to a good approximation of $\gamma_i(T_i)$. In case $\f$ is an approximate objective function, we shall refer to the solution of~\eqref{eq:tracking1}, $(\Tins)_{i=1}^K$, as the \emph{approx-oracle} allocation, and to $\Algapprox$ as the \emph{approx-oracle} allocation rule. 

Once an appropriate $\f(c_i,T_i)$ for a distance $D$ has been derived, the next step is to construct a high probability upper-bound $u_{i,t}$ for the parameter $c_i$. We can then instantiate Alg.~\ref{algo:optimistic_tracking1} for distance $D$ with the selected $\f$ and constructed upper-bound $u_{i,t}$. Finally, Lemma~\ref{lemma:tracking1} plays a key role in the regret analysis of the allocation scheme resulted from each instantiation of Alg.~\ref{algo:optimistic_tracking1}.

%% file: sec3_Algorithms.tex
Following the discussion at the end of Section~\ref{sec:general_tracking}, we now construct and analyze adaptive allocation schemes for the four distance measures considered in this paper.  

\subsection{Adaptive Allocation for  $\ell_2^2$-Distance}
\label{subsec:ell_2}
\input{ell2_loss.tex}

\subsection{Adaptive Allocation for $\ell_1$-Distance}
\label{subsec: ell_1}
\input{ell1_loss.tex}
\subsection{Adaptive Allocation for $f$-Divergence}
\label{subsec:f_divergence}
\input{f_divergence.tex}

\subsection{Adaptive Allocation for Separation Distance}
\label{subsec:separation_distance}
\input{separation_loss.tex}

%% file: ell2_loss.tex
The squared $\ell_2$-distance between two distributions $P$ and $Q$ is defined as $\Dltwo(P,Q) \coloneqq \sum_{j=1}^l(p_j - q_j)^2$. In this case, we can compute the {\bf objective function} of~\eqref{eq:objective1} in closed-form as 
\begin{align}
\gamma_i(T_i) &= \mbb{E}\big[\Dltwo(\hat{P}_i , P_i)\big] =  \mbb{E}\big[\sum_{j=1}^l (\hat{p}_{ij} - p_{ij})^2\big] \nonumber \\
&= \sum_{j=1}^l \frac{p_{ij}(1-p_{ij})}{T_i} \coloneqq  \frac{c_i^{(\ell_2)}}{T_i} \coloneqq  \f(c_i^{(\ell_2)}, T_i). \nonumber
\end{align}

Note that the function $\gamma_i(T_i) = \f(\ciltwo, T_i) = \ciltwo/T_i$ belongs to the class of regular functions $\mc{F}$. The {\bf oracle allocation} is obtained by equalizing $\ciltwo/T_i$, for all $i\in[K]$, and can be written as $T_i^* = \Tins = (\ciltwo n)/(\sum_{k=1}^K c_k^{(\ell_2)})  \coloneqq \lambda_i^{(\ell_2)}\times n$. We denote by $\lambda_{\min}^{(\ell_2)}$ and $\lambda_{\max}^{(\ell_2)}$, the min and max of $\lambda_i^{(\ell_2)}$ over the arms.


The {\bf adaptive allocation scheme} uses optimistic estimates of the $c_i^{(\ell_2)}$ terms to define the arm-selection rule in Algorithm~\ref{algo:optimistic_tracking1} (Line~\ref{line:arm-select-rule}). More specifically, for a given $\delta>0$, we first define the following event.

\begin{lemma}
\label{lemma:ell2_1}
Define \begin{small}$\delta_t \coloneqq 6 \delta /(K l \pi^2 t^2)$\end{small}, \begin{small}$e^{(\ell_2)}_{ij, t} \coloneqq \sqrt{3 p_{ij} \log (2/\delta_t)/\Tit}$\end{small}, and the event 
\[\mathcal{E}_1 \coloneqq  \bigcap_{t\in[n]}  \bigcap_{i\in[K]} \bigcap_{j\in[l]}\{|\hat{p}_{ij, t} - p_{ij}| \leq e^{(\ell_2)}_{ij, t}\}.\] 
 Then, we have \begin{small}$\mathbb{P}(\mc{E}_1) \geq 1- \delta$\end{small}.
\end{lemma}

Lemma~\ref{lemma:ell2_1}, proved in Appendix~\ref{appendix:analysis_ell2}, follows from an application of the relative Chernoff bound~\citep{hagerup1990guided}.  Next, we use Lemma~\ref{lemma:ell2_1} to obtain the required confidence intervals for the parameter $\ciltwo$. 

\begin{lemma}
\label{lemma:ell2_3}
With probability at least $1-\delta$, for all $t\in[n]$ and $i\in[K]$, we have
$  |c_i^{(\ell_2)} - \hatciltwo |     \leq \sqrt{ 27 \log(1/\delta_t)/\Tit} \coloneqq e^{(\ell_2)}_{i,t}.$
\end{lemma}
Lemma~\ref{lemma:ell2_3} allows us to define high probability upper-bounds on the parameters $c_i^{(\ell_2)}$ as $u_{i,t}^{(\ell_2)} \coloneqq \hatciltwo + e^{(\ell_2)}_{i,t}$. These upper-bounds can then be plugged into Algorithm~\ref{algo:optimistic_tracking1} to obtain an adaptive sampling scheme for the $\ell_2$-distance, which we shall refer to it as $\Algltwo$.

We can now state the bound on the regret incurred by the allocation scheme $\Algltwo$~(proof in Appendix~\ref{appendix:analysis_ell2}).
\begin{theorem}
\label{theorem:ell2_regret}
Define 
$\Mltwo \coloneqq \frac{\lambda_{\max}^{(\ell_2)} \sqrt{2\log(1/\delta_n)}}{\lambda_{\min}^{(\ell_2)}\sum_{i=1}^K c_i^{(\ell_2)}}$ with $\delta_n$ from Lemma~\ref{lemma:ell2_1}. If we implement the algorithm $\Algltwo$ with a budget $n$ and $\delta=n^{-5/2}$, then for $n\geq 4\Mltwo^2/(\lambda_{\min}^{(\ell_2)})^2$, we have 
\begin{align*}
\mc{R}_n \lp \Algltwo , \Dltwo \rp & = 
\frac{(K+5)\Mltwo l}{(\lambda_{\min}^{(\ell_2)})^2 \; n^{3/2}} + \mc{O}\big(\frac{\delta l }{\sqrt{n}} + \frac{\log(n)}{n^2}\big)\;. 
\end{align*}
\end{theorem}

The $\tilde{\mc{O}}(n^{-3/2})$ convergence rate of Theorem~\ref{theorem:ell2_regret} recovers the rate derived in~\citet{carpentier2011upper} for the special case of Bernoulli $P_i$.

%% file: ell1_loss.tex
The $\ell_1$-distance between two distributions $P$ and $Q$ is defined as $\Dlone(P,Q) \coloneqq \sum_{j=1}^l |p_j - q_j|$. Note that the total-variation distance, $\Dtv$, is related to $\Dlone$ as $\Dtv= \frac{1}{2}\Dlone$. In this case, the objective function $\gamma_i$ can be obtained in closed-form using the expression for \emph{mean absolute deviation} $\mbb{E}[|\hat{p}_{ij} - p_{ij}|]$ given in~\citet[Eq.~1.1]{diaconis1991closed}. However, since this closed-form expression does not belong to $\mc{F}$, we first obtain an approximation of $\gamma_i$ in $\mathcal{F}$ as 
\begin{align}
& \gamma_i(T_i) =  \mbb{E}\big[ \Dlone(\hat{P}_i, P_i)\big] \coloneqq  \mbb{E}\big[\sum_{j=1}^l |\hat{p}_{ij} - p_{ij}|\big] \nonumber \\
&\stackrel{\text{(a)}}{\leq}  \sum_{j=1}^l \sqrt{\mbb{E}\big[(\hat{p}_{ij} - p_{ij})^2\big]} = \frac{1}{\sqrt{T_i}}\sum_{j=1}^l\sqrt{p_{ij}(1-p_{ij})} \nonumber \\ 
&\coloneqq  \frac{c_i^{(\ell_1)}}{\sqrt{T_i}} \coloneqq  \f(\cilone, T_i).
\label{eq:ell1_1}
\end{align} 

{\bf (a)} follows from the Jensen's inequality and the concavity of the square-root function. We can check that the {\bf approximate objective function} $\f(\cilone, T_i) = \cilone/\sqrt{T_i}$ with $\cilone = \sum_{j=1}^l \sqrt{p_{ij}(1-p_{ij})}$ lies in $\mc{F}$. 


The {\bf approx-oracle allocation}  is given by $\Tins = ((\cilone)^2 n)/\Clone^2 \coloneqq \lambda_i^{(\ell_1)}\times n$, where $\Clone^2=\sum_{i=1}^K ( \cilone)^2$.
In order to obtain the {\bf adaptive allocation scheme} for the $\ell_1$-distance, which we shall refer to as $\Alglone$, we now derive high probability upper-bounds on $(c_i^{(\ell_1)})_{i=1}^K$ and then plug them into Algorithm~\ref{algo:optimistic_tracking1}. 
\begin{lemma}
\label{lemma:ell1_1} 
Define $\delta_t \coloneqq 3\delta/(Kl\pi^2t^2)$, $e^{(\ell_1)}_{ij,t}\coloneqq\sqrt{2\log(2/\delta_t)/\Tit}$, and the event  $\mathcal{E}_2$ as follows:
\[\mathcal{E}_2 \coloneqq \bigcap_{t\in[n]}\bigcap_{i\in[K]}\bigcap_{j\in[l]}\big\{ |\sqrt{\hat{p}_{ij,t}(1-\hat{p}_{ij,t}) }-\sqrt{p_{ij}(1-p_{ij})}|\leq e^{(\ell_1)}_{ij,t}\big\}\]. 
%
%
Then, we have $\mathbb{P}(\mc{E}_2) \geq 1- \delta$.
\end{lemma}
\noindent The proof~(details in Appendix~\ref{proof_ell1_lemma1}) relies on an application of a concentration inequality of the standard deviation of random variables derived in~\citet[Thm.~10]{maurer2009empirical}, followed by two union bounds.
Lemma~\ref{lemma:ell1_1} allows us to define high probability upper-bounds on the parameters $c_i^{(\ell_1)}$ as $u_{i,t}^{(\ell_1)} \coloneqq \hatcilone + e^{(\ell_1)}_{i,t}$, where $e^{(\ell_1)}_{i,t} = \sum_{j=1}^l e^{(\ell_1)}_{ij,t} = \sqrt{2 l^2 \log(2/\delta_t)/\Tit}$.


We now state the regret bound for the adaptive allocation scheme $\Alglone$ (proof in Appendix~\ref{proof_ell1_theorem}). 

\begin{theorem}
\label{theorem:ell1_regret}
Define $\Mlone = \frac{2l\lambda_{\max}^{(\ell_1)}\sqrt{\log(1/\delta_n)}}{\Clone\lambda_{\min}^{(\ell_1)}}$ with $\delta_n$ from Lemma~\ref{lemma:ell1_1}. If we implement the algorithm $\Alglone$ with budget $n$ and $\delta=1/n$, then for $n \geq 6(K-1)\Mlone^2 /(\lambda_{\min}^{(\ell_1)})^3$, we have
\begin{equation}
\label{eq:regret-L1}
\mc{R}_n(\Alglone, \Dlone) = \frac{l\sqrt{(K+5)\Mlone}}{2\;\lambda_{\min}^{(\ell_1)}\;n^{3/4}} + \tilde{\mc{O}}(1/n).
\end{equation}
The exact expressions for the higher order terms are given in~\eqref{eq:ell1_regret1} in Appendix~\ref{proof_ell1_theorem}.  
\end{theorem}

Since here, as well as in Sections~\ref{subsec:f_divergence} and~\ref{subsec:separation_distance}, the objective function $\f$ used in~\eqref{eq:tracking1} is an approximation of the objective function in~\eqref{eq:objective1}, the regret of the resulting allocation scheme, $\mc{R}_n(\mc{A}_{\ell_1}, D_{\ell_1})$, can be decomposed into the regret w.r.t.~the approx-oracle allocation rule (regret in tracking the approximate objective), $\mc{L}_n(\mc{A}_{\ell_1},D_{\ell_1}) - \mc{L}_n(\mc{\tilde{A}}^*_{\ell_1},D_{\ell_1})$, plus a term that depends on the accuracy of the approximation, $\gamma_i(T_i)-\f(T_i)$, (see Proposition~\ref{prop:approximate_regret} in Appendix~\ref{appendix:approximate} for a formal statement and proof). 

%% file: f_divergence.tex
For a convex function $f: \mbb{R} \mapsto \mbb{R}$ satisfying $f(1) = 0$, the $f$-divergence between two distributions $P$ and $Q$ is defined as $\Dfdiv\lp P, Q\rp \coloneqq \sum_{j=1}^l q_j f(p_j/q_j)$. Since we cannot obtain a closed-form expression for the objective function $\gamma_i$ of $f$-divergence, we proceed by writing $\Dfdiv(\hat{P}_i, P_i) = \Dfdiv^{(r)}(\hat{P}_i, P_i) + R_{i, r+1}$, where $\Dfdiv^{(r)}(\hat{P}_i, P_i)$ is the $r$-term Taylor's approximation of $\Dfdiv(\hat{P}_i, P_i)$, i.e.,

\begin{equation}
\Dfdiv^{(r)}(\hat{P}_i, P_i) \coloneqq \sum_{m = 1}^r\frac{f^{(m)}(1)}{m!}\sum_{j=1}^l \frac{1}{p_{ij}^{m-1}}(\hat{p}_{ij} - p_{ij})^m,  
\label{eq:r_term_approx}
\end{equation}

and $R_{i,r+1}=\sum_{j=1}^l R_{ij, r+1}$ is its remainder term (approximation error), i.e.,

\begin{equation}
R_{ij, r+1} \coloneqq \sum_{m=r+1}^\infty \frac{f^{(m)}(1)}{m!\;p_{ij}^{m-1}}(\hat{p}_{ij} - p_{ij})^m.  
\label{eq:remainder} 
\end{equation}

Note that in~\eqref{eq:r_term_approx} and~\eqref{eq:remainder}, $f^{(m)}(\cdot)$ is the $m^{th}$ derivative of $f$. We now define the {\bf approximate objective function} for an $f$-divergence as 

\begin{align}
\label{eq:approx-obj-func-f-div}
\f_i\lp T_i\rp &\coloneqq \mbb{E}[ \Dfdiv^{(r)}(\hat{P}_i, P_i)] \nonumber \\ 
&= \sum_{m=1}^r \sum_{j=1}^l \frac{f^{(m)}(1)}{m!\; p_{ij}^{m-1}\;T_i^m} \mbb{E}\Big[\big(\sum_{s=1}^{T_i} \tilde{Z}_{ij}^{(s)}\big)^m\Big]. 
\end{align}

Here, we use $\f_i(T_i)$, instead of the usual $\varphi(c_i, T_i)$, to denote the \emph{regular} objective function, since deriving the exact parameters characterizing $\mbb{E}[D^{(r)}_f (\hat{P}_i, P_i )]$ requires knowledge of $f^{(m)}(1)$. We present an instance of this derivation for KL-divergence in Section~\ref{sub2sec:kl_divergence}. Also recall from Section~\ref{sec:problem_setup} that the term in the expectation in~\eqref{eq:approx-obj-func-f-div} can be written as $W_{ij,T_i} = \sum_{s=1}^{T_i} \tilde{Z}_{ij}^{(s)}$. 

Next, in Lemma~\ref{eq:lemma:f_div_remainder} (proof in Appendix~\ref{subsubsec:proof-lemma5}), we show that the error between the true and approximate objective functions, i.e.,~$\mbb{E}[\Dfdiv(\hat{P}_i,P_i)] - \f_i(T_i)$ is of $\mc{O}(T_i^{-(r+1)/2})$. For stating Lemma~\ref{eq:lemma:f_div_remainder}, we assume that the function $f$ satisfies two properties: \tbf{(f1)} it is bounded on any compact subset of its domain, and \tbf{(f2)} $f^{(m)}(1)/m! \leq C_1<\infty,\;\forall m\in\mbb{N}$. These assumptions hold for most commonly used $f$-divergences, namely KL-divergence with $f(x) = x\log x$, $\chi^2$-divergence with $f(x) = (x-1)^2$, and Hellinger distance with $f(x) = 2(1- \sqrt{x})$. 
%
\begin{lemma}
\label{eq:lemma:f_div_remainder}
Assume that $f$ satisfies \tbf{(f1)} and \tbf{(f2)}. 
Then, there exists a constant $C_{f, r+1}< \infty$, whose exact definition is given by Eqs.~\ref{eq:remainder1},~\ref{eq:remainder2}, and~\ref{eq:remainder4} in Appendix~\ref{subsubsec:proof-lemma5}, such that the following holds:

\begin{equation*}
\mbb{E} \lb R_{ij, r+1} \rb \leq C_{f, r+1} (p_{ij}T_i)^{-(r+1)/2}. 
\end{equation*}

\end{lemma}
To retrace the approach employed in Sections~\ref{subsec:ell_2} and~\ref{subsec: ell_1}, and apply the general adaptive algorithm~(Alg.~\ref{algo:optimistic_tracking1}), we first need to make sure that the approximate objective function $\f_i$ lies in $\mc{F}$, and then construct the requisite upper-bound. This step involves computations tailored to specific choices of $f$-divergence. Instead, we take an alternative approach and study the regret bound for any adaptive scheme for which we can define a particular high-probability event $\mc{E}_\delta$ (introduced below). This framework allows us to obtain a general decomposition of the regret into several components in Thm.~\ref{theorem:regret_f_div}. We study the particular case of KL-divergence in Sec.~\ref{sub2sec:kl_divergence}, for which we first show that $\f_i\in\mc{F}$ (see Eq.~\ref{eq:kl_objective}) and then construct the appropriate upper-bound necessary to implement Alg.~\ref{algo:optimistic_tracking1} (see Lemma~\ref{lemma:kl_1}). We show that for this adaptive scheme, we can obtain the required high probability event, and thus, employing the general regret decomposition of Thm.~\ref{theorem:regret_f_div}, we derive explicit regret bound in Thm.~\ref{theorem:kl}. Similar results can be obtained for other commonly used $f$-divergences such as Hellinger distance. 


We consider a general {\bf adaptive allocation scheme} that is applied with the approximate objective function $\f_i$ and satisfies the condition:~\tbf{(a1)} for any $\delta>0$, we can define a $(1-\delta)$-probability event $\mc{E}_\delta$ under which $\tau_{0,i} \leq T_{i} \leq \tau_{1,i}$, for $i\in[K]$. Here $\tau_{0,i}$ and $\tau_{1,i}$ are non-negative constants that depend on $n$ and $\delta$, and $T_i$ is the (random) number of times $\Algfdiv$ pulls arm $i$.


We now prove a lemma that bounds the moments of $W_{ij,T_i} = \sum_{s=1}^{T_i}\tilde{Z}_{ij}^{(s)}$
in the definition of $\f_i$ (Eq.~\ref{eq:approx-obj-func-f-div}) for an adaptive allocation scheme $\Algfdiv$ satisfying \tbf{(a1)}. 
\begin{lemma}
\label{lemma:f_div_moments}
Let $\Algfdiv$ be an allocation scheme that satisfies~{\bf\em (a1)} and $\mc{E}_\delta$ be the corresponding high probability event. 
Then, for $m \geq 1$ and $i\in[K]$, we have 

\begin{equation*}
\mbb{E}\big[(W_{ij,T_i})^m \indi{\mc{E}}\big] \leq \mbb{E}\big[(W_{ij,\tau_{0,i}})^m\big] + \beta_m^{(ij)}(\tau_{0,i},\tau_{1,i}),
\end{equation*}

where

\begin{equation*}
\beta_m^{(ij)}(\tau_{0,i},\tau_{1,i}) \coloneqq \sum_{k=1}^{m}(\tau_{1,i} - \tau_{0,i})^k {m \choose k}\mbb{E}\big[(W_{ij,\tau_{0,i}})^{m-k}\big].  
\end{equation*}
%
\end{lemma}
%

Lemma~\ref{lemma:f_div_moments}, proved in Appendix~\ref{proof_f_div_lemma2}, provides the bounds on the moments of $W_{ij, T_i}$ that will be used to upper-bound the term $\mbb{E}[ \Dfdiv^{(r)}(\hat{P}_i, P_i)]$ in the regret analysis of the adaptive scheme $\Algfdiv$ in Theorem~\ref{theorem:regret_f_div}, which we now state. 

\begin{theorem}
\label{theorem:regret_f_div}
Suppose $f$ satisfies~\tbf{(f1)} and~\tbf{(f2)}, and $\Algfdiv$ satisfies~\tbf{(a1)}. 
Then, we have 

\begin{equation}
\mc{R}_n(\Algfdiv,\Dfdiv) \leq \max_{i\in[K]}\big(
\f_i(\tau_{0,i})-\f_i(T_{i}^*) + \Psi_i \big),
    \label{eq:f_div_regret0}
\end{equation}

where $\Psi_i \coloneqq \lp\sum_{k=1}^3 \psi_{k,i}\rp + \psi_4$, and 

\begin{align*}
\psi_{1,i} &\coloneqq \big(\delta n^{ \frac{r+1}{2}   } + (\frac{\sqrt{n}}{\tau_{0,i}})^{r+1}\big)\sum_{j=1}^l\frac{C_{f, r+1}e^2\big(3.2(r+1)\big)^{\frac{r+1}{2}}}{p_{ij}^{r+1}}, \\ 
\psi_{2,i} &\coloneqq \sum_{m=1}^r\sum_{j=1}^l\frac{\beta_m^{(ij)}(\tau_{0,i},\tau_{1,i})}{m! \; p_{ij}^{m-1} \; \tau_{0,i}^m}, \quad \psi_{3,i} \coloneqq \sum_{m=1}^r \sum_{j=1}^l \frac{\delta}{m! \; p_{ij}^{m-1}},\\ 
\psi_4 &\coloneqq \max_{1\leq i\leq K}\lp\sum_{j=1}^l 4C_{f,r+1}(p_{ij}\Tins)^{-\frac{r+1}{2}}\rp. 
\end{align*}

\end{theorem}

\emph{Proof Outline.} From Eq.~\ref{eq:approx_reg2} in Proposition~\ref{prop:approximate_regret} in Appendix~\ref{appendix:approximate}, we have  $\mc{R}_n ( \Algfdiv, \Dfdiv) \leq \big(\mc{L}_n ( \Algfdiv, \Dfdiv) - \f_i(\Tins)\big) + 2 \max_{k \in [K]} |\mbb{E}[R_{k, r+1} (\Tins)]|$.
The second term is upper-bounded with $\psi_4$ by employing Lemma~\ref{eq:lemma:f_div_remainder}.
Next, we decompose $\mc{L}_n (\Algfdiv, \Dfdiv)$ into three terms: $\mbb{E}[\Dfdiv^{(r)}(\hat{P}_i, P_i)\indi{\mc{E}_\delta}]$, 
$\mbb{E}[\Dfdiv^{(r)}(\hat{P}_i, P_i)\indi{\mc{E}_\delta^c}]$ and 
$\mbb{E}[R_{i, r+1}(T_i)]$ and upper-bound them with 
$(\f_i(\tau_{0,i}) + \psi_{2,i})$, $\psi_{3,i}$ and $\psi_{1,i}$ respectively. 
The details of these steps are given in Appendix~\ref{proof_f_div_theorem}.

\begin{remark}
\label{remark:regret_f_div}
The magnitude of the terms $\f_i(\tau_{0,i}) - \f_i(\Tins)$ and $\psi_{2,i}$ in~\eqref{eq:f_div_regret0} depend on how closely $(T_i)_{i=1}^K$ can match the approx-oracle allocation $(\widetilde{T}^*_i)_{i=1}^K$. 
The term  $\psi_4$ is of  $\mc{O}( n^{-(r+1)/2})$, while $\psi_{1,i} = \mc{O} ( \delta n^{(r+1)/2} + (\sqrt{n}/\tau_{0,i})^{r+1})$. 
Finally, the term  $\psi_{3,i}$ 
 is of $\mc{O} \lp \delta \rp$.
\end{remark}

\input{kl_loss.tex}

%% file: kl_loss.tex
\subsubsection{Adaptive Allocation for KL-Divergence}
\label{sub2sec:kl_divergence}
The KL-divergence between two distributions $P$ and $Q$ is defined as $\Dkl(P,Q) \coloneqq \sum_{j=1}^l p_j \log(p_j/q_j)$. Since $\Dkl$ is not symmetric w.r.t.~its arguments, we can obtain two objective functions depending on the position of the estimated distribution $\hat{P}_i$. We will focus on $\Dkl(\hat{P}_i,P_i)$ in this section, the derivations for the other form can be obtained similarly. 

We begin by deriving the $r$-term approximate objective function with $r=5$, i.e., 

\begin{align}
&\mbb{E}\big[\Dkl(\hat{P}_i, P_i)\big] = \mbb{E}\big[\sum_{j=1}^l \hat{p}_{ij} \log(\hat{p}_{ij}/p_{ij})\big] \nonumber \\
&= \sum_{j=1}^l \mbb{E}\big[\hat{p}_{ij}\log(\hat{p}_{ij})\big] - \sum_{j=1}^l \mbb{E}\lb \hat{p}_{ij}\rb \log(p_{ij}) \nonumber \\ 
&\stackrel{\text{(a)}}{=} \mbb{E}\big[H(P_i) - H(\hat{P}_i)\big] \nonumber \\
&\stackrel{\text{(b)}}{=} \frac{l-1}{2T_i} + \frac{1}{12T_i^2}\sum_{j=1}^l\big(\frac{1}{p_{ij}} - 1\big) + \mc{O}(1/T_i^3), \label{eq:kl_objective}
\end{align}

where \tbf{(a)} follows from the fact that $\mbb{E}\lb \hat{p}_{ij}\rb = p_{ij},\;\forall j\in[l]$, and \tbf{(b)} is obtained by calculating the Taylor's approximation of the mapping $x \mapsto x\log(x)$ up to the $5^{th}$ order. The calculations involved in this derivation are described in~\citet[Sec.~2]{harris1975statistical}. 

The choice of $r=5$ is sufficient as it is the smallest $r$ for which the approximation error, which is of~$\mc{O}(n^{-3})$, is smaller than the tracking regret, which is of~$\mc{O}(n^{-5/2})$ (see Proposition~\ref{prop:approximate_regret} in Appendix~\ref{appendix:approximate}). 

Eq.~\ref{eq:kl_objective} gives us the {\bf approximate objective function} $\f(c_i^{\text{(KL)}}, T_i) \coloneqq \frac{l-1}{2T_i} + \frac{c_i^{\text{(KL)}}}{T_i^2}$, with $c_i^{\text{(KL)}} \coloneqq \big(\sum_{j=1}^l 1/p_{ij} - 1\big)/12$. Note that this $\f$ belongs to the class of \emph{regular} functions $\mathcal{F}$.  


Deriving the {\bf approx-oracle allocation} $(\widetilde{T}_i^*)_{i=1}^K$ requires solving a cubic equation. Instead of computing the exact form of $\Tins$, we show in Lemma~\ref{lemma:kl2} in Appendix~\ref{appendix:analysis_kl} that the deviation of $\Tins$ from the uniform allocation is bounded by a problem-dependent constant, implying that the uniform allocation is near-optimal. 
This is not surprising as the first order approximation of $\f$ in~\eqref{eq:kl_objective} does not change with $P_i$.


Having obtained the approximate objective function $\f(c_i^{\text{(KL)}}, T_i)$ for the tracking problem~\eqref{eq:tracking1}, we now need to construct high probability confidence intervals for the parameters $c_i^{\text{(KL)}}$. We present the required concentration result under the assumption that the distributions lie in the $\eta$-interior of the $(l-1)$-dimensional simplex, $\Delta_l^{(\eta)}$, defined in Section~\ref{sec:problem_setup}.

\begin{lemma}
\label{lemma:kl_1}
For a given $\delta \in (0,1)$, define $e_{ij, t} = \sqrt{2\log(2/\delta_t)/\Tit}$ where $\delta_t$ was introduced in Lemma~\ref{lemma:ell2_1}. 
Assume that  $\hat{p}_{ij,t} \geq 7e_{ij,t}/2$. Then, the following inequalities hold for $P_i \in \Delta_l^{(\eta)}$: 
\begin{equation}
\label{eq:kl1}
\frac{1}{p_{ij}} \leq \frac{1}{\hat{p}_{ij, t} - e_{ij}} \leq \frac{1}{p_{ij}} + \frac{4e_{ij}}{\eta\;p_{ij}}\;,
\end{equation}
which implies that 

\begin{equation*}
\label{eq:kl2}
\sum_{j=1}^{l} \frac{1}{p_{ij}} \leq \sum_{j=1}^l \frac{1}{\hat{p}_{ij, t} - e_{ij}} \leq 
\sum_{j=1}^l \frac{1}{p_{ij}} + \frac{l}{\eta^{2}} \sqrt{ \frac{32l^2 \log(2/\delta_t)}{T_{i,t}}}. 
\end{equation*}

\end{lemma}

Lemma~\ref{lemma:kl_1} allows us to define high probability upper bounds on the parameters $c_i^{\text{(KL)}}$ as $u_{i,t}^{\text{(KL)}} \coloneqq \big(\sum_{j=1}^l\frac{1}{\hat{p}_{ij,t} - e_{ij,t}} - 1\big)/12$, if $\hat{p}_{ij,t} \geq 7e_{ij,t}/2$, and $u_{i,t}^{\text{(KL)}} = +\infty$, otherwise. These upper-bounds can then be plugged into Algorithm~\ref{algo:optimistic_tracking1} to obtain an {\bf adaptive allocation scheme} for KL-divergence, which we shall refer to as $\Algkl$.
Finally, we state the regret bound for $\Algkl$ in the following theorem (proof in Appendix~\ref{proof_kl_theorem}).
%
\begin{theorem}
\label{theorem:kl}
Let $n \geq 3K/\eta$ and $P_i \in \Delta^{(\eta)}_l,\;\forall i\in[K]$. Then, we have
%
\begin{align*}
\mc{R}_n(\Algkl,\Dkl)& \leq \frac{2(l-1)K\Mkl}{n^{5/2}} + \mc{O}(n^{-3}), 
\end{align*}
where $\Mkl = \sqrt{ 96K \log(1/\delta_n)}/\eta^3$.
\end{theorem}

\begin{remark}
\label{remark:kl}
The deviation of $\Tins$ from the uniform allocation is due to the second term in~\eqref{eq:kl_objective}, i.e.,~$\cikl/T_i^2$, and as mentioned earlier, is bounded by a constant~(see Lemma~\ref{lemma:kl2}). If we ignore this term and consider $\f(c_i, T_i) = (l-1)/2T_i$ as the approximate objective, then the uniform allocation emerges as the approx-oracle allocation. 
%
However, with $(l-1)/2T_i$ as the objective, we can only achieve a regret bound of $\mc{O}(n^{-2})$. Thus, in order to obtain tighter theoretical guarantee in Thm.~\ref{theorem:kl}, it is necessary to consider the higher order approximate objective, even though in practice, the resulting allocations are close. This is also observed in our experiments in Sec.~\ref{sec:experiments}, where the uniform and adaptive allocations perform equally well. 
\end{remark}

%% file: separation_loss.tex
The separation distance~\citep{gibbs2002choosing} between two distributions $P$ and $Q$ is defined as $D_{s}(P,Q) \coloneqq \max_{j\in[l]}(1 - p_j/q_j)$. We start by introducing new notation. 
Given a probability distribution $P_i \in \Delta_l$ and a {\em non-empty} set $S\subset [l]$, we define $p_{i,S} \coloneqq \sum_{j \in S} p_{ij}$. We also define the functions $\rho_1(p) \coloneqq \sqrt{(1-p)/p}$ and $\rho_2(p) \coloneqq \rho_1(p) + \rho_1(1-p)$, and introduce the terms $c_i^{(s)} \coloneqq \sum_{j=1}^l \rho_1(p_{ij})$ and $\tilde{c}_i^{(s)} \coloneqq \max_{S\subset [l]}\{\rho_2(p_{i,S})\}$. Note that $\tilde{c}_i^{(s)} = c_i^{(s)}$, for $l=2$. 



Because of the $\max$ operation in the definition of $D_s$, in general, we cannot obtain a closed-form expression for the objective function $\gamma_i(T_i)=\mbb{E}[D_s(\hat{P}_i,P_i)]$. We now state a key lemma (proof in Appendix~\ref{appendix:analysis_separation}) that provides an approximation of $\mbb{E}[D_s(\hat{P}_i,P_i)]$. 
\begin{lemma}
\label{lemma:separation1}
For a distribution $P_i\in\Delta_l$, let $\hat{P}_i = (\hat{p}_{ij})_{j=1}^l$ be the empirical distribution constructed from $T_i$ i.i.d.~draws from $P_i$. Then, we have 

\begin{equation*}
\tilde{c}_i^{(s)}\sqrt{\frac{1}{2\pi T_i}} - \frac{\tilde{C}_i^{(s)}}{T_i} \leq \mbb{E}\big[D_s(\hat{P}_i,P_i)\big] \leq c_i^{(s)} \sqrt{\frac{1}{2\pi T_i}}  +  \frac{C_i^{(s)}}{T_i}, 
\end{equation*}

where $C_i^{(s)}$ and $\tilde{C}_i^{(s)}$ are $P_i$-dependent constants defined in~\eqref{eq:Cis} and~\eqref{eq:CisTilde2} in Appendix~\ref{appendix:analysis_separation}. 
\end{lemma}

Lemma~\ref{lemma:separation1} gives us an interval that contains the true objective function we aim to track. To implement the adaptive scheme, we employ the {\bf approximate objective function} $\f(c_i^{(s)},T_i) \coloneqq c_i^{(s)}\sqrt{1/2\pi T_i}$.

In order to instantiate Algorithm~\ref{algo:optimistic_tracking1} for the separation distance, we require to derive high probability confidence intervals for the terms $\sqrt{(1-p_{ij})/p_{ij}}$ in the definition $(c_i^{(s)})_{i=1}^K$. We use the event $\mc{E}_1$ defined in Lemma~\ref{lemma:ell2_1} and prove the following result:
\begin{lemma}
\label{lemma:separation_loss1}
Let $P_i\in\Delta_l^{(\eta)}$ and the event $\mc{E}_1$, and the terms $\delta_t$ and $e_{ij,t}$ defined as in Lemma~\ref{lemma:ell2_1}. Define the terms $a_{i,t} \coloneqq \big(32\log(2/\delta_t)/\Tit\big)^{1/4}$ and $b_{i,t} \coloneqq \big(\frac{l\; a_{i,t}}{\eta}\big)\max\big\{1,\frac{a_{i,t}}{2\eta^{3/2}}\big\}$. 
Then, under the high probability event $\mc{E}_1$, we have 

\begin{equation*}
\sum_{j=1}^l \sqrt{ \frac{1}{p_{ij}} - 1} \leq \sum_{j=1}^l \sqrt{ \frac{1}{\hat{p}_{ij,t} - e_{ij,t}} - 1} \leq \sum_{j=1}^l \sqrt{ \frac{1}{p_{ij}} - 1} + b_{i,t}.
\end{equation*}

\end{lemma}

Using the concentration result of Lemma~\ref{lemma:separation_loss1}, we can now implement Algorithm~\ref{algo:optimistic_tracking1} with the upper-bound $u_{i,t}^{(s)} = \big(\sum_{j=1}^{l}\sqrt{\frac{1}{\hat{p}_{ij,t} - e_{ij,t}}-1}\big)/(\sqrt{2\pi})$, if $\hat{p}_{ij,t} \geq 7e_{ij,t}/2$, and $u_{i,t}^{(s)} = +\infty$, otherwise. This will give us an adaptive allocation scheme for the separation distance, which we shall refer to it as $\mc{A}_s$. Finally, We prove a regret bound for $\mc{A}_s$ (proof in Appendix~\ref{proof_separation_theorem}).


\begin{theorem}
\label{theorem:separation}
Let $P_i \in \Delta_l^{(\eta)}$ and the adaptive scheme $\mc{A}_s$ is implemented with $\delta = \eta/n$. Then, for large enough $n$ (the exact condition is given by Eq.~\ref{eq:n_large_enough_sep} in Appendix~\ref{proof_separation_theorem}), we have

\begin{small}
\begin{align}
\mc{R}_n \lp \mc{A}_s, D_s \rp  \leq \frac{\max_{i\in[K]}2 \big(c_i^{(s)} - \tilde{c}_i^{(s)}\big)}{\sqrt{\lambda_{\min}n}} + \frac{\sqrt{E}}{\lambda_{\min} n} + \varepsilon, 
\label{eq:sep_regret3}
\end{align}
\end{small}

where $E \coloneqq A\tilde{e}_n/B$ for the terms $A$, $B$, and $\tilde{e}_n$ introduced in Lemma~\ref{lemma:tracking1}. The exact expressions for $E$ and the higher-order term $\varepsilon$ are given by Eq.~\ref{eq:sep_def_E} and Eqs.~\ref{eq:sep_regret0} and~\ref{eq:sep_regret00} in Appendix~\ref{proof_separation_theorem}, respectively. 
\end{theorem}
As shown in~\eqref{eq:sep_def_E}, the term $E$ in~\eqref{eq:sep_regret3} is of $\tilde{\mc{O}}(n^{1/2})$ or $\tilde{\mc{O}}(n^{3/4})$ depending on whether $n\geq$ or $<N_0$, where $N_0$ is defined by~\eqref{eq:sep_def_E}.

\begin{remark}
The first term on the RHS of~\eqref{eq:sep_regret3} is the approximation error and the second one is the regret w.r.t.~the approx-oracle allocation (see Proposition~\ref{prop:approximate_regret} in Appendix~\ref{appendix:approximate} for precise decomposition). For $l \geq 3$, the approximation error term dominates and we have $\mc{R}_n(\Algsep, \Dsep) = \mc{\tilde{O}}(n^{-1/2})$. For $l=2$, however, the first term is zero, since $\tilde{c}_i^{(s)}=c_i^{(s)}$. In this case, we have $\mc{R}_n(\Algsep, \Dsep) = \mc{\tilde{O}}(n^{-3/4})$ or $\mc{\tilde{O}}(n^{-5/8})$ depending on whether $n\geq$ or $< N_0$. 
\end{remark}

%% file: LowerBound.tex
Lemma~\ref{lemma:tracking1} provided a general high probability bound on the deviation of the adaptive allocation $(T_i)_{i=1}^K$ from the approx-oracle allocation $(\Tins)_{i=1}^K$. In Sec.~\ref{sec:algorithms}, we observed that when specialized to the objective functions corresponding to $\Dltwo$, $\Dlone$ and $\Dsep$, we have $|T_i - \Tins| = \mc{O}\lp \sqrt{n}\rp$. A natural question to ask is whether there exists any other adaptive scheme which can achieve smaller deviation from the approx-oracle allocation than this. Our next result shows that this is not the case by deriving a lower bound on the expected deviation of any allocation scheme $\mc{A}$. 

To derive the lower bound, we consider a specific class of problems with two arms ($K=2$) and Bernoulli distributions~($l=2$), and objective functions of the form $\f(c_i, T_i) = c_i/T_i^{\alpha}$ for some $\alpha>0$. For some $p_0 \in (1/2, 1)$ and $\epsilon>0$, we define two Bernoulli distributions $P_1 \sim \Ber{p_0}$ and $P_2 \sim \Ber{p_0-\epsilon}$.
We consider two problem instances $\mc{P}_1$ and $\mc{P}_2$ with $K=2$ and distributions $P_1$ and $P_2$, but with their orders swapped: that is, $\mc{P}_1 = (P_1, P_2)$ and $\mc{P}_2 = (P_2, P_1)$. Finally, we introduce the notation $\cfunc{p}$ to represent the distribution dependent constant in the objective function $\f$ corresponding to a $\Ber{p}$ distribution. We now state the lower bound result. 

\begin{theorem}
\label{theorem:lower}
For some $p_0 \in (1/2,3/4]$ and $0<\epsilon>p_0-1/2$, consider two tracking problems $\mc{P}_1 = (P_1, P_2)$ and $\mc{P}_2 = (P_2, P_1)$, with $P_1 \sim \Ber{p_0}$ and $P_2 \sim \Ber{p_0-\epsilon}$ and objective function $\f(c, T) = c/T^{\alpha}$ for $\alpha>0$ where the constant $c = \cfunc{p}$ for $\Ber{p}$ distributions. Finally, introduce the notation $\tau = (n/2)( |\cfunc{p_0}^{1/\alpha} - \cfunc{p_0-\epsilon}^{1/\alpha}|)/( | \cfunc{p_0}^{1/\alpha} + \cfunc{p_0-\epsilon}^{1/\alpha}|)$. 
If $(T_i)_{i=1}^2$ denotes the allocation of any allocation scheme $\mc{A}$, we have 
\begin{align}
& \max_{\mc{P}_1, \mc{P}_2} \; \max_{i=1,2} \; \mbb{E} \lb | T_i - \Tins | \rb \geq  \sup_{ 0<\epsilon <p_0-1/2} \Gamma_\epsilon(\cfuncs, p_0), \;\;  \nonumber \\
& \text{where} \;\;
\Gamma_\epsilon(\cfuncs,  p_0) =  \frac{\tau}{2}\lp 1 - \epsilon \sqrt{n/(1-p_0)} \rp. 
 \nonumber 
\end{align}
\end{theorem}

As an immediate corollary of Theorem~\ref{theorem:lower}, we can observe that the deviation of the optimistic tracking scheme from the approx-oracle for $\Dltwo, \Dlone$ and $\Dsep$ cannot be improved upon by any adaptive scheme. 
\begin{corollary}
\label{corollary:lower} 
For $p_0=3/4$, $\epsilon = 1/(2\sqrt{n})$ and the $\cfuncs$ arising in the study of $\Dltwo, \Dlone$ and $\Dsep$, we have $\Gamma_\epsilon(\cfuncs, p_0) = \Omega(\sqrt{n})$. 
\end{corollary}


The details of the proof of Theorem~\ref{theorem:lower} are provided in Appendix~\ref{appendix:lower}. The proof proceeds by studying the event $\mc{E} = \{T_1 < n/2\}$ under the two probability measures corresponding to $\mc{P}_1$ and $\mc{P}_2$, which we denote by $\mbb{P}_1$ and $\mbb{P}_2$ respectively. By a usual change of measure argument~\cite[Lemma~1]{kaufmann2017learning} and an application of Pinsker's inequality, we obtain an upper bound on the quantity $|\mbb{P}_1(\mc{E}) - \mbb{P}_2 \lp \mc{E} \rp|$. This allows us to conclude that for any allocation scheme $\mc{A}$, the probability that $\max_{i=1,2} |T_i - \Tins|$ is greater than $|\Tins - n/2|$, is strictly bounded away from $0$, which implies the result of Theorem~\ref{theorem:lower}. Finally, the result stated in Corollary~\ref{corollary:lower} follows by setting $\epsilon = \mc{O}\lp 1/\sqrt{n}\rp$ and some loss specific calculations.

%% file: Experiments.tex
\tbf{Setup.} We study the performance of the proposed adaptive schemes on a problem with $K=2$, and $l=10$.  We set $P_1$ as the uniform distribution in $\Delta_l$ and $P_2 = P_\epsilon$ for $\epsilon \in \{0.1, 0.2, \ldots, 0.9\}$, where $P_\epsilon = (p_j)_{j=1}^l$ with $p_1 = \epsilon$ and $p_j= (1-\epsilon)/(l-1)$ for $1<j\leq l$.
To compare the performance of the adaptive schemes, we used three benchmarks: \\
\tbf{(i)} Uniform allocation, in which each arm is allocated $n/K$ samples. Note that the uniform allocation is the oracle scheme for $D_{\chi^2}$ (see Appendix~\ref{appendix:chi_squared}). \\
\tbf{(ii)} Greedy allocation, in which the arms are pulled by plugging in the current empirical estimate of $(P_i)_{i=1}^K$ in the objective function. \\
\tbf{(iii)} Forced Exploration, in which the arms are pulled according to the greedy scheme, while also ensuring that at any time $t$, each arm is pulled at least $\sqrt{t}$ times. This scheme is motivated by the strategy of~\cite{antos2008active}. For every value of $\epsilon$, we ran $500$ trials of all the allocation schemes with the budget $n=5000$. We focus our experiments on the $\ell_2$, $\ell_1$ and separation distances, since we observed no significant difference in the perormance of the different schemes for KL-divergece, as predicted in Remark~\ref{remark:kl}. To compare the performance of the allocation schemes, we used the term $\f(c_i, T_i) - \f(c_i, \Tins)$. \\
\tbf{Observations.} 
We plot the $\f(c_i, T_i) - \f(c_i, \Tins)$ values for the different allocation schemes and loss functions in Figs.~\ref{fig:fig1},~\ref{fig:fig2} and~\ref{fig:fig3}. 
As we can see from Fig.~\ref{fig:fig1}, the adaptive scheme outperforms the uniform allocation for the three distance metrics for both $\epsilon=0.5$ and $\epsilon=0.9$. Note that as $\epsilon$ increases, the optimal allocation get more skewed, and hence the gap in performance between uniform and adaptive also increases. 
The greedy and forced exploration schemes, both perform comparably to our proposed adaptive scheme for $\epsilon=0.5$, although their resulting allocations have higher variability especially for $\ell_2$ and $\ell_1$ distances. For the case of $\epsilon=0.9$ however, the adaptive scheme performs significantly better than both greedy and forced exploration methods for $\ell_2$ and $\ell_1$ distances, and result in a lower variance solution for separation distance.

\begin{figure}[!h]
  \begin{subfigure}[t]{.5\textwidth}
    \centering
    \includegraphics[width=\linewidth, height=4.1cm]{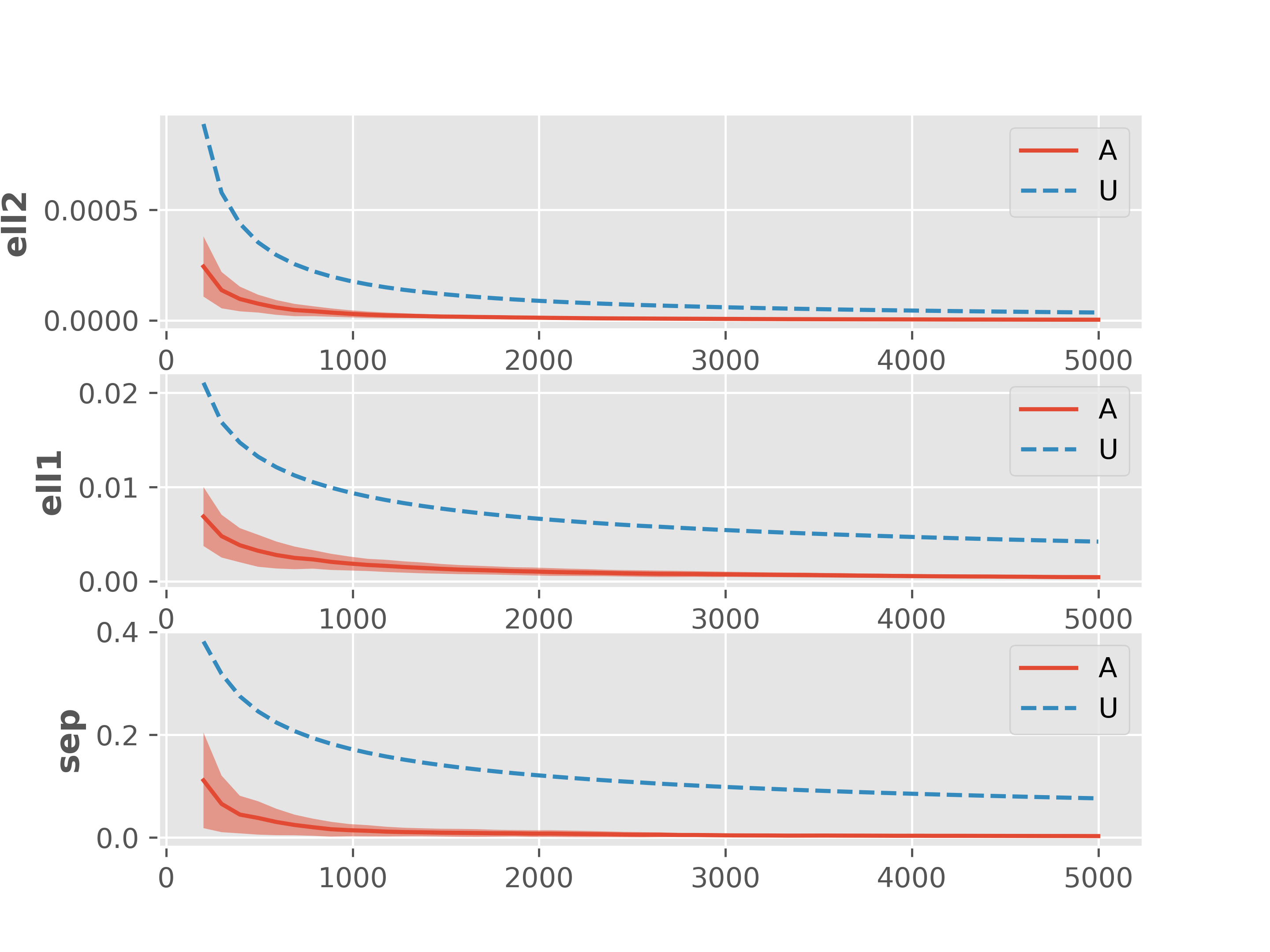}
  \end{subfigure}
  \hfill
  \begin{subfigure}[t]{.5\textwidth}
    \centering
    \includegraphics[width=\linewidth, height=4.1cm]{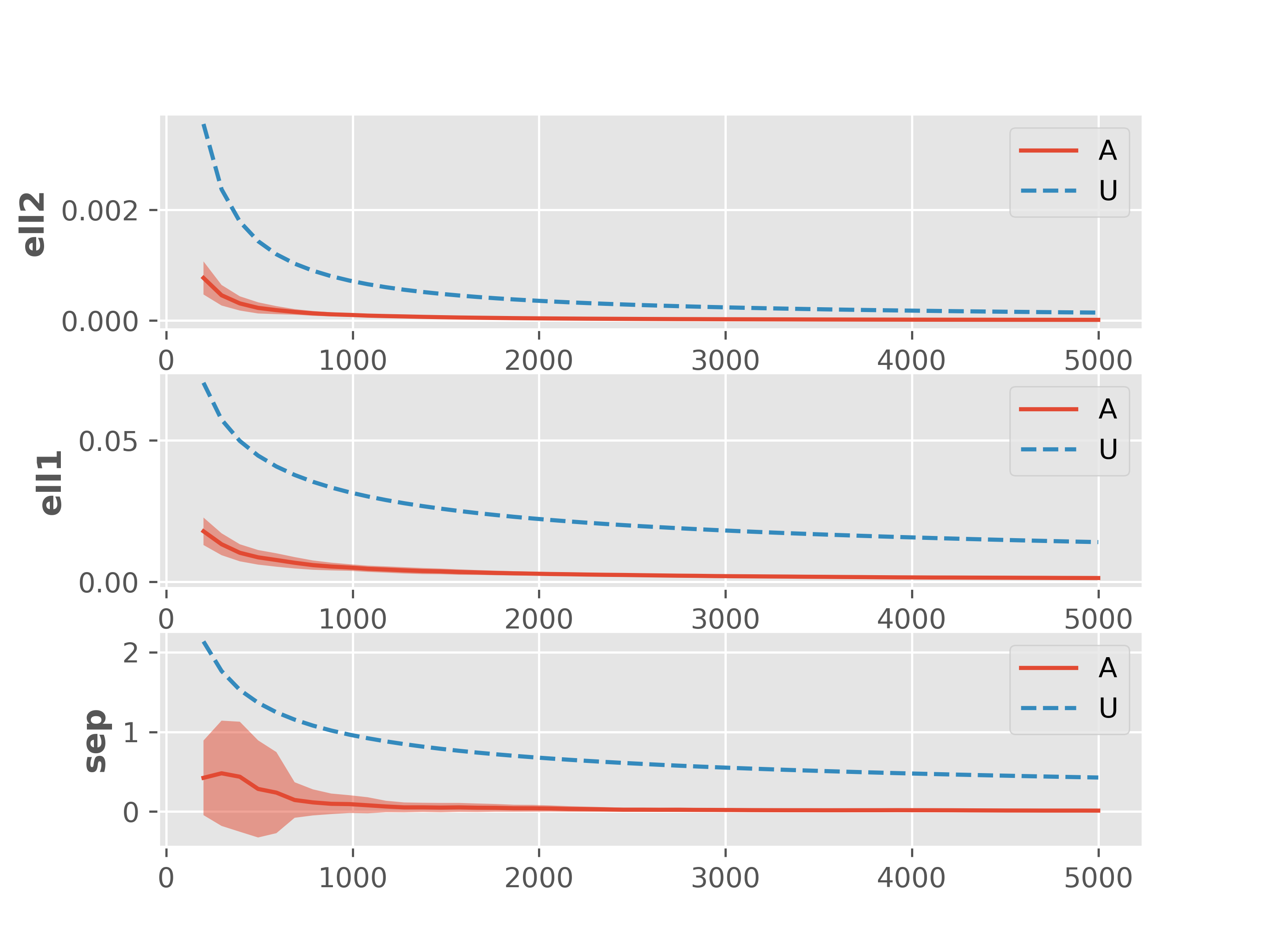}
  \end{subfigure}
  \caption{Plot of $\f(c_i, T_i) - \f(c_i, \Tins)$~vs.~$n$  for Adaptive (A) and Uniform (U) 
  sampling schemes, for $\epsilon=0.5$ (left) and $\epsilon=0.9$~(right).}
\label{fig:fig1}
\end{figure}

\begin{figure}[!h]
  \begin{subfigure}[t]{.5\textwidth}
    \centering
    \includegraphics[width=\linewidth, height=4.1cm]{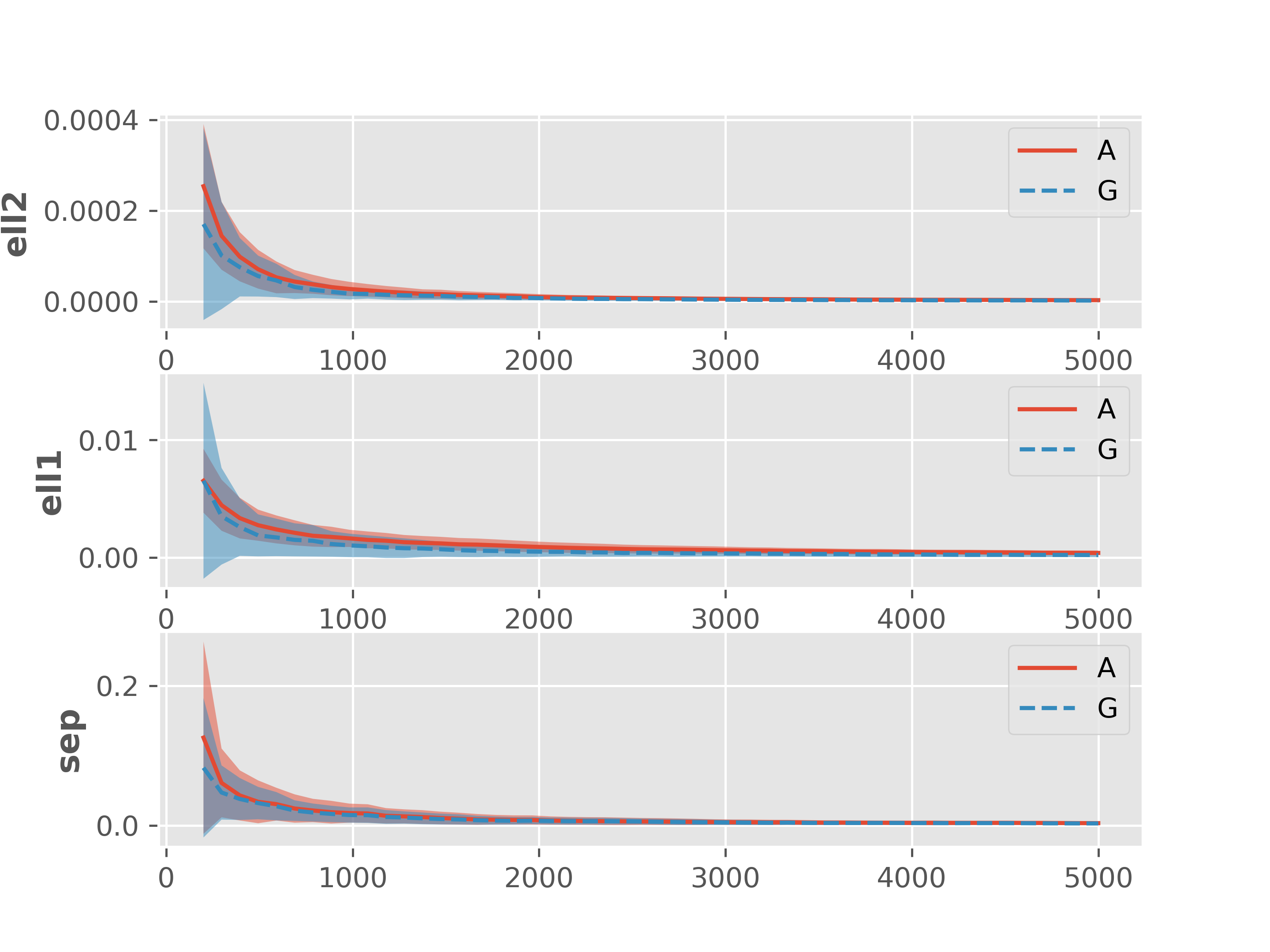}
  \end{subfigure}
  \hfill
  \begin{subfigure}[t]{.5\textwidth}
    \centering
    \includegraphics[width=\linewidth, height=4.1cm]{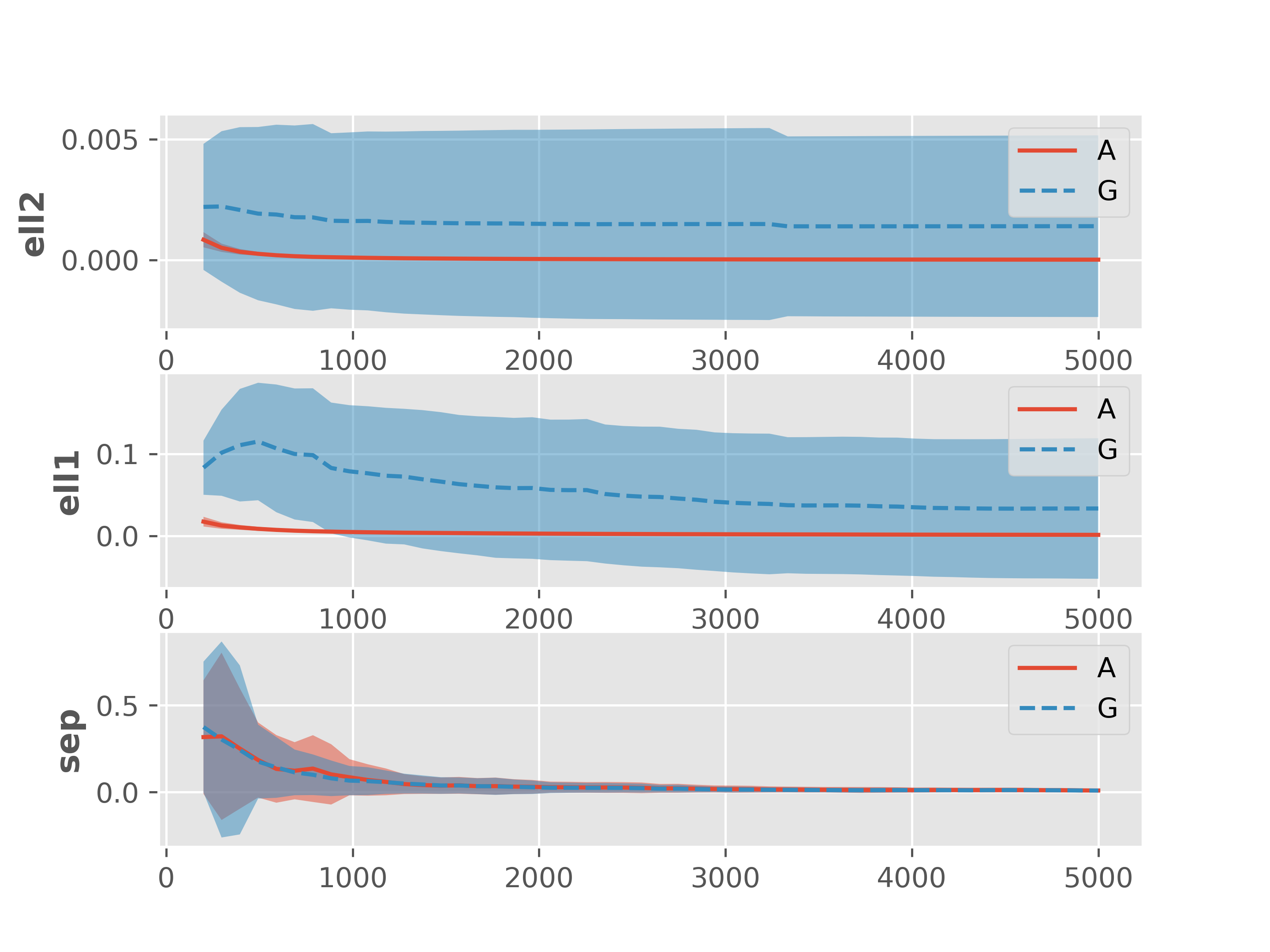}
  \end{subfigure}
  \caption{
      Plot of $\f(c_i, T_i) - \f(c_i, \Tins)$~vs.~$n$  for Adaptive (A) and Greedy (G) 
  sampling schemes, for $\epsilon=0.5$ (left) and $\epsilon=0.9$~(right). 
      }
\label{fig:fig2}
\end{figure}

\begin{figure}[!h]
\label{fig3}
  \begin{subfigure}[t]{.5\textwidth}
    \centering
    \includegraphics[width=\linewidth, height=4.1cm]{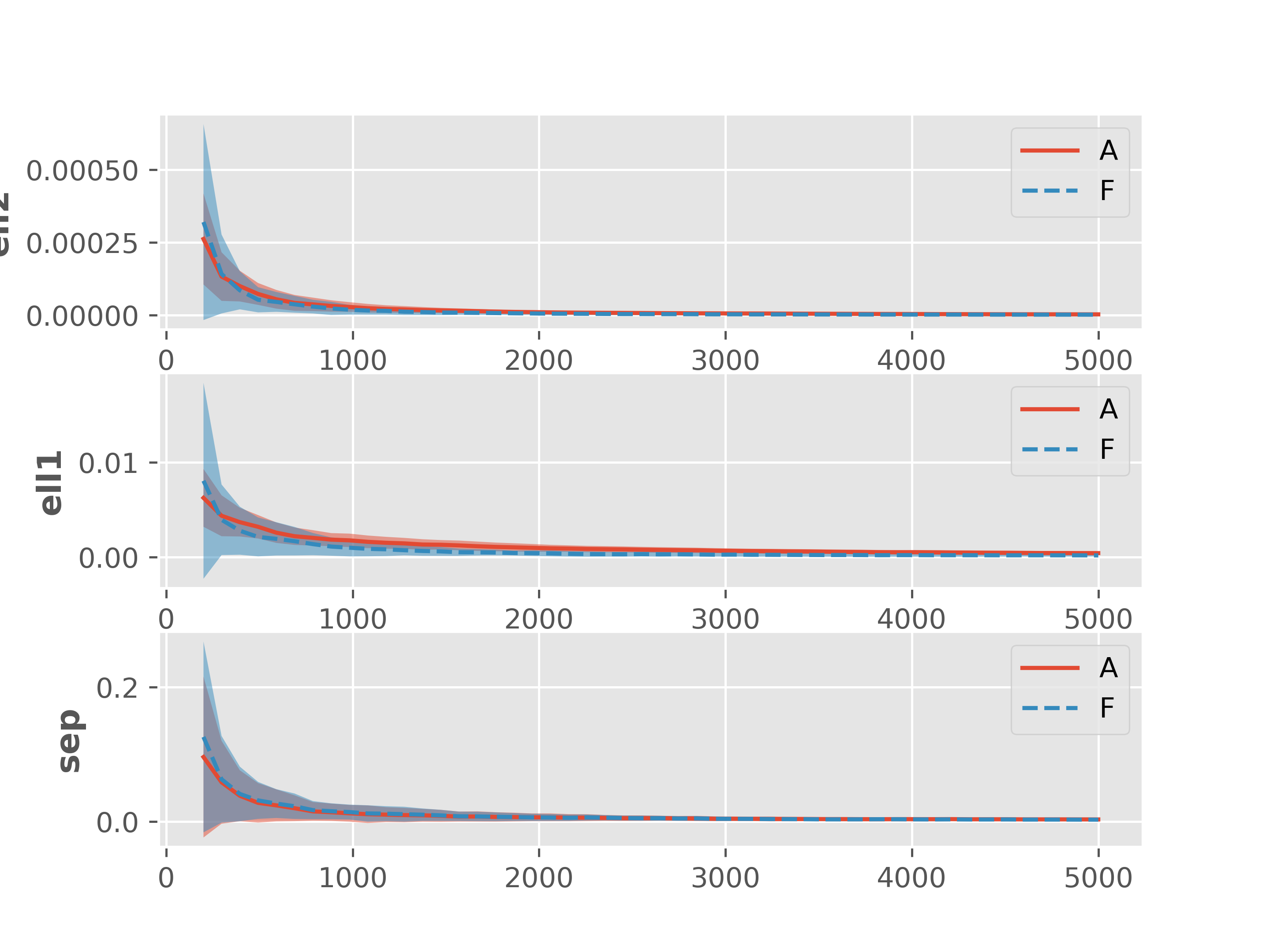}
  \end{subfigure}
  \hfill
  \begin{subfigure}[t]{.5\textwidth}
    \centering
    \includegraphics[width=\linewidth, height=4.1cm]{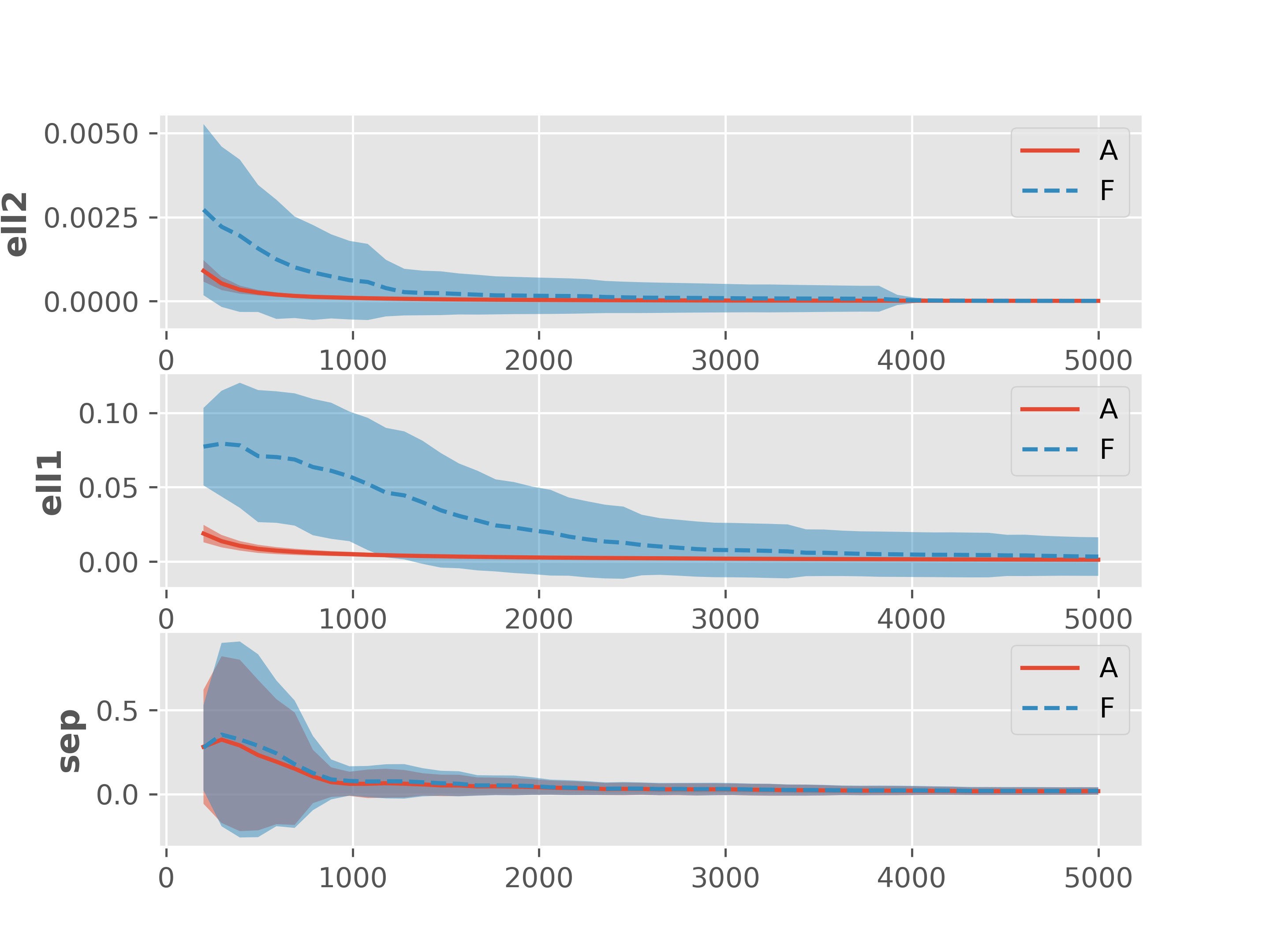}
  \end{subfigure}
  \caption{
      Plot of $\f(c_i, T_i) - \f(c_i, \Tins)$~vs.~$n$  for Adaptive (A) and Forced Exploration (F) 
  sampling schemes, for $\epsilon=0.5$ (left) and $\epsilon=0.9$~(right). 
      }
\label{fig:fig3}

\end{figure}

%% file: Discussion.tex
In this paper, our focus has been on minimizing the maximum distance between the estimate and true distributions (optimization problems~\eqref{eq:objective1} and~\eqref{eq:tracking1}). An important alternative formulation that has been studied in the bandit literature involves minimizing the average discrepancy~\citep{Carpentier11FT,Riquelme17AL}. Our results, in particular our general tracking scheme, can be extended to this case and we are able to provide adaptive allocation strategies to minimize the average distance between distributions, for all distances studied in this paper. Consider the following tracking/optimization problem, which is the equivalent of~\eqref{eq:tracking1} for the average case:
\begin{equation}
\label{eq:average_cost}
    \min_{T_1,\ldots,T_K} \; \frac{1}{K} \sum_{i=1}^K \varphi_i (c_i, T_i) \quad \text{s.t. } \sum_{i=1}^K T_i = n. 
\end{equation}
If $\varphi_i$'s are convex in $T_i$, then the optimal solution must satisfy $\frac{1}{K}\frac{\partial\varphi_i(c_i, T_i)}{\partial T_i}-\lambda=0,\;\forall i\in[K]$ and for some $\lambda\in\mbb{R}$. Thus, if the $T_i$-derivatives of $\varphi_i$ are \emph{regular}~(Definition~\ref{def:regular}), then~\eqref{eq:average_cost} can be solved using the tools developed in Section~\ref{sec:general_tracking}. It is easy to show that the distances studied in this paper (i.e.,~$\ell_2$, $\ell_1$, KL, and separation) satisfy this condition.

%% file: Conclusion.tex
We considered the problem of allocating samples to learn $K$ discrete distributions uniformly well in terms of four distance measures: $\ell_2^2$, $\ell_1$, $f$-divergence, and separation. We proposed a general optimistic tracking strategy for problems with \emph{regular} objective functions and then showed that the class of regular functions is rich enough to either contain or well approximate the true objective functions of all the considered distances. We then derived regret bounds for the proposed algorithm for all four distances. We also empirically verified our theoretical findings through numerical experiments. Finally, we ended with a discussion on extending our results to the related setting of minimizing the average error in learning a set of distributions. 
 
There are several directions in which our work can be extended: \tbf{1)} improving the performance of the adaptive algorithms and the hidden constants in the regret bounds by employing stronger concentration results, 
 and \tbf{2)} extending the results of the paper to the general problem of learning the dynamics (model) of a finite MDP, as discussed in Section~\ref{sec:introduction}.

%% file: appendix_Table.tex
\begingroup
\renewcommand\thetable{A.1}
\begin{longtable}{| r | p{8cm} |c|} 

\hline 
{\bf Symbols} & {\bf Description}  \\ \hline 
\multicolumn{2}{c}{\tbf{Preliminaries} }\\ \hline 
$K$ & number of probability distributions. \\ 
$(P_i)_{i=1}^K$ & the $K$ probability distributions. \\
$\X$ & $\X = \{x_1, x_2, \ldots, x_l\}$, the support of $(P_i)_{i=1}^K$ of size $l$. \\
$\Delta_l$, $\Delta_l^{(\eta)}$ & the $(l-1)$ dimensional probability simplex, and its $\eta$-interior. \\
$D$ & a general distance measure $D:\Delta_l \times \Delta_l \mapsto [0, \infty)$. \\
$D_{\ell_2}$, $D_{\ell_1}$, $D_f$, $D_{KL}$, $D_s$ & $\ell_2^2$ distance, $\ell_1$ distance, $f$-divergence, KL-divergence and separation distance. \\
$\gamma_i(T_i)$ & The exact objective function in ~\eqref{eq:objective1}, i.e., $\mbb{E}\lb D\lp \hat{P}_i, P_i\rp\rb$ with $\hat{P}_i$ constructed from $T_i$ i.i.d. samples. \\

$\mc{A}^*$, $(T_i^*)_{i=1}^K$ & The oracle allocation scheme and the oracle allocation. Solution to~\eqref{eq:objective1} assuming full knowledge of $(P_i)_{i=1}^K$. \\
& \\
$\Tit$ & Number of times distribution $i$ has been sampled by an algorithm $\mc{A}$ before time $t$. \\
$T_i$ & Total number of times distribution $i$ is sampled by an algorithm till the budget is exhausted. Note that $T_i = T_{i, n+1}$. \\
$Z_{ij}^{(s)}$ & Indicator that the $s^{th}$ sample from arm $i$ is $x_{j}$. (Bernoulli with probability $p_{ij}$) \\
$\tilde{Z}_{ij}^{(s)}$ & $Z_{ij}^{(s)} - p_{ij}$. centered version of $Z_{ij}^{(s)}$. \\
$W_{ij, t}$ & $\sum_{s=1}^t \tilde{Z}_{ij}^{(s)}$ \\
$\hat{p}_{ij, t}$ & empirical estimate of $p_{ij}$ at time $t$, $\hat{p}_{ij, t} = W_{ij, t}/T_i$ (assuming $T_i>0$). \\
$\hat{p}_{ij}$ & The empirical estimate of $p_{ij}$ constructed by an algorithm after the sampling budget is exhausted. Note that $\hat{p}_{ij} = \hat{p}_{ij, n+1}$. \\

$\mc{L}_n \lp \mc{A}, D \rp$ & Risk incurred by algorithm $\mc{A}$, for distance $D$ and sampling budget $n$. (Defined in~\eqref{eq:risk}) \\
$\mc{R}_n \lp \mc{A}, D \rp$ & Regret of algorithm $\mc{A}$. (Defined in~\eqref{eq:regret}) \\ 
& \\
$\f(c_i, T_i)$ & a \emph{regular} function. (see Definition~\ref{def:regular}). \\
$\mc{F}$ & the class of all regular functions. \\
 $g_i*$ and $h_i^*$ & $g_i^* \coloneqq \frac{\partial \f(c, \Tins)}{\partial c} \big \rvert_{c=c_i}$ 
and   $h_i^* \coloneqq \frac{\partial \f(c_i, T)}{\partial T} \big \rvert_{T=\Tins}$. \\
$\Algapprox$, $(\Tins)_{i=1}^K$ & Solution to~\eqref{eq:tracking1} assuming full knowledge of parameters $(c_i)_{i=1}^K$. When the objective function $\f$ of~\eqref{eq:tracking1} is an approximation of the objective~\eqref{eq:objective1}, then we shall refer to $\Algapprox$ and $(\Tins)_{i=1}^K$ as the \emph{approx-oracle} allocation rule and the \emph{approx-oracle} allocation respectively. \\

$A$, $B$, and $C$ & Terms defined in Lemma~\ref{lemma:tracking1} for a general adaptive algorithm described in Alg.~\ref{algo:optimistic_tracking1}. \\ 
& \\ \hline

 \multicolumn{2}{c}{ \tbf{Adaptive scheme for $\ell_2$}  }  \\ \hline 
 $\mc{E}_1$, $\delta_t$ & event defined in Lemma~\ref{lemma:ell2_1}, and $\delta_t = (6\delta)/(K l \pi^2 t^2)$ \\
 $\eijtltwo$, and $\eitltwo$ & $\sqrt{3 p_{ij} \log(2/\delta_t)/\Tit}$, and $\sqrt{27 \log(1/\delta_t)/\Tit}$ resp. \\
$\ciltwo$, $\hatciltwo$ and $\uiltwo$ & $1 - \sum_{j=1}^l p_{ij}^2$ and $1 -\sum_{j=1}^l \hat{p}_{ij,t}^2$  and $\hatciltwo + \eitltwo$\\ 
$\lambdailtwo$, $\lambdailtwo[\min]$, $\lambdailtwo[\max]$ & $\ciltwo/(\sum_k \ciltwo[k])$, $\min$ and $\max$ of $\lambdailtwo$.    \\
$\Mltwo$&  $ \frac{\lambda_{\max}^{(\ell_2)} \sqrt{2\log(1/\delta_n)}}{\lambda_{\min}^{(\ell_2)}\sum_{i=1}^K c_i^{(\ell_2)}}$ \\

& \\ \hline 
\multicolumn{2}{c}{ \tbf{Adaptive scheme for $\ell_1$}}  \\ \hline 

$\mc{E}_2$, $\delta_t$ & event defined in Lemma~\ref{lemma:ell1_1}, and $\delta_t = (3\delta)/(K l \pi^2 t^2)$ \\
 $\eijtlone$, and $\eitlone$ & $\sqrt{2  \log(2/\delta_t)/\Tit}$, and $\sqrt{2l^2 \log(2/\delta_t)/\Tit}$ resp. \\
$\cilone$, $\hatcilone$ and $\uilone$ & $\sum_{j=1}^l \sqrt{p_{ij}(1-p_{ij})}$ and $\sum_j \sqrt{ \hat{p}_{ij,t}(1-\hat{p}_{ij,t})}$  and $\hatcilone + \eitlone$\\ 
$\lambdailone$, $\lambdailone[\min]$, $\lambdailone[\max]$ & $(\cilone)^2/(\sum_k (\cilone[k])^2)$, $\min$ and $\max$ of $\lambdailtwo$.    \\
$\Mlone$& $\frac{2l\lambda_{\max}^{(\ell_1)}\sqrt{\log(1/\delta_n)}}{\Clone\lambda_{\min}^{(\ell_1)}}  $ \\

& \\ \hline 

\multicolumn{2}{c}{ \tbf{Adaptive scheme for $\Dfdiv$ and $\Dkl$}}  \\ \hline 
$\Dfdiv^{(r)}(\hat{P}_i, P_i)$, $R_{i, r+1}$ & The $r$-term Taylor's approximation of $ \Dfdiv (\hat{P}_i, P_i )$, and the remainder term. Furthermore, we have $R_{i, r+1} = \sum_{j=1}^l R_{ij, r+1}$ where $R_{ij, r+1}$ is defined in~\eqref{eq:remainder}. \\

$\f_i$ & the approximate objective function obtained by taking the expectation of $\Dfdiv^{(r)}$. We do not use $\f(c_i, \cdot)$ to represent it, the derivation of $c_i$ requires the knowledge of $f^{(m)}(1)$ values for $1 \leq m \leq r$. \\

$C_1$ & constant in assumption \tbf{(f2)} in Sec.~\ref{subsec:f_divergence} \\
$C_{f, r+1}$ & Constant defined in Eqs.~\ref{eq:remainder1}~\ref{eq:remainder2} and~\ref{eq:remainder4} in Appendix~\ref{subsubsec:proof-lemma5}. \\

$\mc{E}_\delta$ & The $(1-\delta)$ probability event used to characterize a general adaptive scheme, in which we have $\tau_{0,i} \leq T_i \leq \tau_{1,i}$ for non-negative constants $\tau_{0,i}$ and $\tau_{1,i}$\\ 

$\beta_m^{(ij)}\lp \tau_{0,i}, \tau_{1,i}\rp$ & Defined in Lemma~\ref{lemma:f_div_moments} \\
$\Psi_i$, $\psi_{1,i}, \psi_{2,i}, \psi_{3,i}$ and $\psi_4$ & Defined in Theorem~\ref{theorem:regret_f_div}. \\
$\cikl$ & $(1/12)(\sum_{j=1}^l1/p_{ij} - 1)$ \\
$e_{ij, t}$ and $e_{i,t}$ & $\sqrt{2\log(2/\delta_t)/\Tit}$ and $\sqrt{ (32l^2 \log(2/\delta_t))/\Tit}$. \\
$\Mkl$ & $ \sqrt{96 K \log(1/\delta_n)}\eta^3$.   \\ 
& \\\hline

\multicolumn{2}{c}{ \tbf{Adaptive scheme for $\Dsep$}}  \\ \hline 
$p_{i,S}$ for $S\subset [l]$, $S \neq \emptyset$ & $\sum_{j \in S} p_{ij}$ \\ 
$c_i^{(s)}$ & $\sum_{j=1}^l \rho_1(p_{ij})$ where $\rho_1(p) = \sqrt{(1-p)/p}$ \\
$\tilde{c}_i^{(s)}$ & $\max_{S \subset [l]} \rho_2(p_{i,S})$ where $\rho_2(p) = \rho_1(p) + \rho_1(1-p)$ \\
$a_{i,t}$ and $b_{i,t}$ &  $\big(32\log(2/\delta_t)/\Tit\big)^{1/4}$ and $ \big(\frac{l\; a_{i,t}}{\eta}\big)\max\big\{1,\frac{a_{i,t}}{2\eta^{3/2}}\big\}$ \\
$E$ & $A\tilde{e}_n/B$ \\
$N_0$ & $\min \{ n \geq 1\; : \; (n/\log(2/\delta_n)) \geq 2/(\lambda_{\min}^{(s)} \eta^6)\}$ \\ 
& \\
& \\\hline

\caption{Table of symbols used in the paper.}
    \label{tab:table_of_notation}
\end{longtable}

\endgroup

%% file: appendix_approximation.tex
In this section, we present the formal proof of the statement about the regret decomposition when using an approximate objective function $\f_i(T_i)$ in place of $\mbb{E}[D(\hat{P}_i,P_i)]$, where $\hat{P}_i$ is the empirical estimate of $P_i$ constructed using $T_i$ i.i.d.~samples. 

\tbf{Note:} In this section, we use the term $\f_i$ to represent any approximation ,and not just the \emph{regular} approximations $\f(c_i, \cdot)$ introduced in Sec.~\ref{sec:general_tracking},  of the true objective function $\gamma_i$  in~\eqref{eq:objective1}. This is because the stated result does not require the regularity assumptions to be satisfied by the approximation. 

\begin{proposition}
\label{prop:approximate_regret}
Suppose $\mc{A}$ is any allocation scheme for a loss function $D$ and sampling budget $n$. Let $\f_i(T_i)$ be an approximation for $ \gamma_i(T_i) \coloneqq \mbb{E}[D(\hat{P}_i, P_i)]$, where $\hat{P}_i$ is the empirical estimate of $P_i$ constructed using $T_i$  i.i.d.~samples. Define the remainder term $R_i(T_i) \coloneqq \gamma_i(T_i) - \f_i(T_i)$. Let $\tilde{\mc{A}}^*$ and $\mc{A}^*$ represent the oracle and the approx-oracle allocation schemes, and let $(T_i^*)_{i=1}^K$ and $(\Tins)_{i=1}^K$ denote their allocations respectively. Then, assuming that $(\gamma_i(\cdot))_{i=1}^K$ are decreasing functions,  we have the following:
\begin{align}
    \mc{R}_n \lp \mc{A}, D \rp  &\coloneqq \mc{L}_n \lp \mc{A}, D \rp - \mc{L}_n(\mc{\tilde{A}}^*, D) \leq \mc{L}_n \lp \mc{A}, D \rp - \mc{L}_n \lp \mc{A}^*, D \rp + 3 \max_{i \in [K]} | R_i (\Tins)| \quad \text{and} \label{eq:approx_reg1} \\
   \mc{R}_n \lp \mc{A}, D \rp & \leq  \mc{L}_n \lp \mc{A}, D \rp - \f_i(\Tins) +  2 \max_{i \in [K]} | R_i (\Tins)|  \label{eq:approx_reg2}
\end{align}
\end{proposition}

\begin{proof} 
Note that we have $\gamma_i(T_i) = \f_i(T_i) + R_i(T_i)$. By definition of regret, we have 
\begin{align*}
    \mc{R}_n \lp \mc{A}, D \rp & = \max_{1 \leq i \leq K} \lp \mbb{E}[D(\hat{P}_i, P_i)] - \gamma_i(T_i^*) \rp  \\
    &\stackrel{(a)}{\leq}  \max_{1 \leq i \leq K} \lp \mbb{E}[D(\hat{P}_i, P_i)] - \f_i(\Tins) \rp + \max_{1 \leq i \leq K} \left \lvert \f_i(\Tins) -  \gamma_i(T_i^*)\right \rvert  \\
    & \stackrel{(b)}{=} \mc{L}_n \lp \mc{A}, D \rp - \f_i (\Tins) + \max_{1 \leq i \leq K} \left \lvert \f_i(\Tins) -  \gamma_i(T_i^*)\right \rvert. \\ 
    & \stackrel{(c)}{\leq}  \mc{L}_n \lp \mc{A}, D \rp - \mc{L}_n(\Algapprox,D) + \max_{1\leq i \leq K}|R_i(\Tins)| + \max_{1 \leq i \leq K} \left \lvert \f_i(\Tins) - \gamma_i(T_i^*)\right \rvert  
\end{align*}
In the above display, \\
\tbf{(a)} follows from an application of triangle inequality, \\
\tbf{(b)} uses the fact that by definition of the approx-oracle rule, we must have $\f_i(\Tins)=\f_j(\tilde{T}_j^*)$ for all $i,j \in [K]$. \\
\tbf{(c)} follows from another application of triangle inequality.

To complete the proof, it suffices to show that the term $  \max_{1 \leq i \leq K} \left \lvert \f_i(\Tins) -  \gamma_i(T_i^*)\right \rvert  $ is less than $2 \max_{1\leq i \leq K}|R_i(\Tins)|$. 
Note that since $\Algapprox$ is the approx-oracle sampling scheme, there must exist a $\lambda \in [0, \infty)$ such that $\f_i(\Tins) = \lambda$ for all $i \in [K]$. Introduce the notation $\Delta T_i^* =  \tilde{T}_i^* - T_i^* $. Then we consider two cases. 

\tbf{Case~1: $\bm{\Delta T_i^* = 0}$ for all $\bm{i}$.}  In this case, we have $\Tins = T_i^*$ for all $i$, which implies that $|\f_i(\Tins) - \gamma_i(T_i^*)| \leq |R_i(\Tins)|$ for all $i$, thus proving that $\max_{i \in [K]} |\f_i(\Tins) - \gamma_i(T_i^*)| \leq \max_{i \in [K]} |R_i(\Tins)|$.

\tbf{Case~2: $\bm{\Delta T_k^* \neq 0}$ for some $k \in [K]$.} Since we have $\sum_{i=1}^K T_i^* = \sum_{i=1}^K \tilde{T}_i^*$, it means that $\sum_{i=1}^K \Delta T_i^* = 0$. Thus there must exist $i, j \in [K]$ such that $\Delta T_i^*>0$ and $\Delta T_j^*<0$. Define $i_0 = \argmin_{i \in [K]} \gamma_i(\Tins)$ and $i_1 = \argmax_{i\in [K]} \gamma_i(\Tins)$. By monotonicity assumption on the functions $\gamma_i$, we must have the following $\gamma_{i_0}(\tilde{T}_{i_0}^*) < \lambda < \gamma_{i_1}(\tilde{T}_{i_1}^*)$. We claim that this implies that 
\begin{align}
\label{eq:claim0}
\gamma_{i_0}\lp \tilde{T}_{i_0}^* \rp \leq \gamma_i(T_i^*) \leq \gamma_{i_1}(\tilde{T}_{i_1}^*). 
\end{align}
If~\eqref{eq:claim0} is true, then the result of the proposition immediately follows from the observation that  $\lambda - \max_{k \in [K]} |R_k(\tilde{T}_k^*)| \leq \; \gamma_i(\Tins) \; \leq \lambda + \max_{k \in [K] } |R_k(\tilde{T}_k^*)|$, which coupled with~\eqref{eq:claim0} implies that $\lambda - \gamma_i(T_i^*) \leq 2 \max_{k \in K} |R_k\lp \tilde{T}_k^* \rp |$.

Finally, it remains to show that~\eqref{eq:claim0} is true. We prove this by contradiction. Introduce the notation $\mu = \gamma_i(T_i^*)$ for all $i \in [K]$ (by the definition of the oracle allocation rule). Next, suppose that $\mu \not \in [\lambda - \max_{k}|R_k(\tilde{T}_k^*)|, \; \lambda + \max_k |R_k(\tilde{T}_k^*)|]$. Without loss of generality, assume that $\mu<\lambda - \max_{k}|R_k(\tilde{T}_k^*)|$ (the other case can be argued similarly). This means that $\mu = \gamma_i(T_i^*) < \gamma_i(\Tins)$ for all $i$. By the monotonicity of $\gamma_i$ this implies that all the $\Delta T_i^* = \tilde{T}_i^* - T_i^*$ have the same sign. However, since $\sum_{i}\Delta T_i^* = 0$, this can only happen if $\Delta T_i^* = 0$ for all $i$. This contradicts the defining assumption of \tbf{Case~2} above.

\end{proof}

%% file: appendix_general_tracking.tex
\label{proof_tracking_lemma}


Assume that arm $k$ was played at least once during the period $K \leq t \leq n-1$. Let $t_k$ be the time at which arm $k$ was played for the last time. Recall that for any $t\geq 1$, we use $\Tit$ to denote the number of times the  arm $i$ is played  before the start of round $t$, and $T_i$ denotes the total number of times arm $i$ is played  until the total sampling budget is exhausted, i.e., $T_i = T_{i, n+1}$ 
Then we have the following for any $1 \leq j \leq K$: 
\begin{align}
\label{eq:tracking2}
    \f\lp u_{j,t_k}, T_{j,t_k}\rp & \stackrel{(a)}{\leq} \f\lp u_{k, t_k}, T_{k, t_k} \rp \stackrel{(b)}{\leq} \f\lp c_k + 2e_{k,t_k}, T_{k,t_k} \rp 
\end{align}
In the above display, \tbf{(a)} follows from the arm selection rule for $t \geq K$ and \tbf{(b)} follows from the fact that $u_{k,t_k} - c_k \leq 2e_{k,t_k}$ and that $\f$ is non-decreasing in its first argument. 

Next, we define $i = \argmin_{1 \leq k \leq K} \f(c_k, T_{k})$. Using the fact that $\f$ is non-increasing in its second argument, we have $T_{i} \geq \Tins$. We now consider two cases:

\paragraph{Case~1: $T_{i}=\Tins$}. In this case, we must have $T_{j} = \Tins[j]$ for all $1\leq j \leq K$. This is because if $\min_{k} \f(c_k, T_{k}) = \f(c_k, \Tins[k])$, then $\f(c_k, T_{k}) = \f(c_k, \Tins[k])$ for all $1\leq k\leq K$, and thus $T_{k} = \Tins[k]$ for all $1 \leq k \leq K$. In this case the result of Lemma~\ref{lemma:tracking1} holds trivially. 

\paragraph{Case~2: $T_{i} > \Tins$.} In this case, the arm $i$ must have been played at least once during the time interval $K \leq t \leq n-1$. Denote that time by $t_i$. 
 Defining $g_i^* \coloneqq \left. \frac{\partial \f (c_i, T)}{\partial T}\right \rvert_{T=\Tins}$,  we have the following sequence of inequalities for any $1\leq j \leq K$:
\begin{align}
   \f(c_j, T_{j}) &\stackrel{(a)}{\leq} \f(c_i + 2e_{i,t_i}, T_{i}-1 ) 
   \stackrel{(b)}{\leq} \f\lp c_i + 2e_i^*, \Tins\rp 
   \stackrel{(c)}{\leq} \f\lp c_i, \Tins \rp  + 
   2g_i^*e_i^* \nonumber\\
      & \stackrel{(d)}{\leq} \f(c_i, \Tins) + 2A e_i^*  \stackrel{(e)}{\leq} \f(c_j, \Tins[j] ) + 2A \tilde{e}_n. \label{eq:tracking2}
\end{align}
In the above display, \\
\tbf{(a)} follows from the fact that since arm $i$ was played at time $t_i$ we must have $\f(u_{j,t_i}, T_{j,t_i}) \leq \f(u_{i,t_i}, T_{i,t_i})$. Furthermore, since $u_{j, t_i} \geq c_j$ under the event $\mc{E}$ by definition, and the fact that $\f$ is non-decreasing in its first argument implies that $\f(c_j, T_{j,t_i}) \leq \f(u_{j,t_i}, T_{j, t_i})$, \\
\tbf{(b)} follows from the fact that $T_{i}-1 \geq \Tins$ (defining assumption of Case~2), that $\f$ is non-increasing in its second argument, and that $e_i(\cdot)$ is a non-increasing in its argument and $\f$ is non-decreasing in its first argument,\\
\tbf{(c)} follows from the fact that $\f$ is a concave function in its first argument and thus is majorized by its linear approximation, \\
\tbf{(d)} follows from the fact that $A = \max_{1 \leq j \leq K} \; g_j^* \geq g_i^*$ by definition, and \\
\tbf{(e)} uses the fact that by definition of $\Tins[j]$ we have $\f(c_j, \Tins[j]) = \f(c_i, \Tins[i])$ and that $e_i^* \leq \max_{1 \leq k \leq K} e_k^* \coloneqq \tilde{e}_n$. 

Next, by defining $h_j^* = \left. \frac{\partial \f\lp c, \Tins[j] \rp}{\partial c}\right\rvert_{c=c_j}$ and from the relation between the first and last terms of \eqref{eq:tracking2} we observe the following:
\begin{align}
\label{eq:tracking3}
\f(c_j, \Tins[j] ) + 2A\tilde{e}_n \geq     \f(c_j, T_{j}) & \stackrel{(a)}{\geq} \f(c_{j}, \Tins[j] ) + h_j^*\lp T_{j} - \Tins[j]  \rp. 
\end{align}
In the above display, \tbf{(a)} follows from the assumption that $\f$ is a convex function in its second argument. 
\eqref{eq:tracking3} implies that 
\begin{align*}
   & h_j^* (T_{j} - \Tins[j] ) \leq 2 A \tilde{e}_n 
\quad    \Rightarrow T_{j} - \Tins[j]   \stackrel{(a)}{\geq} -\frac{2 A \tilde{e}_n }{|h_j^*|} \stackrel{(b)}{\geq} - \frac{2 A \tilde{e}_n }{B} \quad  
\Rightarrow T_{j}  \geq \Tins[j]  - \frac{2 A \tilde{e}_n }{B}. 
\end{align*}
In the above display, \tbf{(a)} follows from the assumption that $\f$ is  non-increasing function of its second argument, and \tbf{(b)} follows from the definition of $B = | \max_{1 \leq j \leq K}\; h_j^*|$ and the assumption that $B>0$. This completes the proof for the lower bound on $T_{j}$ for $1 \leq j \leq K$. 

Next, we obtain the upper bound on $T_{j}$ as follows:
\begin{align*}
    T_{j} = n - \sum_{i\neq j} T_{i} \leq n - \sum_{i \neq j} \Tins + 2(K-1)A \tilde{e}_n/B = \Tins[j]  + 2(K-1)A\tilde{e}_n/B. 
\end{align*}

%% file: appendix_regret_ell2.tex
\subsection{Proof of Lemma~\ref{lemma:ell2_1}}
\label{proof_ell2_lemma1}

\begin{proof}
Write $\mc{E}_1 = \cap_{t, i,j} \mc{E}_1^{(t,i,j)}$ where the events $\mc{E}_1^{(t,i,j)}$ are defined implicitly from the definition of $\mc{E}_1$.
Then by union bound we have that $Pr\lp (\mc{E}_1)^c \rp \leq  \sum_{t=1}^\infty \sum_{i=1}^K
\sum_{j=1}^l Pr\lp \lp \mc{E}_1^{(t,i,j)}\rp^c\rp$. To complete the proof, it suffices to show that $Pr\lp
\lp \mc{E}_1^{(t,i,j)}\rp^c \rp \leq \delta_t = \frac{6 \delta}{K l t^2 \pi^2}$ which implies the result.
We proceed by using the multiplicative form of Chernoff's inequality. 
\begin{align*}
    Pr\lp |\hat{p}_{ij, t} - p_{ij}| > e_{ij,t}^{(\ell_2)} \rp 
    & = Pr\lp \hat{p}_{ij, t} > p_{ij}\lp 1 + \frac{e_{ij,t}^{(\ell_2)}}{p_{ij}}\rp \rp +  Pr \lp \hat{p}_{ij, t} < p_{ij}\lp 1 - \frac{e_{ij, t}^{(\ell_2)}}{p_{ij}} \rp   \rp \\
    & =\sum_{s=0}^t \lp  Pr\lp \hat{p}_{ij, t} > p_{ij}\lp 1 + \frac{e_{ij, t}^{(\ell_2)}}{p_{ij}}\rp, \; \Tit = s\rp +  Pr \lp \hat{p}_{ij, t} < p_{ij}\lp 1 - \frac{e_{ij, t}^{(\ell_2)}}{p_{ij}} \rp, \; \Tit=s   \rp  \rp \\
    & \stackrel{(a)}{\leq} 
     = 2 \sum_{s=0}^t Pr\lp \Tit = s\rp \exp\lp - \frac{3  p_{ij} s \log(2/\delta_t)}{3 p_{ij}s } \rp 
     = \delta_t. 
\end{align*}
In the above display, the inequality $(a)$ follows from applying the multiplicative form of Chernoff's inequality as derived in \cite[Eq.~(6)~and~(7)]{hagerup1990guided}.
\end{proof}

\subsection{Proof of Lemma~\ref{lemma:ell2_3}}
\label{proof_ell2_lemma3}

With the above result, we can construct a confidence interval for the parameter $c_i^{(\ell_2)}$. Under the
event $\mc{E}_1$, we have
\begin{align*}
\sum_{j=1}^l \hat{p}_{ij, t}^2 & \geq \sum_{j=1}^l \lp  p_{ij}^2 + (e_{ij,t}^{(\ell_2)})^2 - 2 e_{ij,t}^{(\ell_2)} p_{ij} \rp \\
\Rightarrow \sum_{j=1}^l \hat{p}_{ij, t}^2  & = \sum_{j=1}^l p_{ij}^2 + \frac{3\log(1/\delta_t)}{\Tit}\sum_{j=1}
^l p_{ij} - 2 \sqrt{ \frac{3\log(1/\delta_t)}{\Tit}}\sum_{j=1}^l p_{ij}^{3/2} \\
& \stackrel{(a)}{\geq} \sum_{j=1}^l p_{ij}^2 + \frac{3 \log(1/\delta_t)}{\Tit} - 2\sqrt{ \frac{3\log(1/\delta_t)}{\Tit}}\sum_{j=1}^l p_{ij},
\end{align*}
where $(a)$ follows from the fact that $\sum_{j=1}^l p_{ij}^{3/2} \leq \sum_{j=1}^l p_{ij} = 1$.

Similarly, we can obtain the following:
\begin{align*}
\sum_{j=1}^l \hat{p}_{ij, t}^2 & \leq \sum_{j=1}^l \lp p_{ij}^2 + (e_{ij, t}^{(\ell_2)})^2 + 2e_{ij, t}^{(\ell_2)} p_{ij} \rp 
 \leq \sum_{j=1}^l p_{ij}^2 + \frac{3 \log(1/\delta_t)}{\Tit} +
2\sqrt{ \frac{3\log(1/\delta_t)}{\Tit}}.
\end{align*}

Combining the inequalities from the previous two displays, we get the required confidence interval
for the term $c_i^{(\ell_2)}$ around its empirical counterpart $\hatciltwo \coloneqq 1-
\sum_{j=1}^l \hat{p}_{ij, t}^2$, as follows:
\begin{align*}
|c_i^{(\ell_2)} - \hatciltwo| &
\leq \min \lp  \frac{3 \log(1/\delta_t)}{\Tit}  +  2\sqrt{ \frac{3\log(1/\delta_t)}{\Tit}}, \; 1 \rp 
 \leq 3 \sqrt{ \frac{3\log(1/\delta_t)}{\Tit}} \coloneqq e_{i,t}^{(\ell_2)}.
\end{align*}

\subsection{Proof of Theorem~\ref{theorem:ell2_regret}}
\label{proof_ell2_theorem}

\begin{proof}
We begin by obtaining the constants $A$, $\tilde{e}_n$ and $B$ from Lemma~\ref{lemma:tracking1}. 
We introduce the notation $\lambda_i^{(\ell_2)} \coloneqq \frac{c_i^{(\ell_2)}}{\sum_{k=1}^K c_k^{(\ell_2)}}$, $\Cltwo \coloneqq \sum_{i=1}^K c_i^{(\ell_2)}$, $\lambda_{\max}^{(\ell_2)} = \max_{1\leq i \leq K} \lambda_i^{(\ell_2)}$, and $\lambda_{\min}^{(\ell_2)} = \min_{1 \leq i \leq K} \lambda_i^{(\ell_2)}$. 
\begin{align*}
    A &= \max_{1 \leq i \leq K} \frac{1}{\Tins} = \frac{1}{\lambda_{\min}^{(\ell_2)}n} = \mc{O} \lp \frac{1}{n} \rp,   \\
    B &= \min_{1 \leq i \leq K} \frac{c_{i}^{(\ell_2)}}{\lp \Tins \rp^2} = \frac{\Cltwo}{\lambda^{(\ell_2)}_{\max} n^2} = \mc{O} \lp \frac{1}{n^2} \rp, \\
    \tilde{e}_n &= \max_{1 \leq i \leq K} e_i^* \leq \sqrt{ \frac{27\log(1/\delta_n)}{\lambda^{(\ell_2)}_{\min} n}} = \mc{O} \lp \sqrt{\frac{\log n}{n}} \rp. 
\end{align*}

The above calculations imply that we have 
\begin{equation}
    \frac{A \tilde{e}_n }{B} = \frac{ \lambda_{\max}^{(\ell_2)} \sqrt{27 \log(1/\delta_n)}}{\lambda_{\min}^{(\ell_2)} \Cltwo } \sqrt{n} \coloneqq \Mltwo \sqrt{n}. 
\end{equation}

Before proceeding, recall that we drop the $n+1$ from the  subscript when referring to the final probability mass estimates, i.e., we use $\hat{p}_{ij}$ instead of $\hat{p}_{ij, n+1}$ to refer to the estimate of $p_{ij}$ after the end of $n$ rounds. Next, consider any arm $i$, and decompose the expected $\ell_2^2$ distance as follows: 
\begin{align*}
    \mbb{E} \lb D_{\ell_2}\lp \hat{P}_i, P_i \rp \rb &= \sum_{j=1}^l \mbb{E} \lb \lp  \hat{p}_{ij} - p_{ij} \rp^2 \rb = \sum_{j=1}^l \mbb{E} \lb \lp \hat{p}_{ij} - p_{ij} \rp^2 \lp \indi{\mc{E}_1} + \indi{\mc{E}_1^c} \rp \rb. 
\end{align*}
We now consider each term of the summation above separately. 
\begin{align*}
    \mbb{E} \lb\lp  \hat{p}_{ij} - p_{ij} \rp ^2 \indi{\mc{E}_1^c } \rb  &\leq 1 \times Pr \lp \mc{E}_1^c \rp \stackrel{(a)}{\leq} \delta \\
   \Rightarrow \sum_{j=1}^l \mbb{E}\lb (\hat{p}_{ij} - p_{ij})^2 \indi{\mc{E}_1^c } \rb &\leq l \delta.  
\end{align*}
In the above display, the inequality \tbf{(a)} follows from the statement of Lemma~\ref{lemma:ell2_1}. 

Next, we proceed as follows:
\begin{align*}
    \mbb{E} \lb \lp \hat{p}_{ij} - p_{ij}\rp^2 \indi{\mc{E}_1} \rb &= \mbb{E} \lb \frac{1}{T_{i}^2} \lp \sum_{s=1}^{T_{i}} Z_{ij}^{(s)} - p_{ij} \rp ^2 \indi{\mc{E}_1} \rb 
     \leq \mbb{E} \lb \lp \sup_{\omega \in \mc{E}_1 } \frac{1}{T_{i}^2} \rp \lp \sum_{s=1}^{T_{i}} Z_{ij}^{(s)} - p_{ij} \rp^2 \rb  \\
     &\stackrel{(a)}{\leq} \lp \frac{1}{T_{i}^* - A \tilde{e}_n/B} \rp^2 \mbb{E} \lb \lp \sum_{s=1}^{T_{i}} Z_{ij}^{(s)} - p_{ij} \rp^2 \rb  \\
    &\stackrel{(b)}{=} \lp \frac{1}{T_{i}^* - A \tilde{e}_n/B} \rp^2 \lp p_{ij}(1-p_{ij})\rp \mbb{E} \lb T_{i} \rb  \\
    & \stackrel{(c)}{\leq}      \lp \frac{1}{T_{i}^* - A \tilde{e}_n/B} \rp^2 \lp p_{ij}(1-p_{ij})\rp  \lp 1 \times \lp T_{i}^* + \frac{(K-1)A\tilde{e}_n}{B}\rp + \delta \times n \rp \\
    & \stackrel{(d)}{\leq} \frac{p_{ij}(1-p_{ij})}{\lp T_{i}^*\rp^2 } \lp 1 + \frac{6 A \tilde{e}_n }{B T_{i}^*} \rp \lp T_{i}^* + \frac{(K-1)A \tilde{e}_n}{B} + n \delta \rp \\
    & = \frac{p_{ij}(1-p_{ij})}{\lp T_{i}^*\rp^2 } \lp T_{i}^* + \frac{(K-1)A\tilde{e}_n}{B} + n\delta + \frac{6 A\tilde{e}_n}{B} + \frac{6(K-1)A^2 \tilde{e}_n^2}{B^2 T_{i}^*} + \frac{6A\tilde{e}_n n \delta}{B T_{i}^*}  \rp
\end{align*}
In the above display, 
\tbf{(a)} follows from the result of Lemma~\ref{lemma:tracking1}, \\
\tbf{(b)} follows from an application of Wald's Lemma, \\
\tbf{(c)} follows from the facts that $Pr(\mc{E}_1^c) \leq \delta\;$, $T_{i} \leq n$ a.s., and the bounds on $T_{i}$ under the event $\mc{E}_1$ given by Lemma~\ref{lemma:tracking1}, \\
\tbf{(d)} follows from the fact that the function $x\mapsto 1/x^2$ is convex, and thus for the function lies below the chord joining the points $(1, 1/1^2)$ and $\lp1/2, 1/(1/2)^2\rp$, and the assumption that $n$ is large enough to ensure that $\frac{A \tilde{e}_n}{B T_{i}^*} < 1/2$. A sufficient condition for this is that $n > \frac{4 (\Mltwo)^2}{\lp \lambda_{\min}^{(\ell_2)}\rp^2}$.

Next,  we have the following:
\begin{align*}
  \mbb{E} \lb (\hat{p}_{ij}^2 - p_{ij}^2 )^2 \indi{\mc{E}_1} \rb & \leq \frac{ (p_{ij}(1-p_{ij})}{T_{i}^*  }
  \lp 1  + \frac{n\delta}{T_{i}^*} \lp 1+ \frac{6\Mltwo \sqrt{n}}{T_{i}^*} \rp + \frac{\Mltwo \sqrt{n}}{T_{i}^*} \lp (K+5) + \frac{6(K-1)\Mltwo \sqrt{n}}{\sqrt{T_{i}^*}}   \rp  \rp. 
\end{align*}
Since $T_{i}^* = \lambda_i^{(\ell_2)}n$, we have 
\begin{align*}
\mbb{E}\lb \lp \hat{p}_{ij} - p_{ij}\rp^2 \indi{\mc{E}_1} \rb & \leq \frac{p_{ij}(1-p_{ij})}{T_{i}^*} \lp 1 + \frac{(K+5)\Mltwo }{\lambda_i^{(\ell_2)}\sqrt{n}} + \frac{\delta}{\lambda_i^{(\ell_2)}}\lp 1 + 6\Mltwo \sqrt{n}\rp + \frac{6(K-1)\lp \Mltwo\rp ^2}{\lp \lambda_i^{(\ell_2)}\rp^2 n} \rp 
\end{align*}
Finally, summing up over the values of $j$ in the range $\{1,2,\ldots, l\}$, we get 
\begin{align*}
\sum_{j=1}^l    \mbb{E}\lb \lp \hat{p}_{ij} - p_{ij} \rp^2 \indi{\mc{E}_1} \rb & \leq \sum_{j=1}^{l}\frac{p_{ij}(1-p_{ij})}{T_{i}^*} + \frac{(K+5)l\Mltwo }{\lp\lambda_i^{(\ell_2)}\rp^2 n^{3/2} } + \mc{O}\lp \frac{\delta l \sqrt{\log n} }{\sqrt{n}} + \frac{\log(n)}{n^2} \rp. 
\end{align*}
We can select $\delta$ small enough to ensure that the regret is of the order $\tilde{\mc{O}} \lp n^{-3/2}\rp$. A suitable choice for this is $\delta = n^{-5/2}$. 
\end{proof}

%% file: appendix_regret_ell1.tex
\subsection{Proof of Lemma~\ref{lemma:ell1_1}}
\label{proof_ell1_lemma1}
\begin{proof}
We can again write the event $\mc{E}_2 \coloneqq \cap_{t} \cap_{i=1}^K \cap_{j=1}^l \mc{E}_2^{(t,i,j)}$. It suffices to show that $Pr\lp \mc{E}_2^{(t,i,j)} \rp \geq 1 - \delta_t$, and the result follows from an application of the union bound  and  the fact that $\sum_{t, i, j} \delta_t = \delta$.

To show that $Pr\lp \mc{E}_2^{(t,i,j)}\rp \geq 1 - \delta_t$, we employ \cite[Theorem~10]{maurer2009empirical} to the collection of random variables $\lp Z_{ij}^{(s)} \rp_{s=1}^{\Tit}$ defined as $Z_{ij}^{(s)} = \indi{X_i^{(s)}=j}$. Then we have $\hat{p}_{ij, t} = \frac{\sum_{s=1}^t Z_{ij}^{(s)}}{t}$, and $Var\lp Z_{ij}^{(s)}\rp = p_{ij}(1-p_{ij})$, and we obtain the required result by applying \cite[Eq.~(3)~and~(4)]{maurer2009empirical}:
\begin{align*}
    Pr\lp \left \lvert \sqrt{\hat{p}_{ij, t}\lp 1- \hat{p}_{ij, t} \rp} - \sqrt{p_{ij}(1-p_{ij})} \right \rvert > \eijtlone \rp  \leq \delta_t 
\end{align*}
\end{proof}


\subsection{Proof of Theorem~\ref{theorem:ell1_regret}}
\label{proof_ell1_theorem}

\paragraph{Value of the constants $A$, $B$ and $\tilde{e}_n$.}
We begin by obtain the values of the constants $A$, $B$ and $\tilde{e}_n$ from Lemma~\ref{lemma:tracking1} corresponding to the $\ell_1$ loss function.

Recall the notation $\lambda_i^{(\ell_1)} \coloneqq \frac{ \lp \cilone \rp^2}{\sum_{k=1}^K\lp \cilone[k]\rp ^2}$, 
$\Clone^2 = \sum_{k=1}^K \lp \cilone[k]\rp^2$, 
$\lambda_{\max}^{(\ell_1)} = \max_{1 \leq i \leq K} \lambda_i^{(\ell_1)}$ and
$\lambda_{\min}^{(\ell_1)} = \min_{1 \leq i \leq K} \lambda_i^{(\ell_1)}$. 

Then we have the following:

\begin{align*}
    A &= \max_{1 \leq i \leq K} \frac{1}{\sqrt{\Tins}} = \frac{1}{\sqrt{\lambda_{\min}^{(\ell_1)}n}} = \mc{O} \lp \frac{1}{\sqrt{n}} \rp,   \\
    B &= \min_{1 \leq i \leq K} \frac{\cilone}{2\lp \Tins\rp^{3/2}} = \frac{\Clone}{\lambda_{\max} n^{3/2}} = \mc{O} \lp \frac{1}{n^{3/2}} \rp, \\
    \tilde{e}_n &= \max_{1 \leq i \leq K} e_i^*) 
    = \sqrt{ \frac{2 l^2 \log(2/\delta_n)}{\Tins-1}}
    \stackrel{(a)}{\leq} \sqrt{ \frac{4 l^2 \log(2/\delta_n)}{\Tins}}
    =\mc{O} \lp \frac{1}{\sqrt{n}} \rp. 
\end{align*}
{In the above display, the inequality \tbf{(a)} follows from the assumption that $n$ is large enough to ensure that $\Tins \geq 3$ for all $1\leq i \leq K$, which is implied by the assumption $n \geq 3/\lambda_{\min}^{(\ell_1)}$.}

\paragraph{Regret Derivation.}
Since we use the approximate objective function for $\ell_1$ loss, we note that the regret can be written as follows:
\begin{align}
\label{eq:ell1_regret_proof1}
    \mc{R}_n \lp \Alglone, \Dlone\rp & \leq \mc{L}_n \lp \Alglone, \Dlone \rp - \f(\cilone, \Tins) + 2 \max_{1 \leq i \leq K}|R_i\lp \Tins \rp |. 
\end{align}

We first bound the term in~\eqref{eq:ell1_regret_proof1},  $\mc{L}_n \lp \Alglone, \Dlone\rp - \f\lp \cilone, \Tins \rp$. Recall that in the final estimates of the probability mass function, we drop the $n+1$ from the subscript, i.e., we write $\hat{p}_{ij}$ and $\hat{P}_i$ instead of $\hat{p}_{ij, n+1}$ and $\hat{P}_{i, n+1}$. 
We next proceed as follows:
\begin{align*}
\mc{L}_n \lp \Alglone, \Dlone \rp =     \mbb{E}\lb \|\hat{P}_i - P_i\|_1 \rb & = \mbb{E} \lb \|\hat{P}_i - P_i \|_1 \lp \indi{\mc{E}_2} + \indi{\mc{E}_2^c}\rp \rb \leq \mbb{E} \lb \|\hat{P}_i - P_i \|_1 \indi{\mc{E}_2} \rb + l\times \delta. 
\end{align*}
where the inequality follows from the fact that $\|\hat{P}_i - P_i\|_1 \leq l$ almost surely, and that $Pr\lp \mc{E}_2^c\rp \leq \delta$. 
Next, we expand the remaining term:
    
\begin{align*}
    \mbb{E}\lb \|\hat{P}_i - P_i \|_1 \indi{\mc{E}_2} \rb & = \mbb{E} \lb \lp \sum_{j=1}^l \sqrt{ |\hat{p}_{ij} - p_{ij}|^2 } \rp \indi{\mc{E}_2}\rb = \mbb{E} \lb \lp \sum_{j=1}^l \sqrt{|\hat{p}_{ij}  - p_{ij}|^2 \indi{\mc{E}_2}}\rp \rb \\
    &\stackrel{(a)}{\leq} \sum_{j=1}^l \sqrt{\frac{2}{\pi} \mbb{E} \lb \lp \hat{p}_{ij} -p_{ij}\rp^2 \indi{\mc{E}_2}        \rb } + \mc{O} \lp \frac{1}{\lp \Tins - A\tilde{e}_n/B \rp^{3/2}} \rp \\
    & \stackrel{(b)}{=} \sum_{j=1}^l \sqrt{\frac{2}{\pi} \mbb{E} \lb \lp \hat{p}_{ij} -p_{ij}\rp^2 \indi{\mc{E}_2}        \rb } + \mc{O}\lp \frac{\sqrt{8}}{\lp \lambda_{\min}^{(\ell_1)} n \rp^{3/2}} \rp \\
    &\coloneqq \sum_{j=1}^l \sqrt{\alpha_{ij}}  + \mc{O}\lp \frac{\sqrt{8}}{\lp \lambda_{\min}^{(\ell_1)} n \rp^{3/2}} \rp.
\end{align*}
The inequality \tbf{(a)} in the above display follows from the  from an  approximation relation between the mean absolute deviation and the standard deviation of Binomial distributions, as proved in \citep[Eq.(2.4)]{blyth1980expected}, while the inequality \tbf{(b)} follows from the definition of the optimal static allocation values $\Tins = \lambda_i^{(\ell_1)} n$ and the values of $A$, $B$ and $\tilde{e}_n$ derived above along with the assumption that $n$ is large enough to ensure that {$ \frac{ A \tilde{e}_n}{B \lambda_{\min}^{(\ell_1)}} \leq 1/2$ or $n \geq 4 \lp \Mlone\rp^2/\lp \lambda_{\min}^{(\ell_1)}\rp^2$}. 

Finally, we analyze the terms $\alpha_{ij}$: 
\begin{align*}
    \alpha_{ij}  & \leq  \mbb{E}\lb \lp \sup_{ \omega \in \mc{E}_2} \frac{1}{T_{i}^2} \rp \lp \sum_{s=1}^{T_{i}}Z_{ij}^{(s)} - p_{ij}  \rp ^2 \rb \leq \frac{1}{\lp \Tins - \frac{A\tilde{e}_n}{B} \rp^2} p_{ij}\lp1-p_{ij}\rp \mbb{E} \lb T_{i} \rb \\
    & \leq \frac{p_{ij}(1-p_{ij})}{\lp \Tins\rp^2}\lp 1 + 6\frac{A\tilde{e}_n}{B \Tins}\rp \lp \Tins + \frac{(K-1)A\tilde{e}_n}{B}  + n \delta \rp \\
    & = \frac{p_{ij}(1-p_{ij})}{\lp \Tins\rp^2} \lp \Tins + \frac{(K-1)A\tilde{e}_n}{B} + n\delta + 
    6\frac{A\tilde{e}_n}{B} + \frac{6(K-1)A^2\tilde{e}_n^2}{B^2\Tins} + \frac{6A\tilde{e}_n n \delta}{B\Tins} \rp\\
    &\coloneqq \f\lp \cilone, \Tins \rp +  E_{ij}, 
\end{align*}
where  $E_{ij}$ represents all the higher order terms. 
Thus, we get the following:
\begin{align*}
    \sum_{j=1}^l \sqrt{\alpha_{ij}} & \leq \sum_{j=1}^{l} \sqrt{\frac{p_{ij}(1-p_{ij})}{\Tins} + E_{ij} }\stackrel{(a)}{\leq} \sum_{j=1}^l \sqrt{ \frac{p_{ij}(1-p_{ij})}{\Tins}} + \sum_{j=1}^l\sqrt{E_{ij}},  
\end{align*}
where \tbf{(a)} follows from the fact that $\sqrt{x + y} \leq \sqrt{x} + \sqrt{y}$ for $x, y>0$. 
Thus to complete the proof we need to analyze the term $E_{ij}$. First, we note that for the choice of the terms $A$, $B$ and $\tilde{e}_n$, we have 
\begin{align*}
    \frac{A\tilde{e}_n}{B} = \frac{\lp \sqrt{\frac{1}{\lambda_{\min}^{(\ell_1)}n}}\rp \lp \sqrt{\frac{4 l^2 \log(2/\delta_n)}{\lambda_{\min}^{(\ell_1)}n}}\rp }
    { \lp \frac{\Clone}{\lambda_{\max}^{(\ell_1)} n^{3/2}}  \rp} \coloneqq \Mlone \sqrt{n}. 
\end{align*}

We now rewrite $E_{ij}$ as follows:
\begin{align*}
    \frac{E_{ij}}{p_{ij}(1-p_{ij})} & = \frac{(K+5)\Mlone}{\lp \lambda_i^{(\ell_1)}\rp^2n^{3/2}} + \frac{\delta}{n \lp\lambda_i^{(\ell_1)}\rp^2} + \frac{6 (K-1)\lp \Mlone\rp^2}{\lp \lambda_i^{(\ell_1)}\rp^3 n^2} + \frac{6\Mlone\delta}{\lp \lambda_i^{(\ell_1)} \rp^3 n^{3/2}}. 
\end{align*}
By setting $\delta = 1/n$, we get the following: 
\begin{align}
    \sum_{j=1}^{l} \sqrt{E_{ij}} & \leq  \sum_{j=1}^{l} \lp 
    \frac{\sqrt{p_{ij}(1-p_{ij})}}{\lambda_i^{(\ell_1)}} \lp \sqrt{\frac{(K+5)\Mlone}{n^{3/2}}} + 
    \sqrt{ \frac{\delta}{n} + \frac{6 (K-1)\lp \Mlone\rp^2}{ \lambda_i^{(\ell_1)} n^2} + \frac{6\Mlone\delta}{ \lambda_i^{(\ell_1)} n^{3/2}}} \rp \rp \label{eq:ell1_regret1} \\
    &\stackrel{(a)}{\leq} \sqrt{\frac{(K+5)\Mlone}{n^{3/2}}} \sum_{j=1}^{l}\frac{\sqrt{p_{ij}(1-p_{ij})}}{\lambda_i^{(\ell_1)}} + \tilde{\mc{O}}\lp \frac{1}{n}\rp \nonumber \\
    & \stackrel{(b)}{\leq} \sqrt{\frac{(K+5)\Mlone}{n^{3/2}}}\lp \frac{l}{2 \lambda_{\min}^{(\ell_1)}} \rp  + \tilde{\mc{O}}\lp \frac{1}{n}\rp \nonumber
\end{align}
where \tbf{(a)} follows from the assumption that $n$ is large enough to ensure that the first term dominates the remaining $\mc{O}\lp 1/n \rp$ terms ({a sufficient condition for this is that $n \geq \lp 24 \lambda_{\min}^{(\ell_1)}\Mlone \rp$}, and  \tbf{(b)} uses the fact that $p(1-p) \leq 1/4$ for all $p \in [0,1]$. Thus we have obtained the regret in tracking the approx-oracle solution, i.e., 
\begin{align}
    \label{eq:ell1_regret_proof2}
    \mc{L}_n \lp \Alglone, \Dlone\rp - \f\lp \cilone, \Tins \rp
    & \leq \sqrt{\frac{(K+5)\Mlone}{n^{3/2}}}\lp \frac{l}{2 \lambda_{\min}^{(\ell_1)}} \rp  + \tilde{\mc{O}}\lp \frac{1}{n}\rp. 
\end{align}

Finally, to obtain the approximation error in~\eqref{eq:ell1_regret_proof1}, we again apply~\citep[Eq.~(2.4)]{blyth1980expected} to note that the approximation error, $\max_{i \in [K]} |\f(\cilone, \Tins) \gamma_i \lp \Tins \rp$ is $\mc{O} \lp 1/(\lambda_{\min}^{\ell_1})^{3/2}\rp$. 
This concludes the proof that the excess regret for the $\ell_1$ loss function is of the order of $\tilde{\mc{O}}\lp n^{-3/4} \rp$. 

%% file: appendix_regret_f_divergence.tex

\subsection{Proof of Lemma~\ref{eq:lemma:f_div_remainder}}
\label{subsubsec:proof-lemma5}

Note that for any fixed value of $\hat{p}_{ij}$, the remainder term $R_{ij, r+1}$ can be written in two ways as follows:
\begin{align}
\label{eq:remainder1}
    R_{ij, r+1} \stackrel{(a)}{=} \sum_{m \geq r+1} \frac{f^{(m)}(1)}{m! p_{ij}^{m-1}} \lp \hat{p}_{ij} - p_{ij} \rp^m \; \stackrel{(b)}{=} \frac{f^{(r+1)}(z)}{m! p_{ij}^r} \lp \hat{p}_{ij} - p_{ij} \rp^{r+1}. 
\end{align}
The equality \tbf{(b)} is the Lagrange form of remainder, where the value $z$ is a function of $\hat{p}_{ij}/p_{ij}$.

Introduce the definition $Q_{\epsilon} \coloneqq \{ q\;:\; |q-p_{ij}|/p_{ij} \geq \epsilon \}$ for a fixed $0<\epsilon<1$. In order to upper bound the expected value of $R_{ij, r+1}$, we consider two cases, \tbf{(1)} first when the random variable $\hat{p}_{ij}$ lies in $Q_{\epsilon}^c$, and \tbf{(2)} second, when $\hat{p}_{ij} \in Q_{\epsilon}$. 

When $\hat{p}_{ij} \not \in Q_\epsilon$, we use the equality \tbf{(a)} in~\eqref{eq:remainder1} to note that 
\begin{align}
   & R_{ij, r+1}\indi{\hat{p}_{ij} \not \in Q_{\epsilon}} \leq p_{ij}\lp \frac{\hat{p}_{ij} - p_{ij}}{p_{ij}}\rp^{r+1} \lp \max_{m \geq r+1} \frac{ f^{(m)}(1)}{m!} \rp \frac{1}{1-\epsilon} \nonumber \\
    \Rightarrow &\mbb{E} \lb R_{ij, r+1} \indi{ \hat{p}_{ij} \not \in Q_\epsilon} \rb  \leq \lp \frac{C_1}{p_{ij}^r (1-\epsilon)} \rp \mbb{E} \lb \lp  \hat{p}_{ij} - p_{ij} \rp^{r+1} \rb.  \label{eq:remainder2}
\end{align}

Next we consider the second case. Here we note that for any $q \in Q_\epsilon$, the remainder can be written as (with dependence on $q$ made explicit): 
\begin{align*}
    R_{ij, r+1}(q) & = p_{ij} \lp f\lp \frac{q}{p_{ij}} \rp - \sum_{m=0}^r \frac{f^{(m)}(1)}{m! } \lp \frac{q - p_{ij}}{p_{ij}} \rp^m \rp  \\
    & = \frac{ f^{(r+1)}(z_q)}{m! p_{ij}^r} \lp q - p_{ij} \rp ^{r+1}. 
\end{align*}
In the second equality, the term $z_q$ varies with $q$. 
Now, since the set $Q_{\epsilon}$ is compact, and by the local boundedness of the function $f$ (assumption~\tbf{(f1)} in \S~\ref{subsec:f_divergence}), we have $\sup_{q \in Q_{\epsilon}} R_{ij, r+1} \coloneqq \gamma_\epsilon < \infty$. This implies the following:
\begin{align}
    \sup_{q \in Q_\epsilon} \; \frac{ f^{(r+1)}(z_q)}{m!} &  \leq \frac{1}{p_{ij}} \lp \frac{\gamma_\epsilon}{\epsilon^{r+1}} \rp  \coloneqq  \frac{C_\epsilon}{p_{ij}}. 
    \label{eq:remainder6}
\end{align}
Using this inequality, we can proceed as follows:
\begin{align}
\label{eq:remainder3}
   \mbb{E} \lb  R_{ij, r+1} \indi{ \hat{p}_{ij} \in Q_\epsilon} \rb  & \leq C_\epsilon \mbb{E} \lb \lp\frac{ \hat{p}_{ij} - p_{ij} }{p_{ij}} \rp^{r+1}\rb 
\end{align}

Combining~\eqref{eq:remainder2} and~\eqref{eq:remainder3}, we get 
\begin{align}
    \mbb{E} \lb R_{ij, r+1} \rb & \leq \lp \inf_{0<\epsilon<1} \lp \frac{C_1 p_{ij}}{(1-\epsilon)} + C_\epsilon \rp \rp    \mbb{E} \lb \lp\frac{ \hat{p}_{ij} - p_{ij} }{p_{ij}} \rp^{r+1}\rb \label{eq:remainder6} \\
    & \stackrel{(a)}{\leq} \lp \inf_{0<\epsilon<1} \lp \frac{C_1 p_{ij}}{(1-\epsilon)} + C_\epsilon \rp \rp   \lp \frac{3 e^{2/e} (r+1)}{2}\rp^{r+1} (p_{ij}T_i)^{-(r+1)/2} 
    \label{eq:remainder4}\\
    & \coloneqq C_{f, r+1} (p_{ij}T_i)^{-(r+1)/2}. \label{eq:remainder5} 
\end{align}
In the inequality \tbf{(a)} above, we used the fact that the random variable $\hat{p}_{ij} - p_{ij}$ is subgaussian with parameter $\sigma = \sqrt{3/(2T_i)}$, and then used the upper bound  on the $(r+1)^{th}$ moment of a subgaussian random variable~\citep[Prop.~3.2]{rivasplata2012subgaussian}.

\subsection{Proof of Lemma~\ref{lemma:f_div_moments}}
\label{proof_f_div_lemma2}

Since $W_{ij, T_i} = \sum_{s=1}^{T_i} \tilde{Z}_{ij}^{(s)} = \sum_{s=1}^{T_i} Z_{ij}^{(s)} - T_i p_{ij}$, where $p_{ij} = \mbb{E}\lb Z_{ij}^{(s)}\rb$ for all $s \geq 1$, we can upper bound $W_{ij, T_i}$ by upper bounding the first term and lower bounding the second term in its definition. This is achieved by exploiting the non-negativity of ${Z}_{ij}^{(s)}$ and the bounds on the random variable $T_i$ under the event $\mc{E}_\delta$. The rest of the proof then proceeds by using the binomial expansion of the expression obtained (i.e., the term bounding $W_{ij, T_i}$). 

\begin{align*}
    \mbb{E} \lb \lp W_{ij,T_i}\rp^m \indi{\mc{E}_\delta} \rb  & 
    = \mbb{E} \lb \lp \sum_{s=1}^{T_i} Z_{ij}^{(s)} - T_i p_{ij} \rp^m \indi{\mc{E}_\delta} \rb 
     \stackrel{(a)}{\leq} \mbb{E} \lb \lp \sum_{s=1}^{\tau_{1,i}} Z_{ij}^{(s)} - \sum_{s=1}^{\tau_{0,i}}p_{ij} \rp ^m \indi{\mc{E}_\delta} \rb \\
 & \stackrel{(b)}{\leq} \mbb{E}\lb \lp \sum_{s=1}^{\tau_{0,i}} \tilde{Z}_{ij}^{(s)} + \sum_{s=\tau_{0,i}+1}^{T_i} 1 \rp ^m \indi{\mc{E}_\delta}\rb 
    = \mbb{E}\lb \lp W_{ij, \tau_{0,i}} + \lp \tau_{1,i}-\tau_{0,i}\rp \rp^m \indi{\mc{E}_\delta} \rb \\
    &\leq \mbb{E}\lb \lp W_{ij, \tau_{0,i}}\rp^m \rb +
     \sum_{k=1}^m \lp \tau_{1,i} - \tau_{0,i}\rp^k {m \choose k }\mbb{E}\lb \lp W_{ij, \tau_{0,i}} \rp^{m-k} \rb. 
\end{align*}
In the above display, \\
\tbf{(a)} follows from the fact that $Z_{ij}^{(s)} \geq 0$ a.s. and that under the event $\mc{E}_\delta$, we have $\tau_{0,i} \leq T_i \leq \tau_{1,i}$. This implies that $\sum_{s=1}^{T_i} Z_{ij}^{(s)} \leq \sum_{s=1}^{\tau_{1,i}} Z_{ij}^{(s)}$ and $T_ip_{ij} \geq \tau_{0,i}p_{ij}$. 
\\
\tbf{(b)} follows from the fact that $|\tilde{Z}_{ij}^{(s)}| \leq 1$ a.s.

\subsection{Proof of Theorem~\ref{theorem:regret_f_div}}
\label{proof_f_div_theorem}

Note that from~\eqref{eq:approx_reg2} in Prop.~\ref{prop:approximate_regret} in Appendix~\ref{appendix:approximate}, we have  $\mc{R}_n ( \Algfdiv, \Dfdiv) \leq \mc{L}_n ( \Algfdiv, \Dfdiv) - \f_i(\Tins) + 2 \max_{k \in [K]} |R_{k, r+1} (\Tins)|$.

The risk term can be further decomposed as follows:
\begin{align*}
    \mc{L}_n \lp \Algfdiv, \Dfdiv\rp &= \mbb{E}\lb \Dfdiv^{(r)}(\hat{P}_i, P_i) + R_{i, r+1}(T_i)\rb \\
    & \leq \mbb{E} \lb \Dfdiv^{(r)} \lp \hat{P}_i, P_i \rp \lp \indi{\mc{E}_\delta} + \indi{\mc{E}_\delta^c } \rp \rb + \mbb{E}\lb R_{i,r+1}(T_i)\rb. 
\end{align*}

We first obtain an upper bound on the term term $\mc{L}_n \lp \Algfdiv,\Dfdiv\rp$.
Suppose $\hat{P}_i$ is the empirical estimate of $P_i$ constructed from the samples collected through the adaptive scheme $\mc{A}_{(f)}$ after $n$ rounds. Then we have the following:
\begin{align}
\mc{L}_n \lp \Algfdiv, \Dfdiv\rp &=
\mbb{E} \lb D_f\lp \hat{P}_i , P_i \rp \rb = 
\mbb{E}\lb \Dfdiv^{(r)}(\hat{P}_i, P_i) + R_{i, r+1}(T_i)\rb \\
& = \mbb{E} \lb \sum_{m =1}^r \sum_{j=1}^l \frac{ f^{(m)}(1)}{m! p_{ij}^{m-1}} \lp \hat{p}_{ij}- p_{ij} \rp^m + \sum_{m \geq r+1} \sum_{j=1}^l \frac{ f^{(m)}(1)}{m! p_{ij}^{m-1}} \lp \hat{p}_{ij}- p_{ij} \rp^m \rb \nonumber\\
& \coloneqq \text{\tbf{term1}} + \text{\tbf{term2}} \label{eq:f_div_proof1}
\end{align}
We consider the two terms above separately. 
\begin{align*}
    \text{\tbf{term1}} &= \mbb{E} \lb \Dfdiv^{(r)} \lp \hat{P}_i, P_i \rp \lp \indi{\mc{E}_\delta} + \indi{\mc{E}_\delta^c } \rp \rb 
     = \mbb{E} \lb 
    \sum_{m=1}^r \sum_{j=1}^l
   \lp  \frac{f^{(m)}(1)}{m!p_{ij}^{m-1}} 
    \frac{ \lp \sum_{s=1}^{T_{i,n}} \tilde{Z}^{(s)}_{ij} \rp^m }{T_{i,n}^m }
    \rp
    \lp \indi{\mc{E}_\delta} + \indi{\mc{E}_\delta^c } \rp
    \rb \\
    & \stackrel{(a)}{\leq} \sum_{m=1}^r \sum_{j=1}^l \frac{f^{(m)}(1)}{m! p_{ij}^{m-1}} \lp \frac{1}{\tau_{0,i}^m}\mbb{E}\lb \lp \sum_{s=1}^{\tau_{0,i}} \tilde{Z}_{ij}^{(s)} \rp^m \rb \rp 
    + \sum_{m=1}^r \sum_{j=1}^l \lp \frac{f^{(m)}(1)}{\tau_{0,i}^m m! p_{ij}^{m-1}}\beta_m^{(ij)}\lp \tau_{0,i}, \tau_{1,i} \rp + \frac{f^{(m)}(1)\delta }{m! p_{ij}^{m-1}} \rp  \\
    & = \f_i\lp \tau_{0,i}\rp  + \psi_{2,i}  
    + \psi_{3,i}.
\end{align*}
The inequality \tbf{(a)} in the above display uses the following facts:\\
(i) Under the event $\mc{E}_\delta$, the term $T_{i}$ is lower bounded by $\tau_{0,i}$ and thus $1/T_{i}^m$ is upper bounded by $1/\tau_{0,i}^m$\\
(ii) We use Lemma~\ref{lemma:f_div_moments} to upper bound the moment of $\mbb{E}\lb \lp \sum_{s=1}^{\tau_{0,i}} \tilde{Z}_{ij}^{(s)} \rp^m \rb$, which gives us the additional term consisting of $\beta^{(ij)}_m\lp \tau_{0,i}, \tau_{1,i}\rp$. \\
(iii) Under the event $\mc{E}_\delta^c$, we can upper bound the required moment by its worst case value of $1$ multiplied by the bound on the probability of $\mc{E}_{\delta}^c$, i.e., $\delta$.

Next, we need to get the upper bound on the second term in~\eqref{eq:f_div_proof1}. 
For any $t \geq 1$, we recall the notation $W_{ij, T_i} = \sum_{s=1}^{T_i} \tilde{Z}_{ij}^{(s)}$.

Since \tbf{term2} is the expectation of the remainder of the $r$-term approximation of $D_f\lp \hat{P}_i, P_i \rp$ with $P_i$ as constructed by the adaptive scheme, we can write the following:
\begin{align*}
    \text{term}2 & \stackrel{(a)}{\leq} \mbb{E} \lb C_{f, r+1} \lp \frac{ W_{ij, T_i} }{p_{ij}T_{i}} \rp^{r+1} \rb = C_{f, r+1} \lp \mbb{E} \lb \lp \frac{ W_{ij, T_i} }{p_{ij}T_{i}} \rp^{r+1} \lp \indi{\mc{E}_\delta} + \indi{\mc{E}_{\delta}^c } \rp \rb \rp \\
    & \stackrel{(b)}{\leq} \frac{C_{f, r+1}}{(p_{ij})^{r+1}} \mbb{E} \lb \lp W_{ij, T_i} \rp^{r+1} \rb \lp \delta + \frac{1-\delta}{\tau_{0,i}^{r+1} } \rp \\
\end{align*}
In the above display, \\
\tbf{(a)} employs the upper bound on the remainder term derived in~\eqref{eq:remainder6} in  the proof of Lemma~\ref{lemma:f_div_moments}, \\
\tbf{(b)} uses the fact that $T_{i} \geq \tau_{0,i} \indi{\mc{E}_\delta} + \indi{\mc{E}_{\delta}^c}$.


It remains to obtain an appropriate upper bound on the term $\mbb{E} \lb \lp W_{ij, T_i} \rp^{r+1} \rb$. We first observe the following:
\begin{align*}
W_{ij, T_i}  \leq \max_{1 \leq s \leq n} W_{ij, s} \; \coloneqq \mc{W}_{ij, n}, 
\end{align*}
where the inequality holds in an almost sure sense, and it follows from the fact that by definition $T_{i} \leq n$ almost surely for all $1\leq i \leq K$.
Next, using the fact that $(W_{ij, s})_{s \geq 1}$ is a martingale sequence, we obtain by an application of Doob's $L_p$ maximal inequality~\citep[Theorem~4.4.6]{durrett2019probability}, the following inequality: 
\begin{align*}
    \mbb{E} \lb \lp W_{ij, T_i} \rp^{r+1}\rb & \leq \mbb{E}\lb \lp \mc{W}_{ij, n}\rp^{r+1} \rb \leq \lp \frac{r+1}{r+1-1} \rp^{r+1} \mbb{E} \lb \left \lvert W_{ij, n} \right \rvert^{r+1} \rb \\
    & \stackrel{(a)}{\leq} \lp \frac{r+1}{r+1-1} \rp^{r+1} 
    \lp \sqrt{\frac{3n}{2}} e^{1/e} \sqrt{r+1} \rp^{r+1} \\
    & \leq  e^2 \lp 3.2(r+1)\rp^{(r+1)/2} n^{(r+1)/2} \coloneqq C_{(r)} n^{(r+1)/2}. 
\end{align*}
The inequality \tbf{(a)} in the above display follows from the observation that $W_{ij, n}$ is a  subgaussian random variable with $\sigma=\sqrt{3n/2}$, and that the $r+1$ moment of a $\sigma$-subgaussian random variable is upper bounded by $\lp \sigma e^{1/e}\sqrt{r+1} \rp^{r+1}$~\citep[Prop.~3.2]{rivasplata2012subgaussian}. 
Thus we finally obtain 
\begin{align*}
    \text{term2} & \leq \frac{C_{f, r+1}}{(p_{ij})^{r+1}} \lp \delta + \frac{1-\delta}{\tau_{0,i}^{r+1}} \rp C_{(r)} n^{(r+1)/2} 
= \frac{C_{f, r+1} C_{(r)}}{(p_{ij})^{r+1}} \lp \delta n^{(r+1)/2} + \lp \frac{ \sqrt{n}}{\tau_{0,i}}\rp^{r+1} \rp \coloneqq \psi_{1,i}.  
\end{align*}


To summarize, we have shown that the risk $\mc{L}_n\lp \Algfdiv, \Dfdiv\rp$, i.e., the expected $f$-divergence between the empirical estimate $\hat{P}_i$ constructed using the adaptive sampling scheme $\mc{A}$ with a budget of $n$ samples can be upper bounded as 
\begin{align}
\label{eq:f_div_reg1}
    \mc{L}_n \lp \Algfdiv, D_f \rp & \leq \max_{1\leq i \leq K} \lp  \f_i\lp \tau_{0,i}\rp 
    + \psi_{1,i} 
    + \psi_{2,i} 
    + \psi_{3,i} \rp 
\end{align}

We can now apply the regret decomposition bound given in~\eqref{eq:approx_reg2} in Proposition~\ref{prop:approximate_regret} in Appendix~\ref{appendix:approximate} to get the following:
\begin{align*} 
\mc{R}_n \lp \Algfdiv, \Dfdiv \rp & \leq \mc{L}_n \lp \Algfdiv, \Dfdiv \rp - \f_i\lp \Tins \rp + 2\max_{1 \leq k \leq K}|R_{i,r+1}(\Tins)| \\
&  \leq \mc{L}_n \lp \Algfdiv, \Dfdiv \rp - \f_i\lp \Tins \rp + \psi_4 \\ 
& \leq  \max_{1 \leq i \leq K} \lp \f_i(\tau_{0,i}) - \f_i\lp \Tins \rp + \Psi_i \rp. 
\end{align*}
In the last inequality which follows from an application of~\eqref{eq:f_div_reg1}, we used the fact that $\f_i\lp \Tins \rp$ and $\psi_4$ do not depend on $i$.  This completes the proof. 



%% file: appendix_regret_kl.tex
We begin with a simple concentration result on the estimates $\hat{p}_{ij, t}$ of $p_{ij}$, which is needed for the analysis of the adaptive scheme for KL-divergence and the separation distance (Appendix~\ref{appendix:analysis_separation}).  This result is similar to Lemma~\ref{lemma:ell2_1}, with the difference that here we use the Hoeffding inequality instead of the relative Chernoff bound employed in Lemma~\ref{lemma:ell2_1}. This is because the confidence interval constructed in Lemma~\ref{lemma:ell2_1} depends on the unknown parameter $p_{ij}$, and hence this confidence interval is not \emph{observable}. This was not an issue in designing the adaptive scheme for $\ell_2$ distance, since we only needed the term $\sum_j e_{ij, t}^{(\ell_2)}p_{ij}$ in which we  used the fact that $\sum_{j=1}^l p_{ij}^{3/2} \leq 1$. 

We now state the concentration lemma.
\begin{lemma}
\label{lemma:concentration2}
Define \begin{small}$\delta_t \coloneqq 6 \delta /(K l \pi^2 t^2)$\end{small}, \begin{small}$e_{ij, t} \coloneqq \sqrt{2 \log (2/\delta_t)/\Tit}$\end{small}, and the event \begin{small}$\mathcal{E}_3 \coloneqq  \bigcap_{t\in[n]}  \bigcap_{i\in[K]} \bigcap_{j\in[l]}\{|\hat{p}_{ij, t} - p_{ij}| \leq e_{ij, t}\}$\end{small}. Then, we have \begin{small}$\mathbb{P}(\mc{E}_3) \geq 1- \delta$\end{small}.
\end{lemma}
\begin{proof}
Write $\mc{E}_3 = \cap_{t, i,j} \mc{E}_3^{(t,i,j)}$ where the events $\mc{E}_3^{(t,i,j)}$ are defined implicitly from the definition of $\mc{E}_3$.
Then by union bound we have that $Pr\lp (\mc{E}_3)^c \rp \leq  \sum_{t=1}^\infty \sum_{i=1}^K
\sum_{j=1}^l Pr\lp \lp \mc{E}_3^{(t,i,j)}\rp^c\rp$. To complete the proof, it suffices to show that $Pr\lp
\lp \mc{E}_3^{(t,i,j)}\rp^c \rp \leq \delta_t = \frac{6 \delta}{K l t^2 \pi^2}$ which implies the result.
We proceed by using the additive form of Chernoff's inequality. 
\begin{align*}
    Pr \lp |\hat{p}_{ij, t} - p_{ij}| \geq e_{ij, t} \rp 
     &= \sum_{s=0}^t \lp Pr \lp |\hat{p}_{ij, t} - p_{ij}|\geq e_{ij, t}, \; \Tit = s \rp      \rp
 \leq 2 \sum_{s=0}^t Pr \lp \Tit = s\rp \exp \lp -\frac{ e_{ij,t}^2\, s}{2} \rp  \\
    & = 2\sum_{s=0}^t Pr\lp \Tit = s\rp \exp \lp -\frac{ 2 \log (2/\delta_t) s}{2s} \rp 
     = \delta_t 
\end{align*}

This completes the proof. 
\end{proof} 


\subsection{Proof of Lemma~\ref{lemma:kl_1}}
\label{proof_kl_lemma1}

The first inequality of~\eqref{eq:kl1} follows directly from the definition of the event $\mc{E}_3$, under which $p_{ij} \geq \hat{p}_{ij, t} - e_{ij, t}$. 

We next claim that under the conditions of the lemma, we must have $\hat{p}_{ij, t} - e_{ij, t}>\eta/2$. We prove this statement by contradiction.
Assume that $ \eta/2 \geq \hat{p}_{ij, t} - e_{ij, t} \geq 7/2e_{ij, t}-e_{ij, t} = 5/2e_{ij, t}$, or equivalently $5e_{ij, t} \leq \eta$. Since $\hat{p}_{ij, t} + e_{ij, t} \geq p_{ij}$ under the event $\mc{E}_1$, we also have the following chain of inequalities: 
\begin{align*}
\eta < p_{ij} \leq \hat{p}_{ij, t} + e_{ij, t} =     \hat{p}_{ij, t} - e_{ij, t} + 2e_{ij, t}  \leq \eta/2 + 2e_{ij, t} 
\end{align*}
which implies that $e_{ij, t}> \eta/4$ or equivalently $4e_{ij, t} > \eta$. This gives us the required contradiction. 

We now invoke the convexity of the mapping $x \mapsto 1/x$ to  claim the following:
\begin{align*}
    \frac{1}{\hat{p}_{ij, t}- e_{ij, t}} & \leq  \frac{1}{p_{ij}} + \lp\frac{ 2/\eta - 1/p_{ij}}{\eta/2 - p_{ij}}\rp \lp \hat{p}_{ij, t} - e_{ij, t} - p_{ij} \rp \\
    &  = \frac{1}{p_{ij}} + \frac{2 \lp p_{ij} - \hat{p}_{ij, t} + e_{ij, t}\rp }{\eta p_{ij}} \leq \frac{1}{p_{ij}} + \frac{ 4e_{ij, t}}{\eta p_{ij}} \leq \frac{1}{p_{ij}} + 
    \sqrt{\frac{32 \log(2/\delta_t)}{\Tit \eta^4}}.
\end{align*}
The result in~\eqref{eq:kl2} follows by taking the summation over $1 \leq j \leq l$. 

\subsection{Deviation of $\Tins$ from the uniform allocation}
\label{proof_kl_lemma2}
Before proceeding to the proof of Theorem~\ref{theorem:kl}, we first present a result that shows that the optimal static allocation for the approximate tracking objective does not deviate much from the uniform allocation (i.e., the first order approximation). 
\begin{lemma}
\label{lemma:kl2}
Suppose $P_i \in \Delta_l^{(\eta)}$ for $1\leq i \leq K$. Define $T_0 = n/K$, and let $\Tins$ for $1\leq i\leq K$ be the optimal static allocation according to the first two terms of~\eqref{eq:kl1}. 
For $1 \leq i \leq K$, define $c_i^{(KL)} \coloneqq \frac{1}{12}\lp \sum_{j=1}^l \frac{1}{p_{ij}} - 1\rp$ and let $c_{\max}^{(KL)}$ and $c_{\min}^{(KL)}$ denote the maximum and minimum values of $c_i^{(KL)}$ as $i$ varies from $1$ to $K$.
Then we have the following:
\begin{align*}
\left \lvert \Tins - T_0 \right \rvert & \leq \frac{c_{\max}^{(KL)} - c_{\min}^{(KL)}}{l-1} \quad \forall \; 1 \leq i \leq K. 
\end{align*}
\end{lemma}

\begin{proof}

Let $\f(c_i^{(KL)}, T) \coloneqq \frac{l-1}{2T} + \frac{c_i^{(KL)}}{T^2}$.
Then we note that for $1\leq i \leq K$, we have $\f\lp c_{\min}^{(KL)}, T_0\rp \leq  \f(c_i^{(KL)}, T_0) \leq \f\lp c_{\max}^{(KL)}, T_0 \rp$.
Furthermore, we know that by definition of $\Tins$, we must have $\f(c_i^{(KL)}, \Tins) = \f(c_j^{(KL)}, \Tins)$ for all $1 \leq i,j \leq K$. 
This implies that $\f\lp c_{\min}^{(KL)}, T_0\rp \leq  \f(c_i^{(KL)}, \Tins) \leq \f\lp c_{\max}^{(KL)}, T_0 \rp$ for $1\leq i \leq K$.
Together, these inequalities imply that $\left \lvert  \f(c_i^{(KL)}, T_0) - \f\lp c_i^{(KL)}, \Tins\rp \right \rvert \leq \frac{c_{\max}^{(KL)} - c_{\min}^{(KL)}}{T_0^2} \coloneqq G$.
Due to the convexity of the function $\f(c_i^{(KL)}, \cdot)$, a necessary condition for this to hold is that $G \geq |g_i(T_0)|  \lp \left \lvert \Tins - T_0 \right \rvert \rp $ where $g_i(T_0) = \left .\frac{\partial \f(c_i^{(KL)}, T)}{\partial T} \right \rvert_{T=T_0}$.
This implies that $|\Tins - T_0| \leq \frac{G}{|g_i(T_0)|}$.
The result then follows from the fact that $|g_i(T_0)| = \frac{l-1}{T_0^2} + \frac{2c_i^{(KL)}}{T_0^3} \geq 
\frac{l-1}{T_0^2}$. 
\end{proof}

\begin{remark}
\label{remark:kl_chi}
We can show that by using the uniform allocation, the regret incurred in learning distribution in terms of KL-divergence is $\mc{O} \lp n^{-2} \rp$. This regret is asymptotically slower than the $\mc{O} \lp n^{-5/2}\rp$ rate presented in Theorem~\ref{theorem:kl}. However, under some parameter regimes, this rate may actually  be faster than the one presented in Theorem~\ref{theorem:kl} due to the larger hidden constants in front of the $\mc{O}\lp n^{-5/2}\rp$ regret bound of Theorem~\ref{theorem:kl}. 
\end{remark}

\subsection{Proof of Theorem~\ref{theorem:kl}}
\label{proof_kl_theorem}

{We assume that $n$ is large enough to ensure that $\Tins \geq n/(2K)$ for $1 \leq i \leq K$. A sufficient condition for this is $n \geq \frac{K(l-\eta)}{6(l-1)}$.}

\paragraph{Value of Constants $A$, $B$ and $\tilde{e}_n$.} 
We have the following:
\begin{align*}
    A &= \max_{1\leq i \leq K}\left . \frac{ \partial \f(c, \Tins)}{\partial c} \right \rvert_{c=c_i^{(KL)}} = \max_{1 \leq i \leq K} \frac{1}{\lp \Tins\rp^2}
    \leq \frac{4K^2}{n^2}. \\
    B &= \max_{1 \leq i \leq K} \left. \frac{ \partial \f(c_i^{(KL)}, T)}{\partial T} \right \rvert_{T = \Tins} = \min_{1 \leq i \leq K} \frac{l-1}{\lp 2\Tins\rp^2} + \frac{2c_{i}^{(KL)}}{\lp \Tins\rp^3} 
    \geq \frac{(l-1)K^2}{n^2}.\\
    \tilde{e}_n &= \max_{1 \leq i \leq K} \frac{1}{\eta^3} \sqrt{ \frac{48 \log(\delta_n)}{\Tins}}
    \leq \frac{1}{\eta^3} \sqrt{ \frac{96 K \log(\delta_n)}{n}}. 
\end{align*}

Thus we have
\begin{align}
    \frac{A\tilde{e}}{B} &\leq \frac{4 \sqrt{96K \log(\delta_n)}}{(l-1) \eta^3 \sqrt{n}} \coloneqq M^{(KL)} \frac{1}{\sqrt{n}}.
\end{align}

We next proceed to analyze the regret of the adaptive scheme when compared  to the oracle static allocation scheme. As suggested by Theorem~\ref{theorem:regret_f_div}, we separately obtain upper bounds on the three terms $\psi_{1,i} \lp \rp$, $\psi_4$ and $\mbb{E} \lb D_f^{(r)}\lp \hat{P}_i, P_i \rp \rb$. 

\subsubsection{Bound on $\bm{\psi_{1,i}}$.}
We proceed as follows:
\begin{align*}
\psi_{1,i}  & = \sum_{j=1}^l \frac{C_{KL, r+1} C_{(r)}}{p_{ij}^{r+1}} \lp \delta n^{(r+1)/2} + \lp \frac{3K}{\sqrt{n}}\rp^{r+1} \rp 
    \leq \frac{ l C_{KL, r+1}, C_{(r)}}{\eta^{r+1}} \lp \delta n^{(r+1)/2} + (3K)^{(r+1)}  n^{-(r+1)/2} \rp. 
\end{align*}
In the above display, $C_{KL, r+1}$ is the instance of the constant $C_{f, r+1}$ for the case of KL-divergence.  Since $r=5$, we see that an appropriate choice of $\delta$ above is {$\delta = (3K/n)^{6}$}, which gives us 
\begin{align}
\label{eq:kl_proof0}
    \psi_{1,i} = \frac{ 2l C_{KL, 6} C_{(5)} (3K)^6}{\eta^6}n^{-3} = \mc{O} \lp n^{-3} \rp. 
\end{align}

\subsubsection{Bound on $\bm{\psi_4}$.} 
We proceed as follows:
\begin{align*}
\psi_4 & = \max_{1 \leq i \leq K } \lp  \sum_{j=1}^l C_{KL, r+1} \lp \Tins p_{ij} \rp^{-(r+1)/2} \rp 
    \stackrel{(a)}{\leq} \sum_{j=1}^l C_{KL, r+1} \lp \eta n/(2K) \rp ^{-(r+1)/2}. 
\end{align*}
In the above display, we used the fact that since $P_i \in \Delta_l^{(\eta)}$ we must have $p_{ij} \geq \eta$ and that $n$ is large enough to ensure that $\Tins \geq n/2K$. 
For the value of $r=5$ in the Taylor's expansion of the KL-divergence, we get the following bound on $\psi_4$
\begin{align}
    \label{eq:kl_proof00}
    \psi_4 \leq l C_{KL, 6}  \lp \frac{2K}{\eta n} \rp^{3}. 
\end{align}

\subsubsection{Bound on $\bm{\psi_{3,i}}$}
We have the following:
\begin{align}
\label{eq:kl_proof10}
\psi_{3,i} = \delta \sum_{m=1}^r \sum_{j=1}^l \frac{ f^{(m)}(1)}{m! p_{ij}^{m-1}} \leq \frac{\delta l 5}{\eta^4} = \frac{5l (3K)^6}{\eta^4}n^{-6}. 
\end{align}
The first inequality in the above display uses the following facts: since $f(x) = x\log x$ for KL-divergence, we have $f^{(m)}(1)/m! \leq 1$, and that since $P_i \in \Delta_l^{(\eta)}$ we have $p_{ij} \geq \eta$. Finally the second inequality uses the value of $\delta = (3K)^6n^{-6}$.

\subsubsection{Bound on $\varphi_r\lp \tau_{0,i}\rp - \varphi_r\lp \Tins\rp $.}

Next we analyze the contribution of the term: $\varphi_r\lp \tau_{0,i} \rp - \varphi_r \lp \Tins\rp$ to the regret.
 We proceed as follows:
\begin{align*}
    \varphi_r\lp \tau_{0,i} \rp - \varphi_r\lp \Tins \rp & \coloneqq \frac{l-1}{2\tau_{0,i} } - \frac{l-1}{2\Tins} + \frac{c_i^{(KL)}}{\lp \tau_{i,n}\rp^2} - \frac{c_i^{(KL)}}{\lp \Tins\rp^2}
\end{align*}
Consider the first term $(l-1)/(2\tau_{0,i})$. 
\begin{align}
\frac{l-1}{2\tau_{0,i}} & = \frac{l-1}{2\Tins \lp 1 - \frac{\Tins - \tau_{0,i}}{\Tins} \rp}
\stackrel{(a)}{\leq} \frac{l-1}{2\Tins} \lp 1 + 2 \frac{\Tins - \tau_{0,i}}{\Tins} \rp  \nonumber \\
& \leq \frac{l-1}{2\Tins} + \frac{ (l-1)(\tau_{1,i} - \tau_{0,i})}{\lp \Tins\rp^2}
\stackrel{(b)}{\leq} \frac{l-1}{2\Tins} + \frac{ 4K^2(l-1) \lp c_{\max}^{(KL)} - c_{\min}^{(KL)}\rp }{n^{5/2}}. 
\label{eq:kl_proof1}
\end{align}    
In the above display, \\
\tbf{(a)} follows from the convexity of the mapping $x \mapsto 1/x$, and the graph of this mapping lies below the chord connecting $(1,1)$ and $(1/2, 2)$. \\
\tbf{(b)} follows from the assumption that $n$ is large enough to ensure that $\Tins \geq n/(2K)$ for all $1 \leq i \leq K$. 

Proceeding in a similar way, we can obtain the following bound on the second term $\frac{c_i^{(KL)}}{\tau_{0,i}^2}$:  
\begin{align}
    \frac{c_i^{(KL)}}{\tau_{0,i}^2} & \leq \frac{c_i^{(KL)}}{\lp \Tins\rp^2 } + 
    \frac{ 32 K^3 c_i^{(KL)} \lp c_{\max}^{(KL)} - c_{\min}^{(KL)}\rp }{n^{7/2}}. 
    \label{eq:kl_proof2}
\end{align}
Together,~\eqref{eq:kl_proof1} and~\eqref{eq:kl_proof2} imply that 
\begin{align}
    \varphi_r\lp \tau_{0,i} \rp - \varphi_r \lp \Tins\rp &\leq 
    \frac{ 4K^2(l-1) \lp c_{\max}^{(KL)} - c_{\min}^{(KL)}\rp }{n^{5/2}} + 
    \frac{ 32 K^3 c_i^{(KL)} \lp c_{\max}^{(KL)} - c_{\min}^{(KL)}\rp }{n^{7/2}} \nonumber \\ 
   & \leq  \frac{ 8K^2(l-1) \lp c_{\max}^{(KL)} - c_{\min}^{(KL)}\rp }{n^{5/2}},
\label{eq:kl_proof3}
\end{align}
where the second inequality follows from the assumption that { $n \geq 8Kc_{\max}^{(KL)}/(l-1)$}.

\subsubsection{Bound on the term $\bm{\psi_{2,i}}$.}

Finally, we obtain the required bounds on the term $\psi_{2,i} \coloneqq \sum_{m=1}^r \sum_{j=1}^l \beta^{(ij)}_m /\lp \tau_{0,i}^m m! p_{ij}^{m-1} \rp$, where $\beta^{(ij)}_m\lp \tau_{0,i} ,\tau_{1,i}\rp$ was introduced in~\eqref{eq:define_q_m}. 
We consider several cases: 

\paragraph{$\bm{m=1}$.} This is the simplest case. 
\begin{align*}
   \sum_{j=1}^l \frac{\beta^{(ij)}_1}{\tau_{0,i}} =  \sum_{j=1}^l \frac{1}{\tau_{0,i}} \lp \lp \tau_{1,i} - \tau_{0,i} \rp\times 1 \times \mbb{E} \lb W_{\tau_{0,i}} \rb \rp = 0, 
\end{align*}
which follows from the fact that $\mbb{E} \lb W_{\tau_{0,i}} \rb = 0$. 

\paragraph{$\bm{m=2}$.} 
\begin{align}
    \sum_{j=1}^l \frac{1}{\tau_{0,i}^2 2! p_{ij}} \beta^{(ij)}_m & = \sum_{j=1}^l \frac{1}{\tau_{0,i}^2 2! p_{ij}} \lp \lp \tau_{1,i} - \tau_{0,i} \rp {2 \choose 1} \mbb{E} \lb W_{\tau_{0,i}} \rb 
    + \lp \tau_{1,i}-\tau_{0,i}\rp^2 {2 \choose 2} \times 1 \rp  \nonumber \\
    & =  \sum_{j=1}^l  \frac{ \lp \tau_{1,i} - \tau_{0,i} \rp^2} {2 \tau_{0,i}^2  p_{ij}}
    \leq \frac{ \lp c_{\max}^{(KL)} - c_{\min}^{(KL)}\rp^2 9K^2 l}{n^3 \eta}
    \label{eq:kl_proof6}
\end{align}

{ assumed that $n$ is large enough to ensure that $\tau_{0,i} \geq n/(3K)$ which is true if $n$ is large enough to ensure that $ n \geq \lp \frac{24K \sqrt{96 K \log(1/\delta_n)}}{(l-1)\eta^3}\rp^{2/3}$.
}

.

\paragraph{$\bm{m \geq 3}$.} Finally we consider the case where $m \geq 3$. 

\begin{align}
    \sum_{j=1}^l \frac{1}{\tau_{0,i}^m m! p_{ij}^{m-1} } \beta^{(ij)}_m & = 
    \sum_{j=1}^l \frac{1}{\tau_{0,i}^m m! p_{ij}^{m-1} } \lp \sum_{k=1}^{m} \lp \tau_{1,i}-\tau_{0,i}\rp^k {m \choose k} \mbb{E} \lb W_{\tau_{0,i}}^{m-k} \rb \rp \nonumber \\
    & \stackrel{(a)}{=}  \sum_{j=1}^l \frac{1}{\tau_{0,i}^m m! p_{ij}^{m-1} } \lp \sum_{k=1}^{m-2} \lp \tau_{1,i}-\tau_{0,i}\rp^k {m \choose k} \mbb{E} \lb W_{\tau_{0,i}}^{m-k} \rb + 
    \lp \tau_{1,i} - \tau_{0,i} \rp^m \times 1 \rp \nonumber \\ 
    & \leq  \sum_{j=1}^l p_{ij} \frac{ \tau_{1,i} - \tau_{0,i}}{\tau_{0,i}^m m! \eta^m } \lp (m-2) m! \lp \sqrt{m-2} e^{1/e} \sqrt{3 \tau_{0,i}/2} \rp^{m-2} \rp \nonumber \\
    & \qquad \qquad + \frac{ \lp c_{\max}^{(KL)} - c_{\min}^{(KL)}\rp^m \lp 2K\rp^m l}{\eta^m n^{7/2}}  \nonumber\\
    & = \frac{ a_m \lp c_{\max}^{(KL)} - c_{\min}^{(KL)}\rp}{\eta^m n^{-5/2}} + \frac{ \lp c_{\max}^{(KL)} - c_{\min}^{(KL)}\rp^m \lp 2K\rp^m l}{\eta^m n^{7/2}}. 
    \label{eq:kl_proof7}
\end{align}
where $a_m \coloneqq (m-2)^{m/2}(3K)^m e^{(m-2)/e} (3/2)^{m/2-1}$ is implicitly defined in the last equality.  

Thus we can obtain the following bound on the term $\psi_{2,i}$ using~\eqref{eq:kl_proof6} and~\eqref{eq:kl_proof7}. 
\begin{align}
    \psi_{2,i} &\leq \frac{ \lp c_{\max}^{(KL)} - c_{\min}^{(KL)}\rp^2 9K^2 l}{n^3 \eta} + \lp r-3\rp \lp 
    \frac{ a_r \lp c_{\max}^{(KL)} - c_{\min}^{(KL)}\rp}{\eta^r n^{-5/2}} + \frac{ \lp c_{\max}^{(KL)} - c_{\min}^{(KL)}\rp^r \lp 2K\rp^r l}{\eta^r n^{7/2}} \rp  \nonumber \\
   & = \frac{(r-3) a_r \lp c_{\max}^{(KL)} - c_{\min}^{(KL)}\rp}{\eta^r n^{-5/2}} + \tilde{\mc{O}} \lp n^{-3} \rp. 
   \label{eq:kl_proof8}
\end{align}

The final bound on the regret for the adaptive sampling scheme with loss $D_{KL}$ is obtained by summing up the contributions outlined in equations ~\eqref{eq:kl_proof0},~\eqref{eq:kl_proof00},~\eqref{eq:kl_proof10},~\eqref{eq:kl_proof6}  and~\eqref{eq:kl_proof8}, we get that 
\begin{align}
    \mc{R}_n \lp \mc{A}_{KL}, D_{KL} \rp & \leq 
    \frac{ 8K^2(l-1) \lp c_{\max}^{(KL)} - c_{\min}^{(KL)}\rp }{n^{5/2}} + 
    \frac{(r-3) a_5 \lp c_{\max}^{(KL)} - c_{\min}^{(KL)}\rp}{\eta^5 n^{-5/2}} + 
    \tilde{\mc{O} } \lp n^{-3} \rp . 
    \label{eq:kl_proof9}
\end{align}

This result holds if $n$ is large enough to ensure all of  the following:
\begin{align}
    \label{eq:n_large_enough_kl}
    \begin{split}
    n &\geq \frac{ K(l-\eta)}{6(l-1)} \\
     n &\geq 8Kc_{\max}^{(KL)}/(l-1)\quad \text{ and} \\
    n &\geq \lp \frac{24K \sqrt{96 K \log(1/\delta_n)}}{(l-1)\eta^3}\rp^{2/3}
   \end{split} 
\end{align}

\subsection{Derivations for $D_{\chi^2}$}
\label{appendix:chi_squared}

The $\chi^2$-distance between two discrete distributions $P$ and $Q$ is defined as $D_{\chi^2}(P, Q) = \sum_{j=1}^l \frac{ (p_j - q_j)^2}{q_j}$. If $\hat{P}_i$ denotes the empirical estimator of a distribution $P_i$ using $T_i$ i.i.d. samples, then we can compute the expected value of $D_{\chi^2}$ between $\hat{P}_i$ and $P_i$ as follows:
\begin{align*}
\mbb{E}\lb D_{\chi^2}\lp \hat{P}_i, P_i \rp \rb &= \sum_{j=1}^l \mbb{E} \lb  \frac{ \lp \hat{p}_{ij} - p_{ij} \rp^2}{p_{ij}} \rb = \sum_{j=1}^l \mbb{E} \lb \frac{ \hat{p}_{ij}^2 + p_{ij}^2 - 2p_{ij}\hat{p}_{ij}}{p_{ij}} \rb \\
    & = \sum_{j=1}^l\lp \mbb{E} \lb \frac{ \hat{p}_{ij}^2}{p_{ij}} \rb - p_{ij} \rp   
    = \sum_{j=1}^l \lp  \frac{ p_{ij}(1-p_{ij})}{T_i p_{ij}} - p_{ij} \rp \\
    &  = \frac{ l-1}{T_i} - 1. 
\end{align*}

Thus the tracking objective function is given by  $\f(c_i, T_i) = \frac{l-1}{2T_i}$. Note that the objective function can be computed exactly  and lies in the function class $\mc{F}$.  Hence we have $T_i^*=\tilde{T}_i^*$. Furthermore, since the objective function only depends on the support size of the distributions, and we assume all $P_i \in \Delta_l$ for the same $l$, the optimal allocation is the uniform allocation.

%% file: appendix_regret_separation.tex
\subsection{Proof of Lemma~\ref{lemma:separation1}}
\label{proof_separation_lemma1}

\begin{proof}

Recall that  $Z_{ij}^{(t)}$ is the indicator that the output of the $s^{th}$ pull of arm $i$ was $x_j$.  
 and define $V_j \coloneqq \frac{\sum_{s=1}^{T_i} p_{ij} - Z_{ij}^{(s)}}{\sqrt{T_i} p_{ij}}$. Next, let $\rho: \mbb{R}^l \mapsto \{1, 2, \ldots, l\}$ represent the $\argmax$ operation which breaks ties by returning the smallest index, i.e., for $\Vec{v}=(v_1,v_2,\ldots,v_l) \in \mbb{R}^l$,  $\rho\lp \Vec{v} \rp = \min \{ j : v_j = \max_{1\leq k \leq l}v_k \}$. With this definition of $\rho$, and with $\Vec{V} \coloneqq (V_1, V_2, \ldots, V_l)$, we introduce the events $\Omega_j \coloneqq \{\rho(\Vec{V}) = j\}$ for $j=1,2,\ldots, l$. 

We can now proceed as follows:
\begin{align*}
    \mbb{E}\lb D_s\lp \hat{P}_i, P_i \rp \rb &\coloneqq \mbb{E} \lb \max_{1 \leq j \leq l} \lp 1- \frac{\hat{p}_{ij}}{{p_{ij}}}\rp \rb = \mbb{E}\lb \max_{1\leq j \leq l} \sum_{t=1}^{T_i} \frac{p_{ij}-Z_{ij}^{(t)} }{\sqrt{p_{ij}(1-p_{ij})T_i}}\sqrt{ \frac{1-p_{ij}}{p_{ij}T_i}}\rb  \\
    & = \frac{1}{\sqrt{T_i}} \mbb{E}\lb \max_{1 \leq j \leq l} V_j \lp \sum_{k=1}^l \mathbbm{1}_{\Omega_k} \rp  \rb \stackrel{(a)}{=}  \frac{1}{\sqrt{T_i}} \mbb{E}\lb \sum_{j=1}^l V_j \mathbbm{1}_{\Omega_j}  \rb\\
    &\stackrel{(b)}{\leq} \frac{1}{\sqrt{T_i}} \sum_{j=1}^l \mbb{E} \lb V_j \indi{V_j \geq 0} \rb \stackrel{(c)}{\leq}
    \sqrt{\frac{1}{2\pi T_i}} \lp \sum_{j=1}^l \sqrt{\frac{1-p_{ij}}{p_{ij}}} \rp + C^{(s)}_i \frac{1}{T_i}.   
\end{align*}
In the above display, \tbf{(a)} follows from the fact that on the event $\Omega_j$, $\max_{1\leq k\leq l}V_k = V_j$, \tbf{(b)} follows from the observation that $\Omega_j \subset \{V_j \geq 0\}$, and \tbf{(c)} follows from an application of Theorem~3.2 of \cite{ross2011fundamentals} and  the term $C_i^{(s)}$ is given by
\begin{equation}
\label{eq:Cis}
C_i^{(s)} = \sum_{j=1}^l\lp  (1-p_{ij})(1+2p_{ij}^2 -2p_{ij}) + \lp 2(1-p_{ij})\lp (1-p_{ij})^3 + p_{ij}^3 \rp/\pi \rp^{1/2} \rp. 
\end{equation}

Next, in order to obtain the lower bound we need some additional notation. For any subset $S$ of $[l]$, such that $|S| <l$, we introduce the terms
\begin{align}
\label{eq:bunch1}
    \mf{p}_{i,S} = \sum_{j\in S} p_{ij}, \quad \text{and} \quad 
    \hat{\mf{p}}_{i,S} = \sum_{j\in S} \hat{p}_{ij}. 
\end{align}
Define the event that  $\Omega_0 \coloneqq \{\hat{\mf{p}}_i \geq \mf{p}_i \}$. 
Then we have the following:
\begin{align*}
    \mbb{E} \lb \max_{1 \leq j \leq l} \lp 1 - \frac{ \hat{p}_{ij}  }{ p_{ij}  } \rp \rb & 
   \stackrel{(a)}{ \geq} \mbb{E} \lb \max \lp 1 - \frac{\hat{\mf{p}}_{i,S}}{\mf{p}_{i,S}}, \; 1 - \frac{ 1- \hat{\mf{p}}_{i,S}}{1 - \mf{p}_{i,S}} \rp \rb  \\
   &= \mbb{E} \lb \lp 1 - \frac{\hat{\mf{p}}_{i,S}}{\mf{p}_{i,S}} \rp \indi{\Omega_0} + \lp 1 - \frac{ 1- \hat{\mf{p}}_{i,S}}{1 - \mf{p}_{i,S}} \rp \indi{\Omega_0^c} \rb  \\
   & \stackrel{(c)}{\geq } \sqrt{ \frac{1}{2\pi T_i }} \lp \sqrt{ \frac{ 1- \mf{p}_{i,S}}{\mf{p}_{i,S}}} + \sqrt{ \frac{ \mf{p}_{i,S}}{1-\mf{p}_{i,S}}} \rp  - \frac{\tilde{C}_i^{(s)}(\mf{p}_{i,S}}{T_i}. 
\end{align*}
In the above display, \\
\tbf{(a)} follows from the argument that `\emph{bunching together}' probability mass can only result in a lower value of the term inside the expectation as formally proved in Lemma~\ref{lemma:separation_loss2}. 
\tbf{(c)} follows from an application of Theorem~3.2 of \cite{ross2011fundamentals} and the term $\tilde{C}_i^{(s)}(S)$ is defined as follows:
\begin{equation}
    \label{eq:CisTilde1}
    \tilde{C}_i^{(s)}(S) \coloneqq \sum_{p \in \{ \mf{p}_{i,S}, 1-\mf{p}_{i,S} \} }
    \lp 
  (1-p)(1+2p^2-2p) + \lp 2(1-p)\lp (1-p)^3 + p^3\rp/\pi \rp^{1/2}
    \rp
\end{equation}
Finally, the result follows by first introducing the definition
$\tilde{c}_i \coloneqq \max \left \{\lp \sqrt{ \frac{ 1- \mf{p}_{i,S}}{\mf{p}_{i,S}}} + \sqrt{ \frac{ \mf{p}_{i,S}}{1-\mf{p}_{i,S}}} \rp \; \mid \; S\subset [l], \; 1 \leq |S| <l \right\}$. Let $S^*$ denote the subset of $[l]$ at which the maximum in the definition of $\tilde{c}_i^{(s)}$ is achieved. Using this we define the remaining term to complete the proof
\begin{equation}
    \label{eq:CisTilde2}
    \tilde{C}_i^{(s)} \coloneqq \tilde{C}_i^{(s)}\lp S^*\rp. 
\end{equation}

\end{proof}
\begin{lemma}
\label{lemma:separation_loss2}
Given two probability distributions $P = (p_1, \ldots, p_l)$ and $Q= (q_1, \ldots, q_l)$ in $\Delta_l$  (the $l-1$ dimensional probability simplex), a set $S\subset [l]$ such that $1\leq |S| <l$, and $\mf{p}_{S}$ and $\mf{q}_{S}$ be defined as in~\eqref{eq:bunch1}. Then we have the following:

\begin{align*}
    D_s\lp P, Q\rp \geq \max \lp 1 - \frac{\mf{p}_{S}}{\mf{q}_S}, \; 1-\frac{ 1-\mf{p}_S}{1-\mf{q}_S}\rp. 
\end{align*}
\end{lemma}

\begin{proof}
Suppose $j_0$ is the index of the maximizer in the definition of $D_s$, i.e., $1-p_{j_0}/q_{j_0} \geq 1-p_j/q_j$ for $j \in [l]$, $j\neq j_0$.
Equivalently, we have $p_{j_0}/q_{j_0} \leq p_j/q_j$ for all $j \in [l]$, $j \neq j_0$.
We then have the following:
\begin{align*}
    \frac{p_{j_0}}{q_{j_0}} &
     \leq \sum_{j\in S} \lp \frac{p_j}{q_j}\rp
      \frac{ q_j}{\sum_{j'\in S} q_{j'}}
      = \frac{ \mf{p}_S}{\mf{q}_S}, \qquad \text{and} \qquad
     \frac{p_{j_0}}{q_{j_0}}      \leq \sum_{j\in S^c} \lp \frac{p_j}{q_j}\rp
      \frac{ q_j}{\sum_{j'\in S^c} q_{j'}}
      = \frac{ 1- \mf{p}_S}{1-\mf{q}_S}. 
\end{align*}
This implies that $D_s\lp P, Q \rp = (1-p_{j_0}/q_{j_0}) \geq \max \{ 1- \mf{p}_S/\mf{q}_S, \; 1- (1-\mf{p}_S)/(1-\mf{q}_S)\}$, which completes the proof. 

\end{proof}

\subsection{Proof of Lemma~\ref{lemma:separation_loss1}}
\label{proof_separation_lemma2}

Assume that the event $\mc{E}_2$ defined in Lemma~\ref{lemma:ell1_1} holds, that $\hat{p}_{ij, t} \geq (7e_{ij, t})/2$ for all $1\leq j \leq l$ and that the distribution $P_i$ lies in $\Delta_l^{(\eta)}$ for $1\leq i \leq K$. 

From the proof of Lemma~\ref{lemma:kl_1}, under the above assumptions we have 
\begin{align*}
    \frac{1}{p_{ij}} - 1 & \leq \frac{1}{\hat{p}_{ij, t} - e_{ij, t}} - 1 \leq \frac{1}{p_{ij}} - 1 + \frac{1}{\eta^2} \sqrt{ \frac{32 \log(2/\delta_t)}{\Tit}}. \\
\end{align*}
We can then proceed as follows:
\begin{align*}
    \sqrt{ \frac{1}{\hat{p}_{ij, t} - e_{ij, t}} - 1 } & \stackrel{(a)}{\leq} \sqrt{ \frac{1}{p_{ij}} - 1} + \frac{1}{\eta}\lp \frac{32 \log(2/\delta_t)}{\Tit} \rp^{1/4} \quad \text{and} \\
    \sqrt{\frac{1}{\hat{p}_{ij, t} - e_{ij, t}} - 1} & \stackrel{(b)}{\leq} \sqrt{ \frac{1}{p_{ij}} -1 } + 
    \frac{1}{2\eta^{5/2}} \sqrt{ \frac{32 \log(2/\delta_t)}{\Tit}}.
\end{align*}
In the above display, \\
\tbf{(a)} follows from the fact that $\sqrt{z_1 + z_2} \leq \sqrt{z_1} + \sqrt{z_2}$ for $z_1, z_2 \geq 0$, and \\
\tbf{(b)} uses the fact that the function $h_2(x) = \sqrt{x}$ is concave, and thus majorized by its first order approximation at the point $x = 1/p_{ij} - 1$. Furthermore, since the derivative of $h_2(x)$ is $1/\lp 2\sqrt{x}\rp$, we also use the fact that $p_{ij} \leq 1-\eta$ as $P_i \in \Delta_l^{(\eta)}$. 

Define the terms $a_{i,t}$ and  $b_{i,t}$ as follows:
\begin{align}
a_{i,t} \coloneqq \lp \frac{ 32 \log( 2/\delta_t)}{\Tit} \rp^{1/4} 
\label{eq:sep_proof1}
 \quad   b_{i,t} \coloneqq  \frac{l}{\eta} a_{i,t}  \min \lp 1, \frac{a_{i,t}}{2\eta^{3/2}} \rp.
\end{align}
Then we have the following:
\begin{align*}
    \sum_{j=1}^l \sqrt{ \frac{1}{p_{ij}} - 1}  \leq 
    \sum_{j=1}^l \sqrt{ \frac{1}{\hat{p}_{ij, t} - e_{ij, t}} - 1} &\leq 
    \sum_{j=1}^l \sqrt{ \frac{1}{p_{ij}} - 1} + b_{i,t}.\\ 
\end{align*}
We will refer to the $a_{i,n}$ and $b_{i,n}$ corresponding to the oracle allocation scheme (i.e., with $T_{i, t} = T_{i}^*$) as $a_{i,n}^*$ and $b_{i,n}^*$ respectively and furthermore, define 
\begin{align}
\label{eq:bnstar}
a_n^* \coloneqq \max_{1 \leq i \leq K} a_{i,n}^*, \quad 
  b_n^* \coloneqq \max_{1 \leq i \leq K} b_{i,n}^*  
\end{align}

\subsection{Proof of Theorem~\ref{theorem:separation}}
\label{proof_separation_theorem}

\paragraph{Calculate the values $A$, $\tilde{e}_n$ and $B$.}
As before, we have $\Tins = \lambda_i^{(s)} n$ where we have $\lambda_i^{(s)}  \frac{(c_i^{(s)})^2 n}{(C_s)^2}$ and $(C_s)^2 = \sum_{k=1}^K c_k^2$. We also introduce $\lambda^{(s)}_{\min} = \min_{1 \leq i \leq K} \lambda_i^{(s)}$ and $\lambda^{(s)}_{\max} = \max_{1 \leq i \leq K} \lambda_i^{(s)}$. 
We proceed as follows:
\begin{align*}
    A &\coloneqq \max_{1 \leq i \leq K} \left. \frac {\partial \f\lp c, \Tins \rp}{ \partial c} \right \rvert_{c=c_i^{(s)}} = \frac{1}{\sqrt{\lambda^{(s)}_{\min}n}} \\
     B & \coloneqq \max_{1 \leq i \leq K} \left. \frac{ \partial \f(c_i^{(s)}, T)}{\partial T} \right \rvert_{T = \Tins} = \min_{1 \leq i \leq K} \frac{ c_i^{(s)}}{2 \lp \Tins \rp^{3/2}} = 
     \frac{C_s}{\lambda^{(s)}_{\max} n^{3/2} } \\
     \tilde{e}_n & = b_n^*, 
\end{align*}
where $b_n^*$ is defined in~\eqref{eq:bnstar}. In the rest of this section, we will use the notation $E \coloneqq \frac{ A \tilde{e}_n}{B}$. Note that by definition we have 
\begin{align}
\label{eq:sep_def_E}
    E = n\frac{ \lambda^{(s)}_{\max} }{\sqrt{\lambda^{(s)}_{\min}} C_s} \frac{l a_n^* }{\eta}\min \lp 1, \frac{a_n^*}{2\eta^{3/2}}\rp. 
\end{align}
Introduce the term  
\begin{equation}
N_0 \coloneqq \min \{ n \geq 1\; : \; (n/\log(2/\delta_n)) \geq 2/(\lambda_{\min}^{(s)} \eta^6)\},  
\label{eq:sep_def_N0}
\end{equation}
where $\delta_n$ is defined in Lemma~\ref{lemma:concentration2}. 
 Then for $n \leq N_0$, the term $E$ is $\tilde{\mc{O}}\lp n^{1/2}\rp$, while for larger values of $n$, $E$ is $\tilde{\mc{O}}\lp n^{3/4}\rp$. 

\paragraph{Regret bound derivation.}
Since we are using an approximate objective function, we employ the decomposition of the regret given in~\eqref{eq:approx_reg2} in the statement of Proposition~\ref{prop:approximate_regret} in Appendix~\ref{appendix:approximate}:
\begin{align}
\label{eq:sep_regret_decomposition}
    \mc{R}_n\lp \Algsep, \Dsep \rp & \leq \mc{L}_n \lp \Algsep, \Dsep \rp - \f \lp \cisep, \Tins \rp + 2\max_{1 \leq i\leq K } \left \lvert  R_i\lp \Tins \rp   \right \rvert. 
\end{align}

We first note that the remainder term can be upper bounded by an application of Lemma~\ref{lemma:separation1} as 
\begin{align} 
 2\max_{1 \leq i\leq K } \left \lvert R_i\lp \Tins \rp   \right \rvert \leq
2\max_{1 \leq i \leq K} \lp  \lp c_i^{(s)} - \tilde{c}_i^{(s)} \rp \sqrt{\frac{1}{2\pi \Tins}}  +  \frac{C_i^{(s)} - \tilde{C}_i^{(s)}}{\Tins} \rp \nonumber \\
= 2\lp c_i^{(s)} - \tilde{c}_i^{(s)} \rp \sqrt{\frac{1}{2\pi \lambda_{\min}^{(s)}}} + \mc{O}\lp \frac{1}{n} \rp.   \label{eq:sep_remainder_bound}
\end{align}

First we get an upper bound on $\mc{L}_n\lp \Algsep, \Dsep \rp$. 
As in the earlier sections, we consider the probability $1-\delta$ event under which we have $\tau_{0,i} \leq T_i \leq \tau_{1,i}$, and $\tau_{1,i} - \tau_{0,1} \leq 2A\tilde{e}_n/B$. Then we replace the random sum in the computation of the expectation with a sum of a (deterministic) constant number of terms.
\begin{align*}
\mc{L}_n \lp \Algsep, \Dsep \rp & \coloneqq
\mbb{E} \lb D_s\lp \hat{P}_i, P_i \rp \rb 
=\frac{1}{\sqrt{T_i}} \mbb{E} \lb \max_{1 \leq j \leq l} \sum_{t=1}^{T_i} \frac{ p_{ij} - Z_{ij}^{(t)}}{\sqrt{p_{ij}(1-p_{ij}) T_i}} \sqrt{\frac{1 - p_{ij}}{p_{ij}}}  \rb  \\ 
& \stackrel{(a)}{\leq} \frac{\delta}{\eta} +
\frac{1}{\sqrt{\tau_{0,i}}} \mbb{E} \lb \max_{1 \leq j \leq l} \sum_{t=1}^{\tau_{0,i}} \frac{ p_{ij} - Z_{ij}^{(t)}}{\sqrt{p_{ij}(1-p_{ij}) \tau_{0,i}}} \sqrt{\frac{1 - p_{ij}}{p_{ij}}}  \rb +
c_i^{(s)}\sqrt{\frac{2(\tau_{1,i} - \tau_{0,i})}{\tau_{0,i}^2} }\\
& \stackrel{(b)}{\leq} \frac{\delta}{\eta} + \f(\tau_{0,i}, c_i^{(s)}) + \frac{C_i^{(s)}}{\tau_{0,i}} + c_i^{(s)}\sqrt{\frac{4E}{\tau_{0,i}^2}}
\end{align*}
In the above display, \\
\tbf{(a)} follows from the fact that the event $\mc{E}_3$ (defined in Lemma~\ref{lemma:concentration2}) occurs with probability at least $1-\delta$, that the separation distance is always upper bounded by $1/\eta$ for $P_i \in \Delta_l^{(\eta)}$, and that under the event $\mc{E}_3$, we have $\tau_{0,i} \leq T_i \leq \tau_{1,i}$, \\
\tbf{(b)} follows from Lemma~\ref{lemma:separation_loss1} and the definition of the term $E$ in~\eqref{eq:sep_def_E}.

Next, by exploiting the convexity of the mappings $x \mapsto 1/x$ and $x \mapsto 1/\sqrt{x}$, we can obtain the following relations. 
\begin{align}
    \f\lp \tau_{0,i},  c_i^{(s)} \rp &= \frac{c_i^{(s)}}{\sqrt{\Tins\lp 1 - \frac{\Tins - \tau_{0,i}}{\Tins } \rp}}
    \leq \frac{c_i^{(s)}}{\sqrt{\Tins}} + \frac{ 2 c_i^{(s)} E}{\lp\lambda^{(s)}_{\min} n\rp^{3/2} }\\
    \frac{C_i^{(s)}}{\tau_{0,i}} & \leq \frac{C_i^{(s)}}{\Tins} + \frac{ 2 C_i^{(s)} E}{\lp \Tins\rp^2} 
    \leq  \frac{ C_i^{(s)}}{\Tins}  + \frac{ 2 C_i^{(s)} E}{\lp \lambda^{(s)}_{\min} n\rp^2 }
    \\
   \frac{2c_i^{(s)} \sqrt{E}}{\tau_{0,i}} & \leq \frac{2c_i^{(s)} \sqrt{E}}{\Tins} + \frac{ 4c_i^{(s)} E^{3/2}}{ \lp \Tins\rp^2} 
\end{align}
In the above display, we have assumed that $n$ is large enough to ensure that $(\Tins - \tau_{0,i})/\Tins \leq 1/2$ for all $1\leq i \leq K$. A sufficient condition for this is that 
\begin{equation}
    \label{eq:n_large_enough_sep}
E \leq (\lambda_{\min}^{(s)} n)/2, 
\end{equation} where $E$ is defined in~\eqref{eq:sep_def_E}. 

 Combining the relations in the previous display, we get the following bound for the expected separation distance between the empirically constructed $\hat{P}_i$ and the true distribution $P_i$. 
\begin{align}
    \mbb{E} \lb D_s \lp \hat{P}_i, P_i \rp \rb -  \f \lp c_i^{(s)}, \Tins \rp &\leq \frac{\delta}{\eta} + \frac{ C_i^{(s)} + 2c_i^{(s)} \sqrt{E}}{\Tins} + 
    \frac{2c_i^{(s)}E}{\lp \lambda^{(s)}_{\min}n\rp^{3/2}} +
    \frac{ 2C_i^{(s)} E}{\lp \lambda^{(s)}_{\min} n \rp^2} + 
    \frac{ 4c_i^{(s)}E^{3/2}}{\lp \lambda^{(s)}_{\min} n \rp^2 } \label{eq:sep_regret_term1}
\end{align}

The final expression for regret stated in Theorem~\ref{theorem:separation} follows by plugging~\eqref{eq:sep_regret_term1} and~\eqref{eq:sep_remainder_bound} in the regret decomposition given in~\eqref{eq:sep_regret_decomposition}.  

Now we show that the first term in RHS of~\eqref{eq:sep_regret_decomposition} is either $\tilde{\mc{O}} \lp n^{-3/4}\rp$ or $\tilde{\mc{O}}\lp n^{-5/8}\rp$ depending on whether $n \geq $ or $< N_0$. 
\paragraph{Case~1: $\bm{n \geq N_0}$.} In this case we have $E = \lp \frac{\lambda^{(s)}_{\max} (a_n^*)^2}{C \sqrt{\lambda^{(s)}_{\min}} 2\eta^{5/2}}\rp n^{1/2} \coloneqq M_1 \sqrt{n}$. Then we have 
\begin{align}
    \mbb{E} \lb D_s \lp \hat{P}_i, P_i \rp \rb -  \f \lp c_i^{(s)}, \Tins \rp & \leq   \frac{2c_i^{(s)} \sqrt{M_1}}{\lambda^{(s)}_{\min}n^{3/4}} + 
    \frac{\delta}{\eta} + \frac{1}{n}\lp \frac{C_i^{(s)}}{\lambda^{(s)}_{\min}} + \frac{2c_i^{(s)} M_1}{\lp \lambda^{(s)}_{\min}\rp^{3/2} n} 
 + \frac{4M_1^2}{\lp \lambda^{(s)}_{\min}\rp^2}    \rp  + \frac{2C_i^{(s)} M_1}{\lp \lambda^{(s)}_{\min}\rp^2 n^{3/2}}. \label{eq:sep_regret0}
\end{align}
If $\delta \leq \eta/n$, then we have that 
\begin{equation}
    \label{eq:sep_regret1}
    \mbb{E} \lb D_s \lp \hat{P}_i, P_i \rp \rb -  \f \lp c_i^{(s)}, \Tins \rp  \leq  \frac{2c_i^{(s)} \sqrt{M_1}}{\lambda^{(s)}_{\min}n^{3/4}}
    + \tilde{\mc{O}}\lp \frac{1}{n} \rp. 
\end{equation}

\paragraph{Case~2: $\bm{n <N_0}$.} In this case, we have  $E = n^{3/4}\frac{ \lambda^{(s)}_{\max} }{\sqrt{\lambda^{(s)}_{\min}} C_s} \frac{l a_n^* }{\eta} \coloneqq M_2 n^{3/4}$, which implies that 
\begin{align}
      \mbb{E} \lb D_s \lp \hat{P}_i, P_i \rp \rb -  \f \lp c_i^{(s)}, \Tins \rp & \leq  \frac{2c_i^{(s)} \sqrt{M_2}}{\lambda^{(s)}_{\min}}n^{-5/8} + 
    \frac{\delta}{\eta} + \frac{2C_i^{(s)} M_2 n^{-3/4}}{\lp \lambda^{(s)}_{\min}\rp^2} + 
    \frac{2C_i^{(s)} M_2 n^{-5/4}}{\lp \lambda^{(s)}_{\min}\rp^2} 
      + \frac{ 4c_i^{(s)} M_2^{3/2} n^{-7/8}}{\lp \lambda^{(s)}_{\min}\rp^2}, 
    \label{eq:sep_regret00}
\end{align}
which implies that in this case we have 
\begin{align}
    \label{eq:sep_regret2}
    \mbb{E} \lb D_s \lp \hat{P}_i, P_i \rp \rb -  \f \lp c_i^{(s)}, \Tins \rp  \leq  
     \frac{2c_i^{(s)} \sqrt{M_2}}{\lambda^{(s)}_{\min}}n^{-5/8}  + 
      \tilde{\mc{O}}\lp n^{-3/4} \rp. 
\end{align}


%% file: appendix_LowerBound.tex
In this section, we present the proof of Theorem~\ref{theorem:lower} and Corollary~\ref{corollary:lower} which were stated in Section~\ref{sec:lower}.

\subsection{Proof of Theorem~\ref{theorem:lower}}
\label{subsec:lower_theorem_proof} 

For any adaptive allocation scheme $\mc{A}$, recall that $T_1$ and $T_2$ refer to the number of samples drawn from the first and second arms respectively. Since $K=2$, we have $T_1 + T_2 = n$ where $n$ is the sampling budget.
Denote by $\Omega = \{0,1 \}^n$, the observation space. 
The allocation scheme $\mc{A}$ along with the distributions of the two arms induce a probability measure on $\Omega$. We use $\mbb{P}_1$ and $\mbb{P}_2$ to represent the two probability measures on $\Omega$ corresponding to the two problem instances $\mc{P}_1$ and $\mc{P}_2$ respectively. Finally recall that we use $(\Tins)_{i=1}^2$ to represent the oracle allocation scheme corresponding to the objective functions $\f(c_i, T_i) = c_i/T_i^{\alpha}$ for $\alpha>0$. For the case of problem $\mc{P}_1$, we have $c_1 = \cfunc(p_0)$ and $c_2=\cfunc(p_0-\epsilon)$ and these terms are swapped for the problem $\mc{P}_2$. Note that since $K=2$, we have $|\widetilde{T}_1^*-n/2| =|\widetilde{T}_2^* - n/2|$ for both problem instances.

Now, without loss of generality we assume that $\cfunc{p_0} < \cfunc{p_0-\epsilon}$, and define the event $\mc{E} \coloneqq \{T_1 < n/2\}$. The other case, (i.e., $\cfunc{p_0}< \cfunc{p_0-\epsilon}$) can be handled in an analogous manner by considering the event $\mc{E}^c$. By an application of  the standard change-of-measure result~\cite[Lemma~1]{kaufmann2017learning}, we have 
\begin{align}
\label{eq:lower1}
    \kl\lp \mbb{P}_1 \lp \mc{E} \rp, \mbb{P}_2 \lp \mc{E} \rp \rp & \leq
     \mbb{E}_1 \lb T_1 \rb \kl \lp p_0, p_0 - \epsilon \rp + \mbb{E}_1 \lb T_2 \rb \kl \lp p_0 - \epsilon, p_0 \rp. 
\end{align}
In the above display, we use the notation $\kl (p,q)$ to represent the $\kl$-divergence between $\Ber{p}$ and $\Ber{q}$ random variables. 
Next, we note that for any $p,q \in (0,1)$ we have $2 (p-q)^2 \leq \kl (p, q)$ by an application of Pinsker's inequality. Furthermore, due to the concavity of $\log(x)$, we can also show that $\kl(p,q) \leq \frac{(p-q)^2}{q(1-q)}$. Combining these two observations with~\eqref{eq:lower1}, along with the assumption that $\epsilon<p_0-1/2$, we have 
\begin{align*}
    2 \lp \mbb{P}_1 \lp \mc{E} \rp - \mbb{P}_2 \lp \mc{E} \rp \rp^2  \leq 
    & \mbb{E}_1 \lb T_1 \rb \frac{ \epsilon^2}{(p_0-\epsilon)(1-p_0 + \epsilon) } 
    + \mbb{E}_1 \lb T_2 \rb \frac{ \epsilon^2}{ p_0 (1-p_0)} 
     \leq \frac{ 2\epsilon^2}{1-p_0} \lp \mbb{E}_1 \lb T_1 + T_2 \rb \rp 
    = \frac{ 2 \epsilon^2 n}{1-p_0}. 
\end{align*}
From the above display, we can conclude that 
\begin{equation}
    \label{eq:lower2}
    |\mbb{P}_1 \lp \mc{E} \rp - \mbb{P}_2 \lp \mc{E} \rp | \leq \epsilon \sqrt{ \frac{n}{1-p_0}}. 
\end{equation}
Introducing the notation $\delta = \frac{ 1- \epsilon\sqrt{ n/(1-p_0)}}{2}$, we note that the inequality~\eqref{eq:lower2} implies that at least one the following two statements is true: \tbf{(1)} $\mbb{P}_1 \lp \mc{E}^c \rp \geq \delta$, or \tbf{(2)} $\mbb{P}_2 \lp \mc{E} \rp \geq \delta$.
With the notation $\tau \coloneqq |\Tins - n/2|$, we note that under the event $\mc{E}$ for the problem instance $\mc{P}_1$, we have $|T_1 - \widetilde{T}_1^*| \geq \tau$ and under the event $\mc{E}^c$ for problem instance $\mc{P}_2$ we have $|T_2 - \widetilde{T}_1^*| \geq \tau$. This implies that we have the following:
\begin{align}
\label{eq:lower3}
    \max_{\mc{P}_1, \mc{P}_2} \; \max_{i=1,2}\; \mbb{E} \lb |T_i - \Tins| \rb \geq 
    \delta \tau = \lp \frac{\lp 1 - \epsilon \sqrt{n/(1-p_0)}\rp\tau }{2} \rp . 
\end{align}

Finally, the result of the statement follows from the following two  observations:
\begin{align*}
|\widetilde{T}_1^* - \widetilde{T}_2^*| &= 2|\widetilde{T}_1^* - n/2| = 2 \tau \; \text{ and} \\
|\widetilde{T}_1^* - \widetilde{T}_2^*| &= n \left \lvert \frac{ \lp \cfunc{p_0}\rp^{1/\alpha} - 
\lp \cfunc{p_0-\epsilon} \rp^{1/\alpha}}{\lp \cfunc{p_0} \rp^{1/\alpha} + \lp \cfunc{p_0-\epsilon} \rp^{1/\alpha} } \right \rvert. 
\end{align*}

\subsection{Proof of Corollary~\ref{corollary:lower}}
\label{subsec:proof_lower_corollary} 
We first note that with a choice of $\epsilon = \frac{1}{4 \sqrt{n}}$ and $p_0 = 3/4$, we have $\delta  =  \frac{1 - \epsilon\sqrt{n/(1-p_0)}}{2} = 1/4$, and the RHS of~\eqref{eq:lower3} becomes $\tau/4$.  
To complete the proof, it suffices to show that $| \cfunc{p_0}^{1/\alpha} - \cfunc{p_0-\epsilon}^{1/\alpha}| = \Omega(\epsilon)$, for the $\cfuncs$ corresponding to the three loss functions $\Dltwo, \Dlone$ and $\Dsep$. The result then follows from the definition of $\tau$ and the choice of $\epsilon$.  We consider the three cases separately. 

\paragraph{$\bm{\ell_2^2}$-distance.} In this case, we have $\cfunc{p} = 1 - p^2 - (1-p)^2$ and $\alpha=1$. Thus we have $|\cfunc{p_0} - \cfunc{p_0 - \epsilon}| = \epsilon - 2\epsilon^2 \geq \epsilon/2$, where the last inequality follows from the fact that for $\epsilon=1/4\sqrt{n}$, we have $\epsilon^2 \leq \epsilon/2$. 

\paragraph{ $\bm{\ell_1}$-distance. } In this case, we have $\cfunc{p} = 2\sqrt{p(1-p)}$ and $\alpha=1/2$.  Thus, we have for $p_0=3/4$  $|\cfunc{p_0}^{1/\alpha} - \cfunc{p_0-\epsilon}^{1/\alpha}| = 4(\epsilon/2 - \epsilon^2) \geq \epsilon$. In the last inequality we used the fact that for the choice of $\epsilon$ mentioned earlier, we have $\epsilon^2 \leq \epsilon/4$. 

\paragraph{ Separation distance.} In this case, we have $\cfunc{p} = \sqrt{1/p - 1} + \sqrt{1/(1-p) + 1}$ and $\alpha=1/2$. Thus, we have for $p_0=3/4$, $|\cfunc{p_0}^{1/\alpha} - \cfunc{p_0-\epsilon}^{1/\alpha}| = \frac{ 1 - 2(1-p_0)}{p_0(1-p_0)} - \frac{1 - 2(1-p_0+\epsilon)}{(p_0-\epsilon)(1-p_0+\epsilon)} \coloneqq \zeta(1/4) - \zeta(1/4+\epsilon)$, where we have defined $\zeta(z) = \frac{1 - 2x}{x(1-x)}$. Note that for $x \in [1/4, 1/2]$ the function $\zeta$ is decreasing and convex, with $\zeta(1/2)=0$. Thus we have $\zeta(1/4) - \zeta(1/4+\epsilon) \geq \lp \frac{ 0 - \zeta(1/4)}{1/2 - 1/4} \rp \lp 1/4 - (1/4 + \epsilon)\rp = \frac{2\epsilon}{3}$. In the last inequality, we used the fact that $\zeta(x)$ is majorized by the straight line joining $(1/4, \zeta(1/4))$ and $(1/2,0)$ due to the convexity of $\zeta$ in this domain.  

Thus for all the three distances, we have $\tau = (n/2) \frac{|\cfunc{p_0}^{1/\alpha} - \cfunc{p_0-\epsilon}^{1/\alpha}|}{|\cfunc{p_0}^{1/\alpha} + \cfunc{p_0-\epsilon}^{1/\alpha}|} = \Omega \lp \epsilon n \rp = \Omega \lp \sqrt{n}\rp$ as required. 

%% file: main.bbl
\begin{thebibliography}{}

\bibitem[Aldous and Diaconis, 1987]{aldous1987strong}
Aldous, D. and Diaconis, P. (1987).
\newblock Strong uniform times and finite random walks.
\newblock {\em Advances in Applied Mathematics}, 8(1):69--97.

\bibitem[Antos et~al., 2008]{antos2008active}
Antos, A., Grover, V., and Szepesv{\'a}ri, C. (2008).
\newblock Active learning in multi-armed bandits.
\newblock In {\em International Conference on Algorithmic Learning Theory},
  pages 287--302.

\bibitem[Auer et~al., 2002]{auer2002finite}
Auer, P., Cesa-Bianchi, N., and Fischer, P. (2002).
\newblock Finite-time analysis of the multiarmed bandit problem.
\newblock {\em Machine learning}, 47(2-3):235--256.

\bibitem[Barron et~al., 1992]{barron1992distribution}
Barron, A., Gyorfi, L., and van~der Meulen, E. (1992).
\newblock Distribution estimation consistent in total variation and in two
  types of information divergence.
\newblock {\em IEEE transactions on Information Theory}, 38(5):1437--1454.

\bibitem[Blyth, 1980]{blyth1980expected}
Blyth, C. (1980).
\newblock Expected absolute error of the usual estimator of the binomial
  parameter.
\newblock {\em The American Statistician}, 34(3):155--157.

\bibitem[Carpentier et~al., 2011]{carpentier2011upper}
Carpentier, A., Lazaric, A., Ghavamzadeh, M., Munos, R., and Auer, P. (2011).
\newblock Upper-confidence-bound algorithms for active learning in multi-armed
  bandits.
\newblock In {\em International Conference on Algorithmic Learning Theory},
  pages 189--203.

\bibitem[Carpentier and Munos, 2011]{Carpentier11FT}
Carpentier, A. and Munos, R. (2011).
\newblock Finite time analysis of stratified sampling for monte carlo.
\newblock In {\em Advances in Neural Information Processing Systems 24}, pages
  1278--1286.

\bibitem[Csisz{\'a}r, 1967]{csiszar1967source}
Csisz{\'a}r, I. (1967).
\newblock Two remarks to noiseless coding.
\newblock {\em Information and Control}, 11(3):317--322.

\bibitem[Csisz{\'a}r, 1995]{csiszar1995channel}
Csisz{\'a}r, I. (1995).
\newblock Generalized cutoff rates and r{\'e}nyi's information measures.
\newblock {\em IEEE Transactions on information theory}, 41(1):26--34.

\bibitem[Diaconis and Zabell, 1991]{diaconis1991closed}
Diaconis, P. and Zabell, S. (1991).
\newblock Closed form summation for classical distributions: variations on a
  theme of de moivre.
\newblock {\em Statistical Science}, pages 284--302.

\bibitem[Durrett, 2019]{durrett2019probability}
Durrett, R. (2019).
\newblock {\em Probability: theory and examples}, volume~49.
\newblock Cambridge university press.

\bibitem[Gibbs and Su, 2002]{gibbs2002choosing}
Gibbs, A. and Su, F. (2002).
\newblock On choosing and bounding probability metrics.
\newblock {\em International statistical review}, 70(3):419--435.

\bibitem[Gyorfi et~al., 2000]{gyorfi2000goodness}
Gyorfi, L., Morvai, G., and Vajda, I. (2000).
\newblock Information-theoretic methods in testing the goodness of fit.
\newblock In {\em 2000 IEEE International Symposium on Information Theory (Cat.
  No. 00CH37060)}, page~28. IEEE.

\bibitem[Hagerup and R{\"u}b, 1990]{hagerup1990guided}
Hagerup, T. and R{\"u}b, C. (1990).
\newblock A guided tour of chernoff bounds.
\newblock {\em Information processing letters}, 33(6):305--308.

\bibitem[Harris, 1975]{harris1975statistical}
Harris, B. (1975).
\newblock The statistical estimation of entropy in the non-parametric case.
\newblock Technical report, University of Wisconsin-Madison Mathematics
  Research Center.

\bibitem[Hazan et~al., 2019]{Hazan19PE}
Hazan, E., Kakade, S., Singh, K., and Soest, A.~V. (2019).
\newblock Provably efficient maximum entropy exploration.
\newblock In {\em Proceedings of the 36th International Conference on Machine
  Learning}.

\bibitem[Kaufmann and Garivier, 2017]{kaufmann2017learning}
Kaufmann, E. and Garivier, A. (2017).
\newblock Learning the distribution with largest mean: two bandit frameworks.
\newblock {\em ESAIM: Proceedings and Surveys}, 60:114--131.

\bibitem[Maurer and Pontil, 2009]{maurer2009empirical}
Maurer, A. and Pontil, M. (2009).
\newblock Empirical bernstein bounds and sample-variance penalization.
\newblock In {\em COLT}.

\bibitem[Riquelme et~al., 2017]{Riquelme17AL}
Riquelme, C., Ghavamzadeh, M., and Lazaric, A. (2017).
\newblock Active learning for accurate estimation of linear models.
\newblock In {\em Proceedings of the 34th International Conference on Machine
  Learning}, pages 2931--2939.

\bibitem[Rivasplata, 2012]{rivasplata2012subgaussian}
Rivasplata, O. (2012).
\newblock Subgaussian random variables: An expository note.
\newblock {\em Technical Report}.

\bibitem[Ross, 2011]{ross2011fundamentals}
Ross, N. (2011).
\newblock Fundamentals of {S}tein's method.
\newblock {\em Probability Surveys}, 8:210--293.

\bibitem[Soare et~al., 2013]{Soare2013Active}
Soare, M., Lazaric, A., and Munos, R. (2013).
\newblock Active learning in linear stochastic bandits.
\newblock In {\em Bayesian Optimization in Theory and Practice}.

\bibitem[Soare et~al., 2014]{soare2014best}
Soare, M., Lazaric, A., and Munos, R. (2014).
\newblock Best-arm identification in linear bandits.
\newblock In {\em Advances in Neural Information Processing Systems}, pages
  828--836.

\bibitem[Tarbouriech and Lazaric, 2019]{tarbouriech2019active}
Tarbouriech, J. and Lazaric, A. (2019).
\newblock Active exploration in markov decision processes.
\newblock In {\em Proceedings of the 22nd International Conference on
  Artificial Intelligence and Statistics}.

\bibitem[Tsybakov, 2009]{tsybakov2009lower}
Tsybakov, A. (2009).
\newblock Lower bounds on the minimax risk.
\newblock In {\em Introduction to Nonparametric Estimation}, pages 77--135.
  Springer.

\end{thebibliography}
